\documentclass[twoside,11pt]{article}

\usepackage{ifthen}
\newboolean{JMLR}
\setboolean{JMLR}{false}
\ifthenelse{\boolean{JMLR}}{\usepackage{jmlr2e}}{\usepackage{jmlr2eArxiv}}

\usepackage{amsmath}
\usepackage{algorithm}
\usepackage[noend]{algpseudocode}
\usepackage{pifont}
\usepackage{xspace}
\usepackage[utf8]{inputenc}
\usepackage{subfigure}
\usepackage{multirow}
\usepackage{booktabs}
\usepackage{hyperref}
\hypersetup{colorlinks=true,linkcolor=black,citecolor=black,filecolor=black,urlcolor=black}
\ifthenelse{\boolean{JMLR}}{\graphicspath{{figs/}}}{}

\jmlrheading{---}{---}{---}{12/12}{---}{Tammo Krueger, Danny Panknin and Mikio Braun}%
\ShortHeadings{Fast Cross-Validation via Sequential Testing}{Krueger, Panknin and Braun}%
\firstpageno{1}%
\editor{Charles Elkan}%

\newcommand{\Bounda}{\log \frac{1-\beta_l}{\alpha_l}}
\newcommand{\Boundb}{\log \frac{\beta_l}{1-\alpha_l}}
\newcommand{\Boundp}{\log \frac{\pi_1}{\pi_0}}
\newcommand{\Boundq}{\log \frac{1-\pi_1}{1-\pi_0}} 
\newcommand{\BoundQ}{\log \frac{1-\pi_0}{1-\pi_1}}

\newcommand{\sX}{\mathcal{X}}
\newcommand{\sY}{\mathcal{Y}}
\newcommand{\steps}{\ensuremath{S}}
\newcommand{\modelsize}{\ensuremath{n}}
\newcommand{\stoppingWindow}{\ensuremath{w_\text{stop}}\xspace}
\newcommand{\BoundPi}{\ensuremath{\sqrt[\steps]{\frac{1-\beta_l}{\alpha_l}}}}
\newcommand{\secZone}{\ensuremath{\frac{\Boundb}{\log 2 - \sqrt[\steps]{\frac{1-\beta_l}{\alpha_l}}}}}
\newcommand{\real}{\mathbb{R}}

\newcommand{\srow}{s_\text{R}}
\newcommand{\scol}{s_\text{C}}
\newcommand{\ssafe}{s_\text{safe}}
\newcommand{\srate}{s_r}
\newcommand{\tfull}{t}

\newcommand{\JMLRfigureMore}[4]{\begin{figure}[tb]%
\begin{center}%
\includegraphics[width=#2\textwidth,#4]{#1}
\end{center}
\caption{#3}
\label{fig:#1}
\end{figure}
}

\newcommand{\JMLRfigure}[3]{\JMLRfigureMore{#1}{#2}{#3}{}}

\newcommand{\JMLRfigureP}[3]{\begin{figure}[tbp]%
\begin{center}%
\includegraphics[width=#2\textwidth]{#1}
\end{center}
\caption{#3}
\label{fig:#1}
\end{figure}
}

\newcommand{\JMLRdoublefigure}[4]{\begin{figure}[tbp]%
\begin{center}%
\includegraphics[width=#3\textwidth]{#1}
\includegraphics[width=#3\textwidth]{#2}
\end{center}
\caption{#4}
\label{fig:#1}
\end{figure}
}

\newcommand{\JMLRtripplefigure}[5]{\begin{figure}[p]%
\begin{center}%
\includegraphics[width=#4\textwidth]{#1}
\includegraphics[width=#4\textwidth]{#2}
\includegraphics[width=#4\textwidth]{#3}
\end{center}
\caption{#5}
\label{fig:#1}
\end{figure}
}

\newcommand{\hypocl}{\mathcal{G}}
\newcommand{\rkhs}{\mathcal{H}}

\begin{document}

\title{Fast Cross-Validation via Sequential Testing}

\author{\name Tammo Krueger \email tammokrueger@gmail.com \\
       \name Danny Panknin \email panknin@cs.tu-berlin.de\\
       \name Mikio Braun \email mikio.braun@tu-berlin.de\\
       \addr Technische Universit\"at Berlin\\
       Machine Learning Group\\
       Marchstr. 23, MAR 4-1\\
       10587 Berlin, Germany\\
}

\maketitle

\begin{abstract}
  With the increasing size of today's data sets, finding the right
  parameter configuration in model selection via cross-validation can
  be an extremely time-consuming task. In this paper we propose an
  improved cross-validation procedure which uses nonparametric
  testing coupled with sequential analysis to determine the best
  parameter set on linearly increasing subsets of the data. By
  eliminating underperforming candidates quickly and keeping promising
  candidates as long as possible, the method speeds up the computation
  while preserving the power of the full cross-validation. Theoretical
  considerations underline the statistical power of our procedure. The
  experimental evaluation shows that our method reduces the
  computation time by a factor of up to 120 compared to a full
  cross-validation with a negligible impact on the accuracy.
\end{abstract}


\section{Introduction}

Model selection by cross-validation is a de-facto standard in applied
machine learning to tune parameter configurations of machine learning
methods in supervised learning settings (see \citealt{Mosteller68,
  stone74, geisser75} and also \citealt{Arlot2010} for a recent and
extensive review of the method). Part of the data is held back and
used as a test set to get a less biased estimate of the true
generalization error. Cross-validation is computationally quite
demanding, though. Doing a full grid search on all possible
combinations of parameter candidates quickly takes a lot of time, even
if one exploits the obvious potential for parallelization.

Therefore, cross-validation is seldom executed in full in practice,
but different heuristics are usually employed to speed up the
computation. For example, instead of using the full grid, local search
heuristics may be used to find local minima in the test error (see for
instance \citealt{Kohavi95,Bengio00,Keerthi06}). However,
in general, as with all local search methods, no guarantees can be
given as to the quality of the found local minima. Another frequently
used heuristic is to perform the cross-validation on a subset of the
data, and then train on the full data set to get the most accurate
predictions. The problem here is to find the right size of the
subset: If the subset is too small and cannot reflect the
true complexity of the learning problem, the configurations selected by
cross-validation will lead to underfitted models. On the other hand, a
too large subset will take longer for the cross-validation to finish.

Effective use of model selection heuristics requires both an
experienced practitioner and familiarity with the data set. However,
as we will discuss in more depth below, the effect of taking subsets
on the estimated generalization error is more manageable: Given
increasing subsets of the data, the test errors converge to the values
on the full data set for each parameter configuration, but the
parameter configuration achieving the minimum test error will converge
much faster. Thus, using subsets in a systematic way opens up a
promising way to speed up the model selection process, since training
models on smaller subsets of the data is much more
time-efficient. During this process care has to be taken when an
increase in available data suddenly reveals more structure in the
data, leading to a change of the optimal parameter
configuration. Still, as we will discuss in more depth, there are ways
to guard against such change points, making the heuristic of taking
subsets a more promising candidate for an automated procedure.

In this paper we will propose a method which speeds up
cross-validation by considering subsets of increasing size. By
removing clearly underperforming parameter configurations on the way
this leads to a substantial saving in total computation time as
sketched in Figure~\ref{fig:ConceptualTimeConsumption}. In order to
account for possible change points, sequential testing \citep{wald47}
is adapted to control a \emph{safety zone}, roughly speaking, a
certain number of allowed failures for a parameter configuration; at
the same time this framework allows for dropping clearly
underperforming configurations. Finally, we add a stopping criterion
to watch for early convergence of the process to further speed up the
computation.  The resulting method thus consumes less time and space
than a full grid cross-validation procedure at no significant loss in
accuracy.  We prove certain theoretical properties about its
optimality, yet, this procedure relies on the availability of a vast
amount of data to guide the decision process into a stable region
where each configuration sees enough data to show its real
performance.

In the following, we will first discuss the effects of taking subsets
on learners and cross-validation (Section~\ref{sec:cvsubsets}),
discuss related work in Section~\ref{sec:relatedWork}, present our
method Fast Cross-Validation via Sequential Testing (CVST,
Section~\ref{sec:fastcv}), state the theoretical properties of the
method (Section~\ref{sec:theory}) and finally evaluate our method on
synthetic and real-world data sets in
Section~\ref{sec:experiments}. Section~\ref{sec:related} gives an
overview of possible extensions and Section~\ref{sec:conclusion}
concludes the paper. The impatient practitioner may skip some
theoretical treatments and focus on the self-contained
Section~\ref{sec:fastcv} describing the CVST algorithm and its
evaluation in Section~\ref{sec:experiments}. To ease the reading
process we collected our notational conventions in
Table~\ref{tab:symbols}.

\begin{figure}[t]
  \begin{tabular}[h]{c|c}
    \bf 5-fold CV & \bf CVST\\
\includegraphics[width=.45\textwidth]{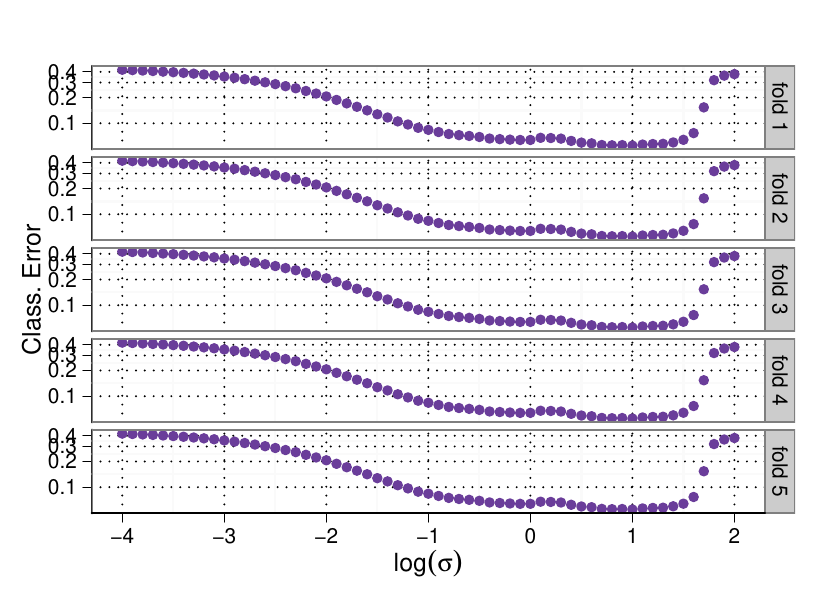} &
\includegraphics[width=.45\textwidth]{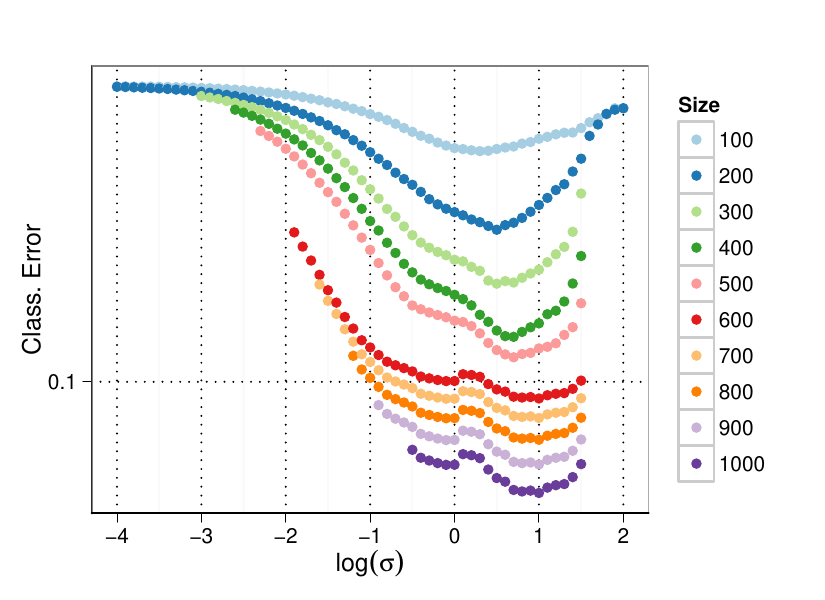}\\
\includegraphics[width=.45\textwidth]{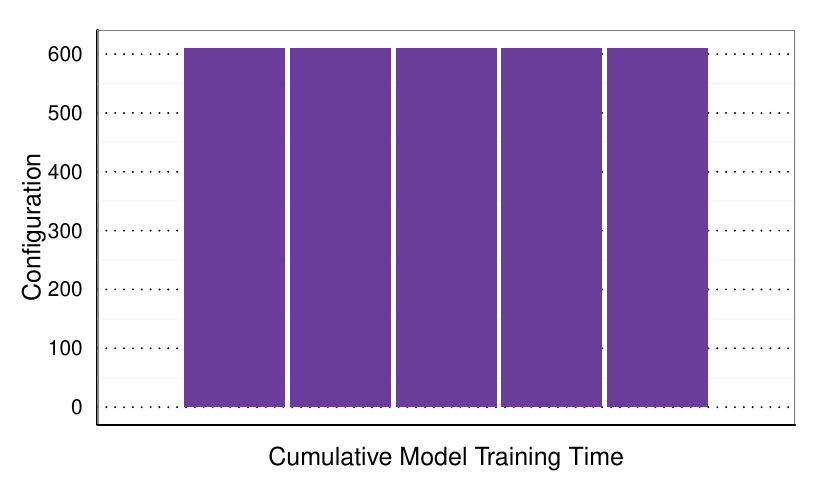} &
\includegraphics[width=.45\textwidth]{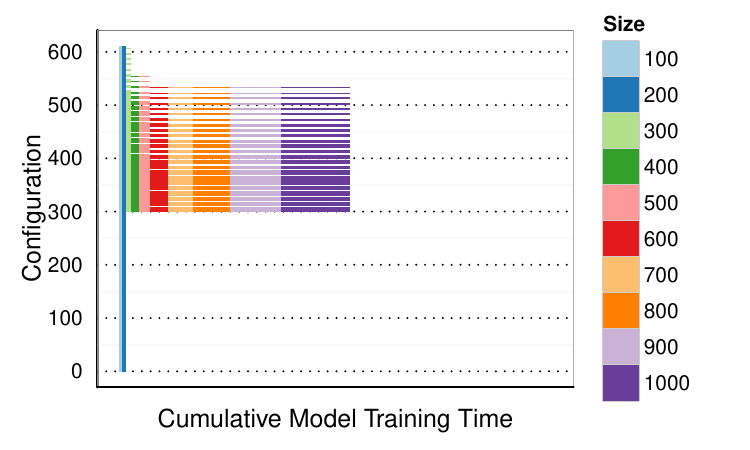}\\
  \end{tabular}
\begin{center}
\end{center}
\caption{Performance of a 5-fold cross-validation (CV, left) and fast
  cross-validation via sequential testing (CVST, right): While the CV
  has to calculate the model for each configuration (here: $\sigma$ of
  a Gaussian kernel) on the full data set, the CVST algorithm uses
  increasing subsets of the data and drops significantly
  underperforming configurations in each step (upper panels), resulting in a drastic
  decrease of total calculation time (sum of colored area in lower panels).}
\label{fig:ConceptualTimeConsumption}
\end{figure}

\begin{table}[htb]
\begin{center}
\begin{tabular}{lp{.65\textwidth}}
  \toprule
  \bf Symbol & \bf Description\\
  \midrule
  $d_i = (X_i, Y_i) \in \sX\times\sY$ & Data points\\
  $N$ & Total data set size\\
  $g : \sX\mapsto\sY$ & Learned predictor\\
  $\ell: \sY \times \sY \mapsto \real$ & Loss function\\
  $R(g)= E[\ell(g(X), Y)]$ & Risk of predictor $g$\\
  $c$ & Configuration of learner\\
  $C$ & Finite set of examined configurations\\
  $g_n(c) : \sX\mapsto\sY $ & Predictor learned on $n$ data points for configuration
  $c$\\
  $c^*$ & Overall best configuration\\
  $c^*_n$ & Best configuration for models based on $n$ data points\\
  $s$ & Current step of CVST procedure\\
  $\steps$ & Total number of  steps\\
  $\Delta=N/\steps$ & Increment of model size\\
  $P_p$ & Pointwise performance matrix\\
  $P_S$ & Overall performance matrix of dimension $|C| \times S$\\
  $T_S$ & Trace matrix of dimension $|C| \times S$\\
  \stoppingWindow & Size of early stopping window\\
  $\alpha, \alpha_l, \beta_l$ & Significance levels\\
  $\pi$ & Success probability of a binomial variable\\
  \bottomrule
\end{tabular}
\end{center}
  \caption{List of symbols}
  \label{tab:symbols}
\end{table}

\section{Cross-Validation on Subsets}
\label{sec:cvsubsets}

Our approach is based on taking subsets of the data to speed up
cross-validation. For this approach to work well, we need
that the minima of the test errors give reliable estimates of the true
performance already for small subsets. In this section, we will
discuss the setting and motivate our approach. The goal is to
understand which effects lead to making the estimates reliable.

Let us first introduce some notation: Assume that our $N$ training
data points $d_i$ are given by input/output pairs $d_i = (X_i,
Y_i)\in\sX\times\sY$ drawn i.i.d.~from some probability distribution
$P$ on $\sX\times\sY$. We assume an example-wise loss
function $\ell\colon \sY\times\sY\to\real$ so that the overall
error or expected risk of a predictor $g\colon \sX\to\sY$ is given by
$R(g) = E[\ell(g(X), Y)]$ where $(X, Y)\sim P$. For some finite set of
possible parameter configurations $C$, let $g_n(c)$ be the predictor
learned for parameter $c\in C$ from the first $n$ training examples.

The core procedure in a cross-validation approach is to train
predictors for each $c$ and consider their test error. Denote by
$g_n(c)$ the predictor obtained by training on the first $n$ points of
the training data for parameter $c$. We wish to study whether this
error converges as $n$ grows. Let us denote by $c^*_n$ a configuration
optimal for subset size $n$:
\begin{equation*}
  R(g_n(c^*_n)) = \min_{c\in C} R(g_n(c)).
\end{equation*}
We will ignore cases where there are multiple minima $c^*_n$, because
we are only interested in the test error achieved, not the location of
the minimum.

In cross-validation, the true test error is not known and estimated by
the empirical error on an independent test set. For the sake of
simplicity, we will consider the true test error nevertheless for the
remainder of this section. In our experience, the effects discussed
below also hold for cross-validation, because the estimation error is
small and does not create a systematic distortion of the choice of
configuration.

Since we want to infer the performance of the predictor on the full
training set based on its performance on a subset, we need that the
errors are similar for a fixed configuration $c$ as the size of the
subset approaches the full training set size. A necessary condition
for this to hold in general is that $R(g_n(c))$ converges as $n$ tends
to infinity. Luckily, this holds for most existing learning methods
(see Appendix~\ref{app:cond} for some examples).  A counter example is
the case of $k$-nearest neighbor with fixed $k$. Training with $k=10$
leads to quite different predictions on data sets of size $100$
compared to, say, $10{,}000$. More discussion can be found below in
Section~\ref{sec:discussionanalyses}.

We are interested in the difference in errors between the best
parameter configuration learned on the subset of size $n$, and on the
full data set $N$, that is, $R(g_n(c^*_n)) - R(g_N(c^*_N))$. This
error can be bounded by considering the difference between $R(g_n(c))$
and $R(g_N(c))$ uniformly over the whole configuration set $C$. If the
learner itself converges it is trivial to show that the errors also
converge for finite parameter configuration sets $C$. On the other
hand, uniform convergence is quite a strong requirement, since it
requires that the test errors also converge for suboptimal
configurations. In particular for parameter configurations $c$ which
correspond to complex models, $g_n(c)$ may continue to improve right
up to the full number $N$ of data points.

\begin{figure}
  \centering
  \subfigure[\label{fig:ErrorExample}]{\includegraphics[width=.49\textwidth]{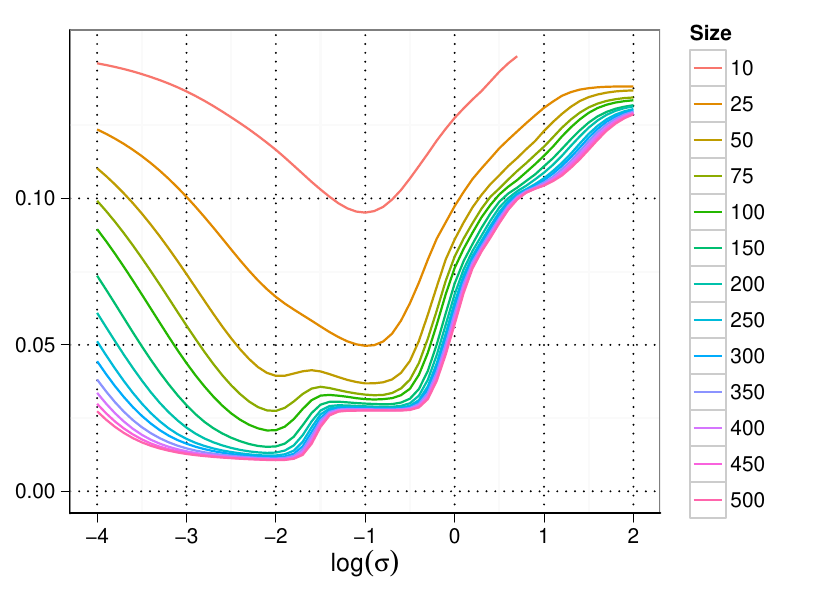}}
  \subfigure[\label{fig:ErrorAnnotated}]{\includegraphics[width=.49\textwidth]{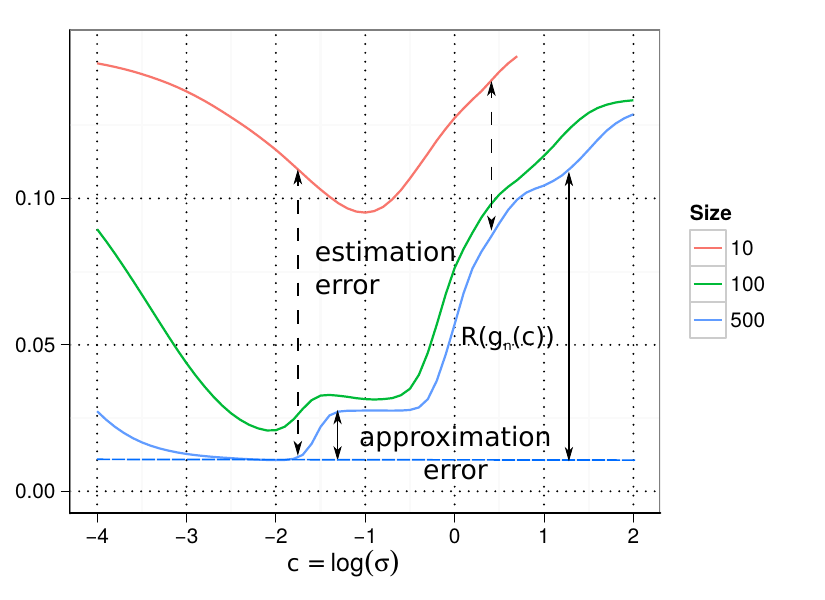}}
  \caption{Test error of an SVR model on the \emph{noisy sinc} data
    set introduced in Section~\ref{sec:art}. We can observe a shift of
    the optimal $\sigma$ of the Gaussian kernel to the fine-grained
    structure of the problem, if we have seen enough data. In Figure
    (b), approximation error is indicated by the black solid line, and
    the estimation error by the black dashed line. The minimal risk is
    shown as the blue dashed line. One can see that uniform
    approximation of the estimation error is not the main driving
    force, instead, the decay of the approximation error with smaller
    kernel widths together with an increase of the estimation error at
    small kernel widths makes sure that the minimum converges
    quickly.}
\end{figure}

So while uniform convergence seems a sufficient condition, let us look
at a concrete example to see whether uniform convergence is a
necessary condition for convergence of the
minima. Figure~\ref{fig:ErrorExample} shows the test errors for a
typical example. We train a support vector regression model (SVR) on
subsets of the full training set consisting of 500 data points. The
data set is the \emph{noisy sinc} data set introduced in
Section~\ref{sec:art}. Model parameters are the kernel width $\sigma$
of the Gaussian kernel used and the regularization parameter, where
the results shown are already optimized over the regularization
parameter for the sake of simplicity.

We see that the minimum converges rather quickly, first to the plateau
of $\log(\sigma)\in [-1.5, -0.3]$ approximately, and then towards the
lower one at $[-2.5, -1.7]$, which is also the optimal one at training
set size $n = 500$. We see that uniform convergence is not the main
driving force. In fact, the errors for small kernel widths are still
very far apart even when the minimum is already converged.

In the following, it is helpful to continue the discussion using the
concepts of estimation error and approximation error. We assume that
the learner is trained by picking the model which minimizes the
empirical risk over some hypothesis set $\hypocl$. Let us denote this
predictor as $g_n^*$. In this setting, one can write the difference
between the expected risk of the predictor $R(g_n^*)$ and the Bayes
risk $R^*$ as follows (see Section~12.1 in~\citealt{devroy96} or
Section~2.4.3 in~\citealt{mohri12}):
\begin{equation*}
  R(g_n^*) - R^* 
  = \underbrace{\left(R(g_n^*) - \inf_{g\in \hypocl}R(g)\right)}
     _{\text{estimation error}}
  + \underbrace{\left(\inf_{g\in \hypocl}R(g) - R^*)\right)}
     _{\text{approximation error}}.
\end{equation*}
The estimation error measures how far the chosen model is from the one
which would be asymptotically optimal, while the approximation error
measures the difference in risk between the best possible model in the
hypothesis class and the true function.

Using this decomposition, we can interpret the figure as follows (see
Figure~\ref{fig:ErrorAnnotated}): The kernel width controls the
approximation error. For $\log(\sigma) \ge -1.8$, the resulting
hypothesis class is too coarse to represent the function under
consideration. It becomes smaller until it reaches the level of the
Bayes risk as indicated by the dashed blue line. For even larger
training set sizes, we can assume that it will stay on this level even
for smaller kernel sizes.

The difference between the blue line and the upper lines shows the
estimation error. The estimation error has been extensively studied in
statistical learning theory and is known to be linked to different
notions of complexity like VC-dimension \citep{vapnik98},
fat-shattering dimension \citep{bartlett96}, or the norm in the
reproducing kernel Hilbert space (RKHS) \citep{evgeniou99}.  A typical
result shows that the estimation error can be bounded by terms of the
form
\begin{equation*}
  R(g^*_n) - \inf_{g\in\hypocl} R(g) \le O\left(\sqrt{\frac{d(\hypocl) \log n}{n}}\right),
\end{equation*}
where $d(\hypocl)$ is some notion of complexity of the underlying
hypothesis class, and the bound holds with high probability. For our
figure, this means that we can expect the estimation error to become
larger for smaller kernel widths.

If we image the parameter configurations ordered according to their
complexity, we see that for parameter configurations with small
complexity (that is, large kernel width), the approximation error will
be high, but the estimation error will be small. At the same time, for
parameter configurations with high complexity, the approximation error
will be small, even optimal, but the estimation error will be large,
although it will decay with increasing training set size. In
combination, the estimates at smaller training set sizes tend to
underestimate the true model complexity, but as the estimation
error decreases and becomes small compared to the approximation error,
the minimum also converges to the true one. The fact that the
estimation error is larger for more complex models acts as a guard to
choose too complex models. The estimation error for models which have
higher complexity than the optimal one can effectively be
ignored. Therefore, we can expect much faster convergence than given
by a uniform error bound, which is, however, highly data dependent.

Unfortunately, existing theoretical results are not able to bound the
error sufficiently tightly to make these arguments more exact. In
particular, the speed of the convergence on the minimum hinges on a
tight lower bound on the approximation error, and a realistic upper
bound on the estimation error. Approximation errors have been studied
for example in the papers by \cite{smale03} and \cite{steinwart07},
but the papers only prove upper bounds, and the rates are also
worst-case rates which are likely not close enough to the true
errors. A more formal study of the effects discussed above is
therefore the subject of future work.

On the other hand, the mechanisms which lead to fast convergence of
the minimum are plausible when looking at concrete examples as we did
above. Therefore, we will assume in the following that the location of
the best parameter configuration might initially change but then
become more or less stable quickly. Note that we do not claim that the
speed of this convergence is known. Instead, we will use sequential
testing to introduce a \emph{safety zone} which will be as large as
possible to ensure that our method is robust against these initial
changes and good configurations survive till final stable regime.

\section{Related Work}
\label{sec:relatedWork}
Using statistical tests and the sequential analysis framework in order
to speed up learning has been the topic of several lines of
research. However, the existing body of work mostly focuses on
reducing the number of test evaluations, while we focus on the overall
process of eliminating candidates themselves. To the best of our
knowledge, this is a new concept and can apparently be combined with
the already available racing techniques to further reduce the total
calculation time.

\cite{Maron94,Maron1997} introduce the so-called \emph{Hoeffding
  Races} which are based on the nonparametric Hoeffding bound for the
mean of the test error. At each step of the algorithm a new test point
is evaluated by all remaining models and the confidence intervals of
the test errors are updated accordingly. Models whose confidence
interval of the test error lies outside of at least one interval of a
better performing model are dropped. In a similar vein \cite{Zheng13}
have applied this concept to cross-validation and improve this
approach by using paired t-test and power analysis to control both the
false positive and false negative rate. \cite{Chien1995,Chien1999}
devise a similar range of algorithms using concepts of PAC learning
and game theory: Different hypotheses are ordered by their expected
utility according to the test data the algorithm has seen so far. As
for Hoeffding Races, the emphasis in this approach lies on reducing
the number of evaluations. Thus, the application domain for these kind
of algorithms is best suited where the evaluation of a data point
given a learned model is costly. Since this approach expects that a
model is fully trained before its evaluation, the direct utilization
of racing algorithms for model selection would result in a procedure
similar to a one-fold cross-validation: First learn a model on one
half of the data and do the time efficient evaluation as described
above on the other half. Obviously, this would yield a maximal
relative time improvement of $k$ compared to standard $k$-fold
cross-validation since we learn one model instead of the $k$ for
$k$-fold cross-validation. Yet, the orthogonality of this approach to
the CVST procedure could be utilized in each step and for each
remaining configuration to further increase the runtime benefits by
minimizing the necessary evaluations of a model for determining
whether it belongs to the top configurations or not.

This concept of racing is further extended by \cite{Domingos2001}: By
introducing an upper bound for the learner's loss as a function of the
examples, the procedure allows for an early stopping of the learning
process, if the loss is nearly as optimal as for infinite
data. \cite{Birattari2002} apply racing in the domain of evolutionary
algorithms and extend the framework by using the Friedman test to
filter out non-promising configurations. While \cite{Bradley08} use
similar concepts in the context of boosting (FilterBoost),
\cite{Mnih2008} introduce the empirical Bernstein Bounds to extend
both the FilterBoost framework and the racing algorithms. In both
cases the bounds are used to estimate the error within a specific
$\epsilon$-region with a given probability. \cite{pelossof09} use the
concept of sequential testing to speed up the boosting process by
controlling the number of features which are evaluated for each
sample. In a similar fashion this approach is used in
\cite{pelossof10} to increase the speed of the evaluation of the
perceptron and in \cite{pelossof11} to speed up the Pegasos
algorithm. \cite{stanski12} uses a partial leave-one-out evaluation of
model performance to get an estimate of the overall model performance,
which is used to pick the most probable best model. These racing
concepts are applied in a wide variety of domains like reinforcement
learning \citep{Heidrich-Meisner2009} and timetabling
\citep{birattari09} showing the relevance and practical impact of the topic.

Recently, Bayesian optimization has been applied to the problem of
hyper-parameter optimization of machine learning
algorithms. \cite{Bergstra11} use the sequential model-based global
optimization framework (SMBO) and implement the loss function of an
algorithm via hierarchical Gaussian processes. Given the previously
observed history of performances, a candidate configuration is selected
which minimizes this historical surrogate loss function. Applied to
the problem of training deep belief networks this approach shows
superior performance over random search strategies. \cite{snoek12}
extend this approach by including timing information for each
potential model, i.e., the cost of learning a model and optimizing the
expected improvement per seconds leads to a global optimization in
terms of wall-clock time. \cite{Thornton12} apply the SMBO framework
in the context of the WEKA machine learning toolbox: The so-called
Auto-WEKA procedure not only finds the optimal parameter for a
specific learning problem but also searches for the most suitable
learning algorithm. Like the racing concepts, these Bayesian
optimization approaches are orthogonal to the CVST approach and could
be combined to speed up each step of the CVST loop.

On first sight, the multi-armed bandit problem
\citep{berry85,cesa-bianchi06} also seems to be related to the problem
here in another way: In the multi-armed bandit problem, a number of
distributions are given and the task is to identify the distribution
with the largest mean from a chosen sequence of samples from the
individual distributions. In each round, the agent chooses one
distribution to sample from and typically has to find some balance
between exploring the different distributions, rejecting distributions
which do not seem promising and focusing on a few candidates to get
more accurate samples.

This looks similar to our setting where we also wish to identify
promising candidates and reject underperforming configurations early
on in the process, but the main difference is that the multi-armed
bandit setting assumes that the distributions are fixed whereas we
specifically have to deal with distributions which change as the
sample size increases. This leads to the introduction of a safety
zone, among other things. Therefore, the multi-armed bandit setting is
not applicable across different sample sizes.  On the other hand, the
multi-armed bandit approach is a possible extension to speed up the
computation within a fixed training size similar to the Hoeffding
races already mentioned above.

\section{Fast Cross-Validation via Sequential Testing (CVST)}
\label{sec:fastcv}

Recall from Section~\ref{sec:cvsubsets} that we have a data set
consisting of $N$ data points $d_i = (X_i, Y_i) 
\in \sX\times\sY$ which we assume to be drawn i.i.d.~from
$P$. We have a learning algorithm which
depends on several parameters collected in a configuration $c \in
C$. The goal is to select the configuration $c^*$ out of all possible
configurations $C$ such that the learned predictor $g$ has the best
generalization error with respect to some loss function $\ell:
\mathcal{Y}\times\mathcal{Y}\to \real$.

Our approach attempts to speed up the model selection process by
learning just on subsamples of size $n := s \frac{N}{\steps} =
s\Delta$ for $1\le s\le \steps$ where $\steps$ is the maximal number
of steps the CVST algorithm should run. The procedure starts with the
full set of configurations and eliminates clearly underperforming
configurations at each step $s$ based on the performances observed in
steps 1 to $s$. The main loop of Algorithm~\ref{alg:complete} on page
\pageref{alg:complete} executes the following parts at each step $s$:

\begin{dingautolist}{202}
\item The procedure learns a model on the first $n$ data points for
  the remaining configurations and stores the test errors on the
  remaining $N-n$ data points in the pointwise performance matrix
  $P_p$ (Lines~\ref{alg:perf1}-\ref{alg:perf2}). This matrix $P_p$ is
  used on Lines~\ref{alg:best1}-\ref{alg:best2} to estimate the top
  performing configurations via robust testing (see
  Algorithm~\ref{alg:topConfigurations}) and saves the outcome as a
  binary ``top or flop'' scheme accordingly.
\item The procedure drops significant loser configurations along the
  way (Lines~\ref{alg:loser1}-\ref{alg:loser2} and
  Algorithm~\ref{alg:isFlopConfiguration}) using tests from the
  sequential analysis framework.
\item Applying robust, distribution free testing techniques allows for
  an early stopping of the procedure when we have seen enough data for
  a stable parameter estimation (Line~\ref{alg:early} and
  Algorithm~\ref{alg:similarPerformance}).
\end{dingautolist}

In the following we will discuss the individual steps in the algorithm
and formally define the notations used. A conceptual overview of one
iteration of the procedure is depicted in Figure~\ref{fig:overview}
for reference. Additionally, we have released a software package on
CRAN named \texttt{CVST} which is publicly available via all official
CRAN repositories and also via GitHub
(\url{https://github.com/tammok/CVST}). This package contains the CVST
procedure and all learners used in Section~\ref{sec:experiments} ready
for use.

\newcommand{\Assign}{\ensuremath{\leftarrow}}
\begin{algorithm}[tbp]
   \caption{CVST Main Loop}
   \label{alg:complete}
   \begin{algorithmic}[1]
   \Function{CVST}{$d_1, \dots, d_N$, \steps, $C$, $\alpha$, $\beta_l$, $\alpha_l$, \stoppingWindow}
   \State $\Delta \Assign \text{N} / \steps$ \Comment{Initialize
     subset increment}
   \State $\modelsize \Assign \Delta$  \Comment{Initialize model size}
   \State test \Assign {\sc getTest}(\steps, $\beta_l$, $\alpha_l$)
   \Comment{Get sequential test}
   \State $\forall \text{s} \in \{1, \dots, \steps\}, c \in C: T_S[c,
   s] \Assign 0$
   \State $\forall \text{s} \in \{1, \dots, \steps\}, c \in C: P_S[c, s] \Assign \text{NA}$
   \State $\forall c \in C:$ isActive[$c$] \Assign {\bf true}
   \For{$s \Assign 1$ {\bf to} \steps}
   \State $\forall \text{i} \in \{1, \dots, N-n\}, c \in C: P_p[c, i]
   \Assign \text{NA}$ 
   \For{$c \in C$} \label{alg:perf1}
   \If{isActive[$c$]}
   \State $g = g_n(c)$ \Comment{Learn model on the first $n$ data points}
   \State $\forall i \in \{1, \dots, N - n\} : P_p[c, i] \Assign
   \ell(g(x_{n+i}), y_{n+i})$ \Comment{Evaluate on the rest}
   \State $P_S[c, s] \Assign \frac{1}{N-n}\sum_{i=1}^{N-n}P_p[c,
   i]$ \label{alg:perf2} \Comment{Store mean performance}
   \EndIf
   \EndFor
   \State index$_\text{top}$ \Assign {\sc topConfigurations}($P_p,
   \alpha$) \Comment{Find the top configurations} \label{alg:best1}
   \State $T_S$[index$_\text{top}$, $s$] \Assign 1 
   \Comment{And set entry in trace matrix}\label{alg:best2}
   \For{$c \in C$} \label{alg:loser1}
   \If{isActive[$c$] and {\sc isFlopConfiguration}($T_S[c, 1:s]$, s, \steps, $\beta_l$, $\alpha_l$)} \label{alg:isFlop}
   \State isActive[c] \Assign {\bf false}  \Comment{De-activate flop
     configuration}\label{alg:loser2}
   \EndIf
   \EndFor
   \If{{\sc similarPerformance}($T_S$[isActive, $(s-\stoppingWindow+1):s$], $\alpha$)} \label{alg:early}
   \State {\bf break}
   \EndIf
   \State $\modelsize \Assign \modelsize + \Delta$ \label{alg:inc}
   \EndFor
   \State \Return {\sc selectWinnner}($P_S$,
   isActive, \stoppingWindow, $s$) \label{alg:winner} 
   \EndFunction
   \end{algorithmic}
\end{algorithm}

\newcommand{\algorithmicbreak}{\textbf{break}}
\newcommand{\Break}{\State \algorithmicbreak}
\begin{algorithm}[tbp]
   \caption{Find the top configurations via iterative testing}
   \label{alg:topConfigurations}
   \begin{algorithmic}[1]
   \Function{topConfigurations}{$P_p, \alpha$}
   \State $\forall i \in \{1, \dots, C\} : P_m[k] \Assign \frac{1}{N-n}\sum_{j=1}^{N-n}P_p[k, j]$ 
   \State $\text{index}_\text{sort} \Assign$ {\sc sortIndexDecreasing}($P_m$)
   \State $\widetilde{P}_p = P_p[\text{index}_\text{sort}, ]$ \Comment{Sort $P_p$
     according to the mean performance}
   \State $K \Assign ${\sc which}({\sc isNA}$(P_m)) - 1$ \Comment{$K$
     is the number of active configurations}
   \State $\tilde \alpha = \alpha / (K-1)$\Comment{Bonferroni
     correction for $K-1$ potential tests}
   \For{$k \in \{2, \dots, K\}$}
   \If{is classification task}\Comment{Choose according test}
   \State $p \Assign $ {\sc cochranQTest}$(\widetilde{P}_p[1:k, ])$ \label{alg:topConfigurationsClass}
   \Else
   \State $p \Assign $ {\sc friedmanTest}$(\widetilde{P}_p[1:k, ])$ \label{alg:topConfigurationsReg}
   \EndIf
   \If{$p \leq \tilde \alpha$}\Comment{We found a significant effect}
   \Break\Comment{so the $k-1^\text{th}$ preceding configurations are the top ones}
   \EndIf
   \EndFor
   \State \Return $\text{index}_\text{sort}[1:(k-1)]$
   \EndFunction
   \end{algorithmic}
\end{algorithm}

\begin{algorithm}[tb]
   \caption{Check for flop configurations via sequential testing}
   \label{alg:isFlopConfiguration}
   \begin{algorithmic}[1]
   \Function{isFlopConfiguration}{$T$, s, \steps, $\beta_l$, $\alpha_l$}
   \State $\pi_0 \Assign 0.5; \pi_1  \Assign \frac{1}{2} \BoundPi$
   \State $a \Assign \frac{\Boundb}{\Boundp - \Boundq}$
   \State $b \Assign \frac{\BoundQ}{\Boundp - \Boundq}$
   \State \Return $\sum_{i=1}^s T_i \leq a + b s$
   \EndFunction
   \end{algorithmic}
\end{algorithm}

\begin{algorithm}[tb]
   \caption{Compare performance of remaining configurations}
   \label{alg:similarPerformance}
   \begin{algorithmic}[1]
   \Function{similarPerformance}{$T_S$, $\alpha$}
   \State $p \Assign $ {\sc cochranQTest}$(T_S)$
   \State \Return $p \leq \alpha$
   \EndFunction
   \end{algorithmic}
\end{algorithm}

\begin{algorithm}[tb]
   \caption{Select the winning configuration out of the remaining ones}
   \label{alg:selectWinnner}
   \begin{algorithmic}[1]
   \Function{selectWinnner}{$P_S$, isActive, \stoppingWindow, $s$}
   \State $\forall \text{i} \in \{1, \dots, s\}, c \in C: R_S[c, i]
   \Assign \infty$
   \For{$i \in \{1, \dots, s\}$}
   \For{$c \in C$}
   \If{isActive$[c]$}
   \State $R_S[c, i] = \text{rank}(P_S[c, i], P_S[,
   i])$\Comment{Gather the rank of $c$ in step $i$}
   \EndIf
   \EndFor
   \EndFor
   \State $\forall c \in C: M_S[c] \Assign \infty$
   \For{$c \in C$}
   \If{isActive$[c]$}
   \State $M_S[c] \Assign
   \frac{1}{\stoppingWindow}\sum_{i=s-\stoppingWindow + 1}^sR_S[c, i]$
   \Comment{Mean rank for the last $\stoppingWindow$ steps}
   \EndIf
   \EndFor
   \State \Return {\sc whichMin}$(M_S)$ \Comment{Return
     configuration with minimal mean rank}
   \EndFunction
   \end{algorithmic}
\end{algorithm}

\JMLRfigure{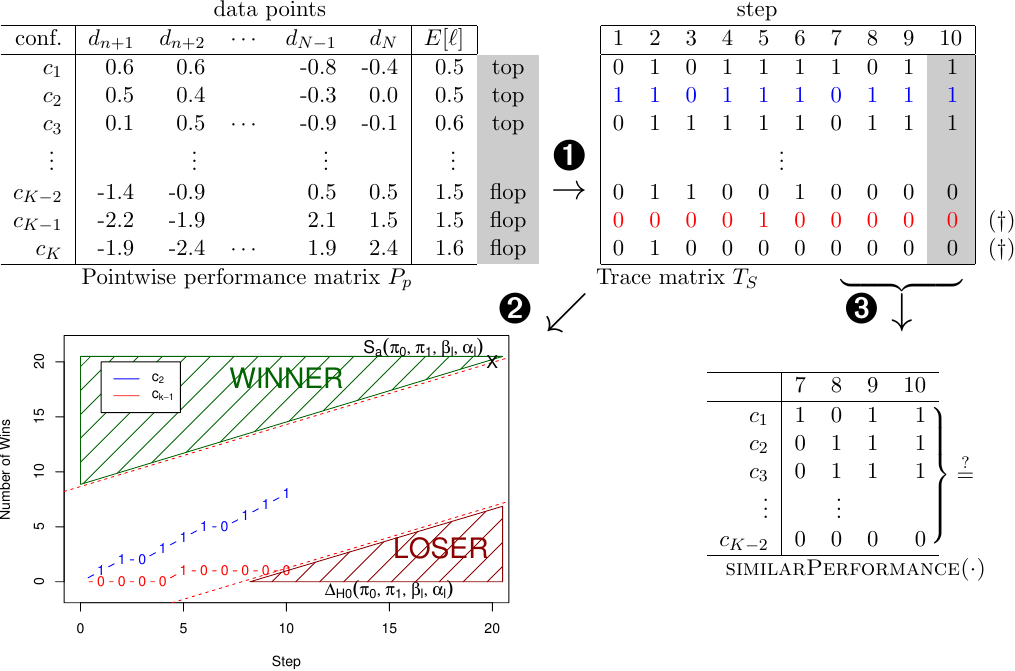}{1}{One step of CVST. Shown is the situation in
  step $s = 10$. \ding{202} A model based on the first \modelsize~data
  points is learned for each configuration ($c_1 \text{ to }
  c_K$). Test errors are calculated on the remaining data ($d_{n+1}
  \text{ to } d_N$) and transformed into a binary performance
  indicator via robust testing. \ding{203} Traces of configurations
  are filtered via sequential analysis ($c_{K-1}$ and $c_K$ are
  dropped). \ding{204} The procedure checks whether the remaining
  configurations perform equally well in the past and stops if this is
  the case. See Appendix~\ref{sec:exampleRun} for a complete example
  run.}

\subsection{Robust Transformation of Test Errors}
\label{sec:transform}
To robustly transform the performance of configurations into the
binary information whether it is among the top-performing
configurations or turns out to be a flop, we rely on distribution-free
tests. The basic idea is to calculate the performance of a given
configuration on data points not used during learning and store this
information in the pointwise performance matrix $P_p$. Then we find
the group of best configurations by first ordering them according to
their mean performance in this step and then compare in a stepwise
fashion whether the pointwise performance matrix $P_p$ of a given
subset of the configurations are significantly different.

We give an example of this procedure by the situation depicted in
Figure~\ref{fig:overview} with $K$ remaining configurations $c_{1},
c_{2}, \dots, c_{K}$ which are ordered according to their mean
performances (i.e., sorted ascending with regard to their expected
loss). We now want to find the smallest index $k \leq K$, such that
the configurations $c_{1}, c_{2}, \dots, c_{k}$ all show a similar
behavior on the remaining data points $d_{n+1}, d_{n+2}, \dots, d_N$
not used in the current model learning process based on a statistical
test.

The rationale behind our comparison procedure is three-fold: First, by
ordering the configurations by the mean performances we start with the
comparison of the currently best performing configurations
first. Second, by using the first $n := s \Delta$ data points for the
model building and the remaining $N-n$ data points for the estimation
of the average performance of each configuration, we compensate the
error introduced by learning on smaller subsets of the data by better
error estimates on more data points. I.e., for small $s$ we will learn
the model on relatively small subsets of the overall available data
while we estimate the test error on relatively large portions of the
data and vice versa. Third, by applying test procedures directly on
the error estimates of individual data points we exploit a further
robustifying pooling effect: If we have outliers in the testing data,
all models will be affected by this and therefore the overall testing
result will not be affected. We will see in the evaluation section
that all these effects are indeed helpful for an overall good
performance of the CVST algorithm. 

To find the top performing configurations for step $s$ we look at the
outcome of the learned model for each configuration, i.e., we
subsequently take the rows of the pointwise performance matrix $P_p$
into account and apply either the Friedman test \citep{friedman37} for
regression experiments or the Cochran's Q test \citep{cochran50} to
see whether we observe statistically significant differences between
configurations (see Appendix~\ref{sec:tests} for a summary of these
tests). In essence these robust tests check whether the performance
outcomes of a subset of the configurations show significant
differences, i.e., in our case behave differently in terms of overall
best performance. The assumption here is that the mean performance of
a configuration is a good, yet wiggly estimator of its overall
performance. By subsequently checking the finer-grained outcome of the
models on the individual data points we want to find the breakpoint
where the overall top-performing configurations for this step are
separated from the rest of the configurations which will show a
significantly different behavior on the individual data points.

More formally, the function {\sc topConfigurations} described in
Algorithm~\ref{alg:topConfigurations} takes the pointwise performance
matrix $P_p$ as input and rearranges the rows according to the mean
performances of the configurations yielding a matrix
$\widetilde{P}_p$. Now for $k \in \{2, 3, \dots, K\}$ we check,
whether the first $k$ configurations show a significantly different
effect on the $N-n$ data points. This is done by executing either the
Friedman test or the Cochran's Q test on the submatrix
$\widetilde{P}_p[1:k, 1:(N-n)]$ with the pre-specified significance
level $\alpha$. If the test does not indicate a significant difference
in the performance of the $k$ configurations, we increment $k$ by one
and test again until we find a significant effect. Suppose we find a
significant effect at index $\tilde k$. Since all previous tests
indicated no significant effect for the $\tilde k - 1$ configurations
we argue that the addition of the $\tilde k^\text{th}$ configuration
must have triggered the test procedure to indicate that in the set of
these $\tilde k$ configurations is at least one configuration, which
shows a significantly different behavior than all other
configurations. Thus, we flag the configurations $1, \dots, \tilde k -
1$ as top configurations and the remaining $\tilde k, \dots, K$
configurations as flop configurations. Note that this incremental
procedure is a multiple testing situation, thus we apply the
Bonferroni correction to the calculated p-values.

For the actual calculation of the test errors we apply an incremental
model building process, i.e., the data added in each step on
Line~\ref{alg:inc} increases the training data pool for each step by a
set of size $\Delta$. This would allow online algorithms to adapt
their model also incrementally leading to even further speed
improvements. The results of this first step are collected for each
configuration in the trace matrix $T_S$ (see
Figure~\ref{fig:overview}, top right), which shows the gradual
transformation for the last 10 steps of the procedure highlighting the
results of the last test. More formally, $T_S[c, s]$ is $1$ iff
configuration $c$ is amongst the top configuration in step $s$; if $c$
is not a top configuration in step $s$, the entry $T_S[c, s]$ is $0$.

So this new column generated in step $s$ in the trace matrix $T_S$
summarizes the performance of all models learned on the first $n$ data
points in a robust way. Thus, the trace matrix $T_S$ records the
history of each configuration in a binary fashion, i.e., whether it
performed as a top or flop configuration in each step of the CVST main
loop. This leads to a robust transformation of the test errors of the
configurations which can be modeled in the next step as a binary
random variable with a success probability $\pi$ indicating whether a
configuration is amongst the top (high $\pi$) or the flop (low $\pi$)
configurations.

\subsection{Determining Significant Losers}
\label{sec:seqana}

Having transformed the test errors in a scale-independent top or flop
scheme, we can now test whether a given parameter configuration is an
overall loser. For this we represent a configuration as a binary
random variable which turns out to be a top configuration with a given
probability $\pi$. During the course of the execution of the CVST
algorithm we gather information about the behavior of each
configuration and want to estimate at each step whether the
observed behavior is more likely to be associated with a binomial
variable having a high $\pi$ meaning that it is a winning
configuration or a low one deeming it as a loser configuration. The
standard tool for this kind of task is the sequential testing of
binary random variables which is addressed in the \emph{sequential
  analysis} framework developed by \cite{wald47}. Originally it has
been applied in the context of production quality assessment (compare
two production processes) or biological settings (stop bioassays as
soon as the gathered data leads to a significant result). In this
section we focus on the general idea of this approach while
Section~\ref{sec:theory} gives details about how the CVST algorithm
deals with potential switches in the winning probability $\pi$ of a
given configuration. 

The main idea of the sequential analysis framework is the following:
One observes a sequence of i.i.d.~Bernoulli variables $B_1, B_2,
\ldots$, and wants to test whether these variables are distributed
according to the hypotheses $H_0: B_i \sim \pi_0$ or the alternative
hypotheses $H_1: B_i \sim \pi_1$ with $\pi_0 < \pi_1$ denoting the
according success probabilities of the Bernoulli variables. Both
significance levels for the acceptance of $H_1$ and $H_0$ can be
controlled via the user-supplied meta-parameters $\alpha_l$ and
$\beta_l$. The test computes the likelihood for the so far observed
data and rejects one of the hypothesis when the respective likelihood
ratio exceeds an interval controlled by the meta-parameters. It can be
shown that the procedure has a very intuitive geometric
representation, shown in Figure~\ref{fig:overview}, lower left: The
binary observations are recorded as cumulative sums at each time
step. If this sum exceeds the upper red line $L_1$, we accept $H_1$;
if the sum is below the lower red line $L_0$ we accept $H_0$; if the
sum stays between the two red lines we have to draw another sample.

Wald's test requires that we fix both success probabilities $\pi_0$
and $\pi_1$ beforehand. Since our main goal is to use the sequential
test to eliminate underperformers, we choose the parameters $\pi_0$
and $\pi_1$ of the test such that $H_1$ (a configuration wins) is
postponed as long as possible. This will allow the CVST algorithm to
keep configurations until the evidence of their performances
definitely shows that they are overall loser configurations. At the
same time, we want to maximize the area where configurations are
eliminated (region $\triangle_{H_0}$ denoted by ``LOSER'' in
Figure~\ref{fig:overview}), rejecting as many loser configurations on
the way as possible:
\begin{eqnarray}
\label{eq:seq1}
  (\pi_0, \pi_1) &=& \underset{\pi_0^\prime, \pi_1^\prime}{\operatorname{argmax}}
  ~\triangle_{H_0}(\pi_0^\prime, \pi_1^\prime, \beta_l, \alpha_l)\\
  &&\text{s.t.} ~S_a(\pi_0^\prime, \pi_1^\prime, \beta_l, \alpha_l)
  \in (\steps - 1, \steps] \notag
\end{eqnarray}
with $S_a(\cdot, \cdot, \cdot, \cdot)$ being the earliest step of
acceptance of $H_1$ marked by an X in Figure~\ref{fig:overview} and
$\steps$ denotes again the total number of steps. By using
approximations from \cite{wald47} for the expected number of steps the
test will take, if the real success probability of the underlying
process would indicate a constant winner (i.e., $\pi = 1.0$), we can
fix $S_a$ to the maximal number of steps $S$ and solve
Equation~(\ref{eq:seq1}) as follows (see Appendix~\ref{sec:SecProof}
for details):
\begin{equation}
  \pi_0 = 0.5 \wedge \pi_1  = \frac{1}{2} \BoundPi.
\label{eq:seq2}
\end{equation}
Equipped with these parameters for the sequential test, we can check
each remaining trace on Line~\ref{alg:isFlop} of
Algorithm~\ref{alg:complete} in the function {\sc isFlopConfiguration}
detailed in Algorithm~\ref{alg:isFlopConfiguration} whether it is a
statistically significant flop configuration (i.e., exceeds the lower
decision boundary $L_0$) or not.

Note that sequential analysis formally requires i.i.d.~variables. In
the CVST procedure both the independence of the top/flop variable and
the identically distributed assumption might be violated for
configurations which transform to a winner configuration later on,
thereby changing their behavior from a flop to a top
configuration. With the modeling approach taken in this step of the
CVST algorithm this would amount to a change of the underlying success
probability $\pi$ of the configuration. Thus, the assumptions of the
sequential testing framework would definitely be violated. We
accommodate for this by introducing in Section~\ref{sec:fnr} a
so-called \emph{safety zone} which acts as a safeguard against
prematurely dropping of a configuration. Note that this safety zone
can be controlled by the experimenter using the parameters $\alpha_l
\text{ and } \beta_l$ of the sequential test. If the experimenter
chooses the right safety zone the underlying success probabilities of
the configuration remain stable after the safety zone and, hence,
again will satisfy the preconditions of the sequential testing
framework. So by ensuring no premature drop of a configuration in the
safety zone we heuristically adapt the sequential test to the
potential switch of underlying success probabilities. To give a
complete account of the assumptions of the sequential analysis we will
discuss potential violations of the independence of the top/flop
variables and its implication for the CVST procedure in Section~\ref{sec:theory}.

For details of the open sequential analysis please consult
\cite{wald47} or see for instance \cite{wetherill86} for a general
overview of sequential testing procedures. Appendix~\ref{sec:SecProof}
contains the necessary details needed to implement the proposed
testing scheme for the CVST algorithm.

\subsection{Early Stopping and Final Winner}
\label{sec:early}

Finally, we employ an early stopping rule (Line~\ref{alg:early}) which
takes the last $\stoppingWindow$ columns from the trace matrix and
checks whether all remaining configurations performed equally well in
the past. In Figure~\ref{fig:overview} this submatrix of the overall
trace matrix $T_S$ is shown for a value of $\stoppingWindow = 4$ for
the remaining configurations after step 10. For the test, we again
apply the Cochran's Q test (see Appendix~\ref{sec:tests}) in the {\sc
  similarPerformance} procedure on the submatrix of $T_S$ as denoted
in
Algorithm~\ref{alg:similarPerformance}. Figure~\ref{fig:EarlyStoppingBlowups}
illustrates a complete run of the CVST algorithm for roughly 600
configurations. Each configuration marked in red corresponds to a flop
configuration and a black one to a top configuration. 
Configurations marked in gray have been dropped via the sequential
test during the CVST algorithm. The small zoom-ins in the lower part
of the picture show the last $\stoppingWindow$ remaining
configurations during each step which are used in the evaluation of
the early stopping criterion. We can see that the procedure keeps on
going if there is a heterogeneous behavior of the remaining
configurations (zoom-in is mixed red/black). When all the remaining
configurations performed equally well in the past (zoom-in is nearly
black), the early stopping test does not see a significant effect
anymore and the procedure is stopped.

\JMLRfigure{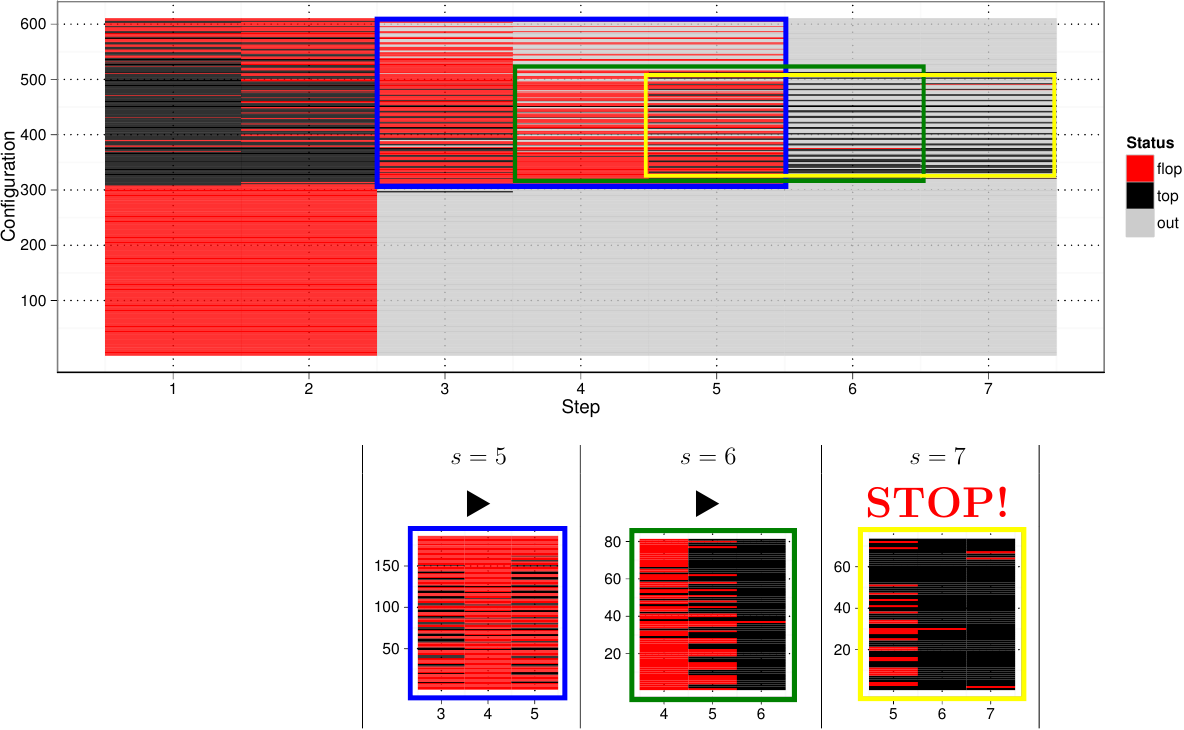}{1}{The upper plot shows a run of the
  CVST algorithm for roughly 600 configurations. At each step a
  configuration is marked as top (black), flop (red) or dropped
  (gray). The zoom-ins show the situation for step 5 to 7 without the
  dropped entries. The early stopping rule takes effect in step 7,
  because the remaining configurations performed equally well during step 5 to 7.}

Finally, in the procedure {\sc selectWinner}, Line~\ref{alg:winner}
and Algorithm~\ref{alg:selectWinnner}, the winning configuration is
picked from the configurations which have survived all steps as
follows: For each remaining configuration we determine the rank in a
step according to the average performance during this step. Then we
average the rank over the last \stoppingWindow steps and pick the
configuration which has the lowest mean rank. This way, we make most
use of the data accumulated during the course of the procedure. By
restricting our view to the last \stoppingWindow observations we also
take into account that the optimal parameter might change with
increasing model size: Since we focus on the most recent observations
with the biggest models, we always pick the configuration which is
most suitable for the data size at hand.

\subsection{Meta-Parameters for the CVST}
\label{sec:meta}

The CVST algorithm has a number of meta-parameters which the
experimenter has to choose beforehand. In this section we give
suggestions on how to choose these parameters. The parameter $\alpha$
controls the significance level for the test for similar behavior in
each step of the procedure. We suggest to set this to the usual
level of $\alpha = 0.05$. Furthermore $\beta_l$ and $\alpha_l$ control
the significance level of the $H_0$ (configuration is a loser) and
$H_1$ (configuration is a winner) respectively. We suggest an
asymmetric setup by setting $\beta_l = 0.1$, since we want to drop
loser configurations relatively fast and $\alpha_l = 0.01$, since we
want to be really sure when we accept a configuration as overall
winner. Finally, we set $\stoppingWindow$ to 3 for $\steps=10$ and 6
for $\steps = 20$, as we have observed that this choice works well in
practice.

\section{Properties of the CVST Algorithm}
\label{sec:theory}

After having introduced the overall concept of the CVST algorithm, we
now focus on some properties of the procedure: Exploiting properties
of the underlying sequential testing framework, we show how the
experimenter can control the algorithm to work in a stable
regime. Given some assumptions about the top/flop variables we
show that the CVST algorithm performs with high accuracy after a
configuration has reached its stable regime. Additionally, we show how
the CVST algorithm can be used to work best on a given time
budget. Finally, we discuss some unsolved questions and give possible
directions for future research.

\subsection{Performance in a Stable Regime}
\label{sec:fnr}

As discussed in Section~\ref{sec:cvsubsets} the winning probability of a
configuration might change if we feed the learning algorithm more
data. Therefore, a reasonable algorithm exploiting the learning on
subsets of the data must be capable of dealing with these difficulties
and potential change points in the behavior of certain
configurations. In this section we investigate some
properties of the CVST algorithm which makes it particularly suitable
for learning on increasing subsets of the data. 

The first property of the open sequential test employed in the CVST
algorithm comes in handy to control the overall convergence process
and to assure that no configurations are dropped prematurely:

\begin{lemma}[Safety Zone]\label{eq:SafetyBound} Given the CVST
  algorithm with significance level $\alpha_l, \beta_l$ for being a
  top or flop configuration respectively, and maximal number of steps
  $S$, and a configuration which loses for the first $s_\text{cp}$
  iterations, as long as
\[
  0 \leq \frac{s_\text{cp}}{\steps} \leq \frac{\ssafe}{\steps} \text{ with } \ssafe = \secZone\text{ and }\steps \geq
  \left\lceil \Bounda / \log 2 \right\rceil,
\]
the probability that the configuration is dropped prematurely by the CVST
algorithm is zero.
\end{lemma}

\begin{proof}
  The details of the proof are deferred to Appendix~\ref{sec:SecProof}.
\end{proof}

The consequence of Lemma~\ref{eq:SafetyBound} is that the experimenter
can directly control via the significance levels $\alpha_l, \beta_l$
until which iteration no premature dropping should occur and therefore
guide the whole process into a stable regime in which the
configurations will see enough data to show their real
performance. Note that this property is a direct consequence of the
sequential analysis framework and is used here to guide the test into
a controlled region where we do not observe a premature dropping of
configurations. Equation~(\ref{eq:seq2}) ensures that we actually
perform a meaningful test to discriminate a loser configuration
($\pi_0 = 0.5$) from a winning configuration ($\pi_1 > \pi_0$). 
Thus, by adjusting the safety zone of the CVST algorithm
the experimenter can ensure that the configurations act according to
the preconditions of the sequential testing framework introduced in
Section~\ref{sec:seqana}, namely exhibiting a fixed probability $\pi$
of being a winner configuration at each step. 

Note that this safety zone is solely a guard for premature dropping of
a configuration due to insufficient data in the first few steps of the
CVST algorithm. The experimenter should have a notion at which point
the performance of the configurations should stabilize in terms of
error behavior, i.e., the learners see enough data to show their real
performance. We argue that at this point the configurations behave
reasonable stable, thus, fulfilling both the independence and identically
distributed assumption of the sequential learning framework. We are
fully aware that these assumptions are strong, yet, backed up by our
extensive experimental evaluation in Section~\ref{sec:experiments}, we
want to shed some light on why the CVST procedure shows such
impressive speed-ups with small impact on the accuracy compared to
ordinary cross-validation and even outperforms other model selection
heuristics.

Hence, we define a {\it stable} configuration as a configuration which
sticks to a certain probability $\pi$ of being a winning
configuration. So after having seen enough data to show its real
behavior the robust transformation of the test error of the
configuration inside the CVST algorithm (see
Section~\ref{sec:transform}) exhibits the properties of an
i.i.d. Bernoulli variable and thus acts as a {\it stable}
configuration in the subsequent steps of the CVST procedure. So a
global winning configuration will be a {\it stable} configuration
with a probability $\pi \gg \pi_0 = 0.5$.

Using these assumptions we can now take a look at the worst case
performance of the CVST algorithm. Suppose a global winning
configuration has been constantly marked as a loser up to the safety
zone, because the amount of data available up to this point was not
sufficient to show the superiority of this configuration. Given that
the global winning configuration now sees enough data to be marked as
a winning configuration by the binarization process throughout the
next steps with probability $\pi$, we can give an error bound of
the overall process by solving specific recurrences.

\JMLRfigureMore{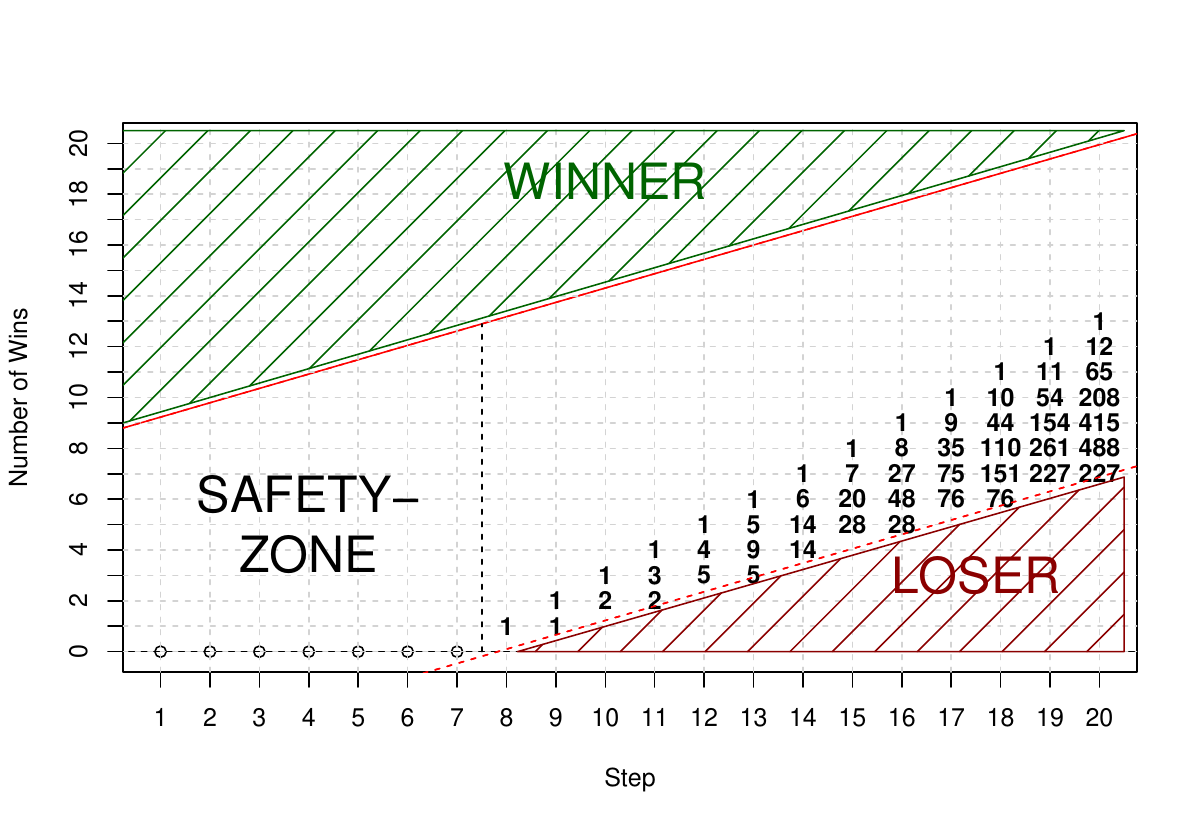}{.8}{Visualization of the worst-case
  scenario for the error probability of the CVST algorithm: A global
  winner configuration is labeled as a constant loser until the
  safety zone is reached. Then we can calculate the probability that
  this configuration endures the sequential test by a recurrence
  scheme, which counts the number of remaining paths ending up in the
  non-loser region.}{trim=5pt 15pt 30pt 40pt,clip=true}

Figure~\ref{fig:testRecurrence} gives a visual impression of our
worst case analysis for the example of a 20 step CVST execution:
The winning configuration generated a straight line of zeros up to the
safety zone of 7. Our approach to bound the error of the fast
cross-validation now consists essentially in calculating the probability
mass that ends up in the non-loser region. 
The following lemma shows how we can express the number of paths
which lead to a specific point on the graph by a two-dimensional
recurrence relation:

\begin{lemma}[Recurrence Relation]\label{eq:rec}
  Denote by $\operatorname{Path}(\srow, \scol)$ the number
  of paths, which lead to the point at the intersection of row
  $\srow$ and column $\scol$ and lie above the lower
  decision boundary $L_0$ of the sequential test. Given the worst case
  scenario described above the number of paths can be calculated as
  follows:
\[
\operatorname{Path}(\srow, \scol) =
\begin{cases}
  1 & \text{if } \srow = 0 \wedge c \leq \ssafe = \secZone\\
  1 & \text{if } \srow = \scol - \ssafe\\
  \operatorname{Path}(\srow, \scol - 1) +
  \operatorname{Path}(\srow - 1, \scol - 1) & \text{if } L_0(c) < \srow <  \scol - \ssafe\\
  0 & \text{otherwise.}
\end{cases}
\]
\end{lemma}

\begin{proof}
We split the proof into the four cases:
\begin{enumerate}
\item The first case is by definition: The configuration has a
  straight line of zeros up to the safety zone $\ssafe$.
\item The second case describes the diagonal path starting from the point
  $(1, \ssafe + 1)$: By construction of the paths (1 means diagonal up; 0
  means one step to the right) the diagonal path can just be reached
  by a single combination, namely a straight line of ones.
\item The third case is the actual recurrence: If the given point is
  above the lower decision bound $L_0$, then the number of paths 
  leading to this point is equal to the number of paths that lie
  directly to the left of this point plus the paths which lie directly
  diagonal downwards from this point. From the first paths this
  point can be reached by a direct step to the right and from
  the latter the current point can be reached by a diagonal step
  upwards. Since there are no other options than that by
  construction, this equality holds.
\item The last case describes all other paths, which either lie below
  the lower decision bound and therefore end up in the loser region or
  are above the diagonal and thus can never be reached.
\end{enumerate}
\end{proof}

This recurrence is visualized in Figure~\ref{fig:testRecurrence}. Each
number on the grid gives the number of valid, non-loser paths, which
can reach the specific point. With this recurrence we are now able to
prove a global, worst-case error probability of the fast
cross-validation.

\begin{theorem}[Error Bound of CVST for Stable Configuration] \label{eq:bound} Suppose a global
  winning configuration has reached the safety zone with a constant
  loser trace and then switches to a stable winner configuration with a
  success probability of $\pi \gg \pi_0 = 0.5$. Then the error that the CVST
  algorithm erroneously drops this configuration can be determined as
  follows:
\[
  P(\text{reject } \pi) \leq 1 - \sum_{i=\lfloor L_0(\steps) \rfloor + 1}^{r}
  \operatorname{Path}(i, \steps)\pi^i(1-\pi)^{r-i}\quad \quad \text{with }
  r = \steps - \left\lfloor \secZone \right\rfloor.
\]
\end{theorem}
\begin{proof}
The basic idea is to use the number of paths leading to the non-loser
region to calculate the probability that the configuration actually
survives. This corresponds to the last column of the example in
Figure~\ref{fig:testRecurrence}. Since we model the outcome of the
binarization process as a binomial variable with the success
probability of $\pi$, the first diagonal path has a probability of
$\pi^r$. The next paths each have a probability of
$\pi^{(r-1)}(1-\pi)^1$ and so on until the last viable paths are
reached in the point ($\lfloor L_0(\steps) \rfloor + 1, \steps)$. So the
complete probability of the survival of the configuration is summed up
with the corresponding number of paths from Lemma~\ref{eq:rec}. Since we are
interested in the complementary event, we subtract the resulting sum
from one, which concludes the proof.
\end{proof}

Note that the early stopping rule does not interfere with this bound:
The worst case is indeed that the process goes on for the maximal
number of steps \steps, since then the probability mass will be
maximally spread due to the linear lower decision boundary and the
corresponding exponents are maximal. So if the early stopping rule
terminates the process before reaching the maximum number of steps,
the resulting error probability will be lower than our given bound.

\JMLRfigureMore{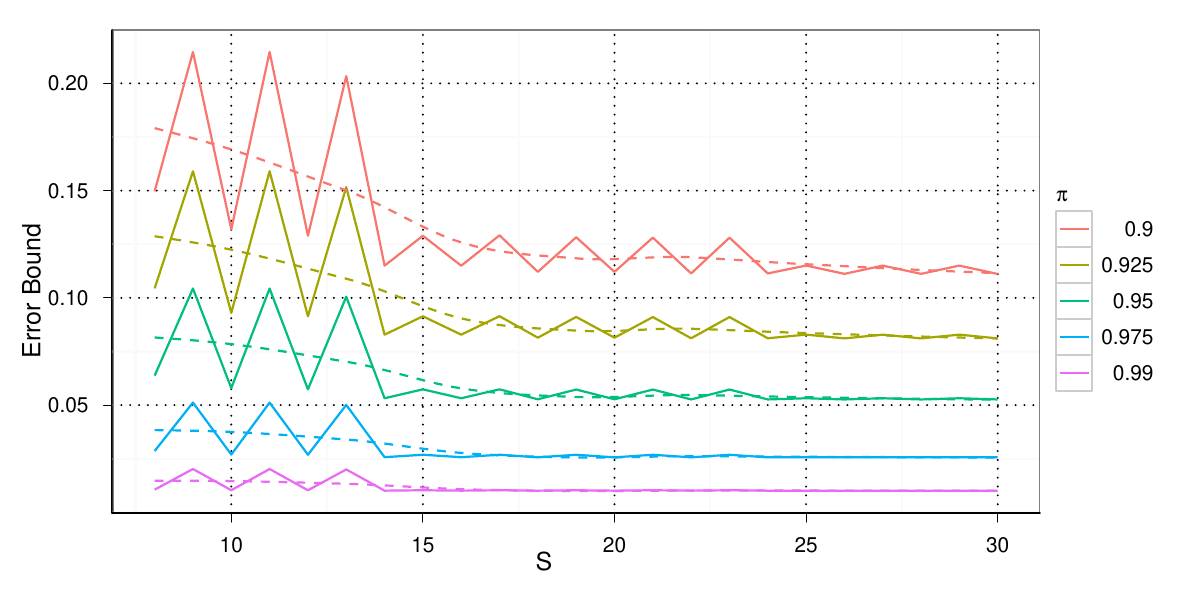}{.8}{Error bound of the fast
  cross-validation as proven in Theorem~\ref{eq:bound} for different
  success probabilities $\pi$ and maximal step sizes \steps. To mark the
  global trend we fitted a LOESS curve given as dotted line to the
  data.}{trim=10pt 12pt 10pt 12pt,clip=true}

The error bound for different success probabilities and the proposed
sequential test with $\alpha_l=0.01$ and $\beta_l=0.1$ are depicted in
Figure~\ref{fig:errorBound}. First of all we can observe a relatively
fast convergence of the overall error with increasing maximal number
of steps \steps. The impact on the error is marginal for the shown
success probabilities, i.e., for instance for $\pi = 0.95$ the error
nearly converges to the optimum of $0.05$. Note that the oscillations
especially for small step sizes originate from the rectangular grid
imposed by the interplay of the $\operatorname{Path}$-operator and the
lower decision boundary $L_0$ leading to some fluctuations. Overall,
the chosen test scheme allows us not only to control the safety zone
but also has only a small impact on the error probability, which once
again shows the practicality of the open sequential ratio test for the
fast cross-validation procedure. By using this statistical test we can
balance the need for a conservative retention of configurations as
long as possible with the statistically controlled dropping of
significant loser configurations with nearly no impact on the overall
error probability. 

Our analysis assumes that the experimenter has chosen the right safety
zone for the learning problem at hand. For small data sizes it could
happen that this safety zone was chosen too small, therefore the
change point of the global winning configuration might lie outside the
safety zone. While this will not occur often for today's sizes of data
sets we have analyzed the behavior of CVST under this circumstances to
give a complete view of the properties of the algorithm.
To get insight into the drop rate for the case when the
experimenter underestimated the change point $s_\text{cp}$ we simulate
those switching configurations by independent Bernoulli variables
which change their success probability $\pi$ from a chosen
$\pi_\text{before} \in \{0.1, 0.2, \dots, 0.5\}$ to a constant $1.0$
at a given change point. This behavior essentially imitates the
behavior of a switching configuration which starts out as a loser
(i.e., up to the change point the trace will consist more or less of
zeros) and after enough data is available turns into a constant
winner.

\JMLRfigure{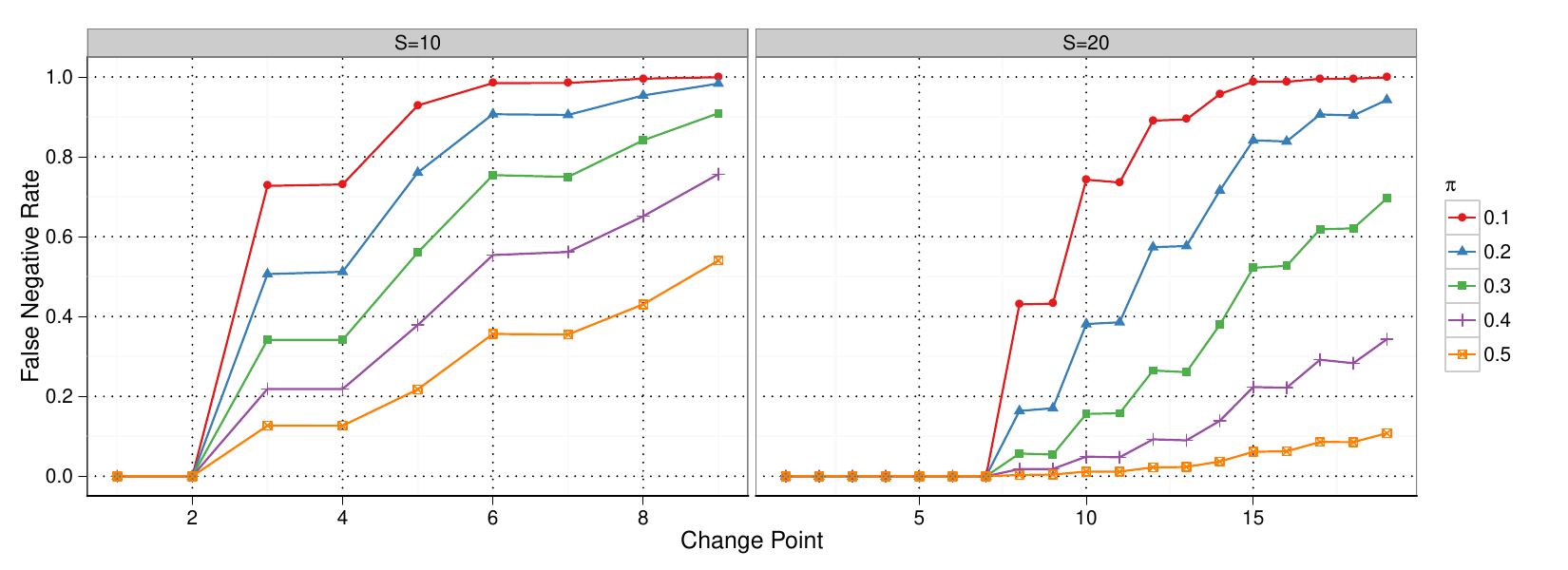}{1}{False negatives
  generated with the open sequential test for non-stationary
  configurations , i.e., at the given change point the Bernoulli
  variable changes its $\pi_\text{before}$ from the indicated value to
  1.0.}

The relative loss of these configurations for 10 and 20 steps is
plotted in Figure~\ref{fig:falseNegativeSimulationsResultsOpen} for
different change points. The figure reveals our theoretical findings
of Lemma~\ref{eq:SafetyBound} showing the corresponding \emph{safety
  zone} for the specific parameter settings: For instance for
$\alpha_l=0.01 \text{ and } \beta_l=0.1$ and $\steps = 10$ steps, the
safety zone amounts to $0.27 \times 10$, meaning that if the change
point for all switching configurations occurs at step one or two, the
CVST algorithm would not suffer from false positives. Similarly, for
$\steps = 20$ the safety zone is $0.39 \times 20 = 7.8$. These
theoretical results are confirmed in our simulation study, where the
false negative rate is zero for sufficiently small change points for
the open variant of the test. After that, there are increasing
probabilities that the configuration will be removed. Depending on the
success probability of the configuration before the change point, the
resulting false negative rate ranges from mild for $\pi=0.5$ to
relatively severe for $\pi=0.1$. The later the change point occurs,
the higher the resulting false negative rate will be. Interestingly,
if we increase the total number of steps from 10 to 20, the absolute
values of the false negative rates are significantly lower. So even
when the experimenter underestimates the actual change point, the CVST
algorithm has some extra room which can even be extended by increasing
the total number of steps.

\subsection{Fast-Cross Validation on a Time Budget}

\JMLRfigure{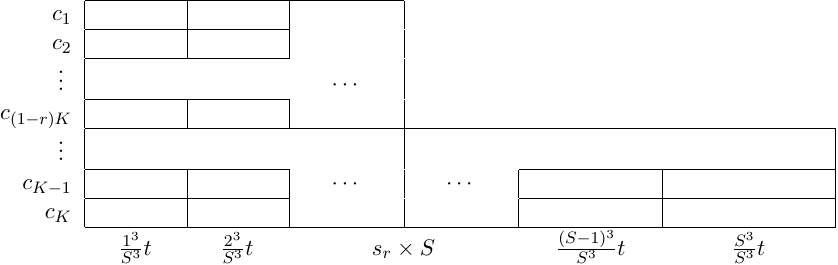}{.8}{Approximation of the time
  consumption for a cubic learner. In each step we calculate a
  model on a subset of the data, so the model calculation time $\tfull$ on
  the full data set is adjusted accordingly. After $\srate \times \steps$
  steps of the process, we assume a drop to $r \times K$ remaining
  configurations.}

While the CVST algorithm can be used out of the box to speed up
regular cross-validation, the aforementioned properties of the
procedure come in handy when we face a situation in which an optimal
parameter configuration has to be found given a fixed computational
budget. If the time is not sufficient to perform a full
cross-validation or the amount of data that has to be processed is too
big to explore a sufficiently spaced parameter grid with ordinary
cross-validation in a reasonable time, the CVST algorithm can easily
be adjusted to the specified time constraint. Thus, the experimenter
is able to get the maximal number of model evaluations given the time
budget available to judge which model is the best.

This is achieved by calculating a maximal steps parameter \steps{} which
leads to a near coverage of the available time budget $T$ as depicted
in Figure~\ref{fig:ComputationalBudget}. The idea is to specify an
expected drop rate $(1-r)$ of configurations and a safety zone bound
$\ssafe$. Then we can give a rough estimate of the total time needed
for a CVST with a total number of steps \steps, equating this with the
available time budget $T$ and solving for \steps. More formally, given
$K$ parameter configurations and a pre-specified safety zone bound
$\ssafe = \srate \times \steps$ with $0 < \srate < 1$ to ensure that
no configuration is dropped prematurely, the computational demands of
the CVST algorithm are approximated by the sum of the time needed
before step $\ssafe$ involving the model calculation of all $K$
configurations and after step $\ssafe$ for $r \times K$ configurations
with $0 < r < 1$. As we will see in the experimental evaluation
section, this assumption of a given drop rate of $(1-r)$ leading to
the form of time consumption as depicted in
Figure~\ref{fig:ComputationalBudget} is quite common. The observed
drop rate corresponds to the overall difficulty of the problem at
hand.

Given the computation time $\tfull$ needed to perform the model calculation
on the full data set, we prove in Appendix~\ref{sec:BudgetProof} that
the optimal maximum step parameter for a learner of time complexity
$f(n) = n^m$ can be calculated as follows:
\begin{align*}
\steps &= \left\lfloor \frac{m+1}{4}\frac{2T - \tfull K(1-r)\srate^m + \tfull K r}{((1-r)\srate^{m+1}+r)\tfull K}\right.\\
&+\left.\sqrt{\left[\frac{m+1}{4}\frac{2T - \tfull K(1-r)\srate^m + \tfull K r}{((1-r)\srate^{m+1}+r)\tfull K}\right]^2 - \frac{m(m+1)}{12}\frac{(1-r)\srate^{m-1}+r}{(1-r)\srate^{m+1}+r}}\right\rfloor.
\end{align*}
After calculating the maximal number of steps \steps{} given the time
budget $T$, we can use the results of Lemma~\ref{eq:SafetyBound} to
determine the maximal $\beta_l$ given a fixed $\alpha_l$, which yields
the requested safety zone bound $\ssafe$.

\subsection{Discussion of Further Theoretical Analyses}
\label{sec:discussionanalyses}

In Section~\ref{sec:cvsubsets} we have noted that in order for the
test performances to converge, the parameter configurations should be
independent of the sample size. As shown in Appendix~\ref{app:cond}
this holds for a range of standard methods in machine learning. Yet,
special care has to be taken to really ensure this assumption. For
instance for kernel ridge regression one has to scale the ridge
parameter during each step of the CVST algorithm to accommodate for
the change in the learning set size (see the reference implementation
in the official CRAN package named \texttt{CVST} or the development
version at \url{https://github.com/tammok/CVST}). The $\nu$-Support
Vector Machine on the other hand directly incorporates this scaling of
parameters which makes it a good fit for the CVST
algorithm. Generally, it would be preferable to have this scaling
automatically incorporated in the CVST algorithm such that the
experimenter could plug-in his favorite method without the need to
think about any scaling issues of hyper-parameters. Unfortunately,
this is highly algorithm dependent and, thus, is an open problem for
further research.

An additional concern to the practitioner is how to choose the correct
size of the safety zone $\ssafe$. If the training set does not contain
enough data to get to a stable regime of the parameter configurations,
even regular cross-validation on the full data set would yield
incorrect configurations. But if we have just barely enough data to
reach this stable region, setting the right safety zone is essential
for the CVST algorithm to return the correct
configurations. Unfortunately we are not aware of any test or bound
which could hint at the right safety zone given a data set and
learner. Yet, in today's world of big data where sample sizes are more
often too big than too small, this might not pose a serious problem
anymore. Nevertheless, we have analyzed the behavior of the CVST
algorithm in case the experimenter underestimates the safety zone in
Section~\ref{sec:fnr} showing that even for these cases CVST
is able to absorb a certain amount of misspecification.

The similarity test introduced in Section~\ref{sec:transform} relies
on two assumptions: First, the averaged loss function over the data
not used for training in one step gives us a good indicator of the
performance of a configuration. Second, well performing configurations
show similar behavior in classification or regression on the data not
used for learning. While these assumptions definitely make sense, they
encode a certain optimism of how the grid of configurations is
populated: If we have too few configurations as input to the procedure
it might happen that some non-optimal configurations mask out the
other, normally optimal, configurations just by chance. To overcome
this problem we therefore would need a certain amount of redundancy in
the configuration grid. Both the amount of redundancy and thus the
similarity measure underlying this redundancy assumption are hard to
grasp theoretically, yet, it could lead to new ways to model the binary
transformation of the performance of configurations in each step of
the CVST algorithm.

There might be even further potential in the behavior of similar
configurations that could be used in the CVST algorithm: If there is a
notion of similarity between different configurations, it would be
interesting to exploit this information and incorporate it into the
CVST algorithm. For instance, one could add this kind of information
in the function {\sc topConfigurations} of
Algorithm~\ref{alg:complete} to average the result of similar
configurations and, hence, extend the pooling effect of the test
already available for the data point dimension in the direction of
configurations.

While the selection scheme explained in Section~\ref{sec:seqana} deals
with the fact of potential change points of a configuration, it is not
clear how independent the individual entries of a trace for a given
configuration are and how much these potential dependencies influence
the power of the sequential testing framework. Preliminary experiments
comparing the CVST algorithm as described in this paper and a version
of the CVST algorithm where at each step the data pool is shuffled,
thus, yielding always different data points for learning and
evaluation, showed no significant differences between these two
versions. This indicates that at least the potential dependencies
introduced by the overlap of learning sets due to subsequent addition
of data points do not interfere with the dependency assumption of
the sequential testing framework. We will see in the evaluation
section that the CVST procedure in its current form shows excellent
behavior throughout a wide range of data sets; yet, further research
of the theoretical properties of CVST might yield even better
procedures in the future.

\section{Experiments}
\label{sec:experiments}

Before we evaluate the CVST algorithm on real data, we investigate its
performance on controlled data sets. Both for regression and
classification tasks we introduce special tailored data sets to
highlight the overall behavior and to stress-test the fast
cross-validation procedure. To evaluate how the choice of learning
method influences the performance of the CVST algorithm, we compare
kernel logistic regression (KLR) against a $\nu$-Support Vector
Machine (SVM) for classification problems and kernel ridge regression
(KRR) versus $\nu$-SVR for regression problems each using a Gaussian
kernel \citep[see][]{Roth2001, Scholkopf:2000}. In all experiments we
use a 10 step CVST with parameter settings as described in
Section~\ref{sec:meta} (i.~e.~$\alpha = 0.05, \alpha_l = 0.01, \beta_l
= 0.1, \stoppingWindow = 3$) to give us an upper bound of the expected
speed gain. Note that we could get even higher speed gains by either
lowering the number of steps or increasing $\beta_l$. From a practical
point of view we believe that the settings studied are highly
realistic.

\subsection{Artificial Data Sets}
\label{sec:art}
To assess the quality of the CVST algorithm we first examine its
behavior in a controlled setting. We have seen in our motivation
section that a specific learning problem might have several layers of
structure which can only be revealed by the learner if enough data is
available. For instance in Figure~\ref{fig:ErrorExample} we can see
that the first optimal plateau occurs at $\sigma=0.1$, while the real
optimal parameter centers around $\sigma=0.01$. Thus, the real optimal
choice just becomes apparent if we have seen more than 200 data
points.

\JMLRdoublefigure{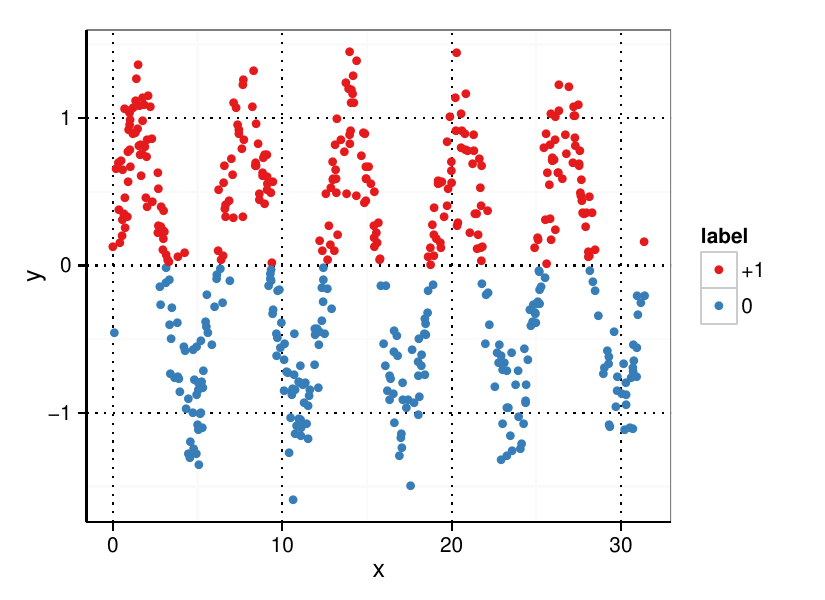}{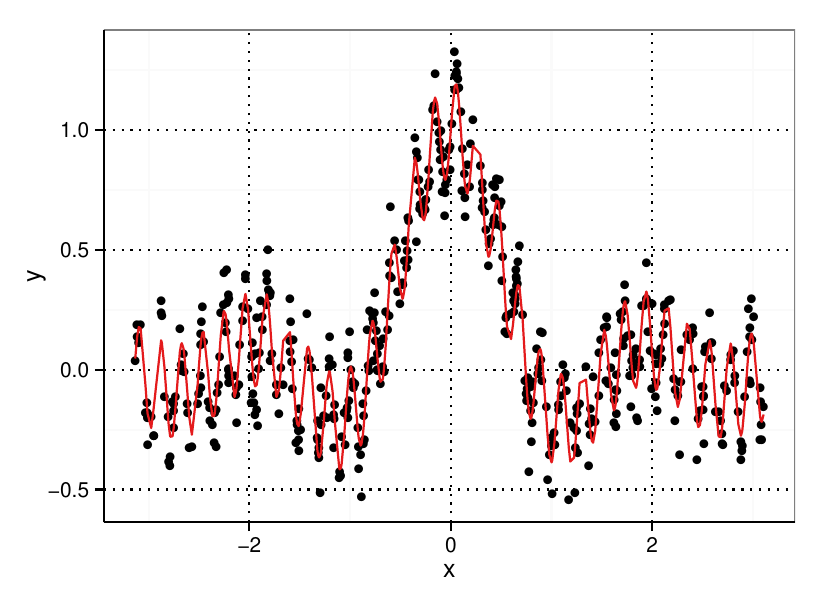}{.48}{The
  \emph{noisy sine} (left) and \emph{noisy sinc} data set (right).}

In this section we construct a learning problem both for regression
and classification tasks which could pose severe problems for the CVST
algorithm: If it stops too early, it will return a suboptimal
parameter set. We evaluate how different intrinsic dimensionalities
of the data and various noise levels affect the performance of the
procedure. For classification tasks we use the \emph{noisy sine} data
set, which consists of a sine uniformly sampled from a range
controlled by the intrinsic dimensionality $d$:
\[
y = \sin(x) + \epsilon \quad \text{ with } \epsilon \sim
\mathcal{N}(0, n^2), x \in [0, 2 \pi d], n \in \{0.25, 0.5\}, d \in
\{5, 50, 100\}.
\]
The labels of the sampled points are just the sign of $y$. An example
for $d=5, n=0.25$ is plotted in the left subplot of
Figure~\ref{fig:Example_NoisySine}. For regression tasks we devise the
\emph{noisy sinc} data set, which consists of a sinc function
overlayed with a high-frequency sine:
\[
y = \operatorname{sinc}(4 x) + \frac{\sin(15 d x)}{5} + \epsilon \text{ with }
\epsilon \sim \mathcal{N}(0, n^2), x \in [-\pi, \pi], n
\in \{0.1, 0.2\}, d \in \{2, 3, 4\}.
\]
An example for $d=2, n=0.1$ is plotted in the right subplot of
Figure~\ref{fig:Example_NoisySine}. For each of these data sets we
generate 1,000 data points and run a 10 step CVST and compare its
results with a normal 10-fold cross-validation on the full data
set. We record both the test error on additional 10,000 data points
and the time consumed for the parameter search. The explored parameter
grid contains 610 equally spaced parameter configurations for each
method ($\log_{10}(\sigma) \in \{-3, -2.9, \dots , 3\}$ and $\nu \in
\{0.05, 0.1, \dots, 0.5\}$ for SVM/SVR and $\log_{10}(\lambda) \in
\{-7, -6, \dots 2\}$ for KLR/KRR, respectively). This process is
repeated 50 times to gather sufficient data for an interpretation of
the overall process. Apart from recording the difference in mean
square error (MSE) of the learner selected by normal cross-validation
and by the CVST algorithm we also look at the relative speed
gain. Note that we have encoded the classes as $0$ and $1$ for the
classification experiments so the MSE corresponds to the
misclassification rate of the learner. So the difference in MSE gives
us a good measurement of the impact of using the CVST algorithm for
both classification and regression experiments.

\JMLRdoublefigure{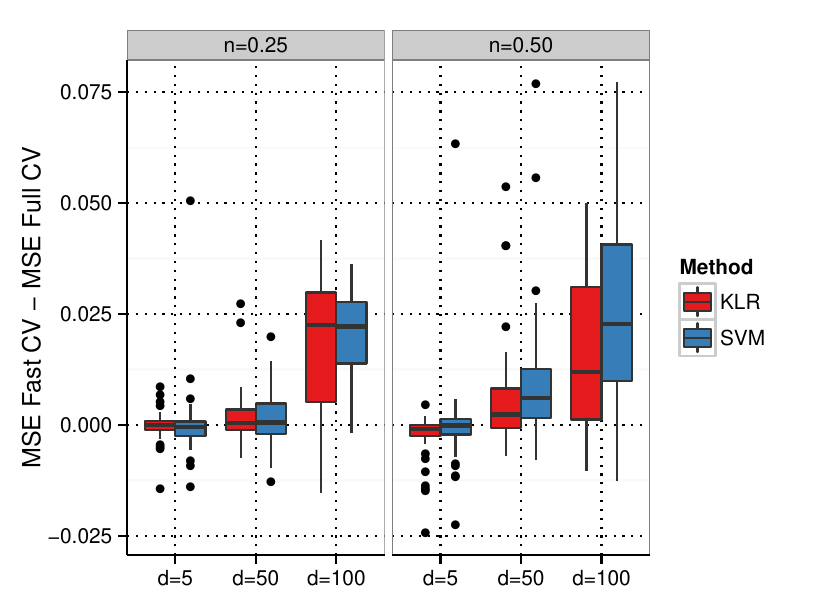}{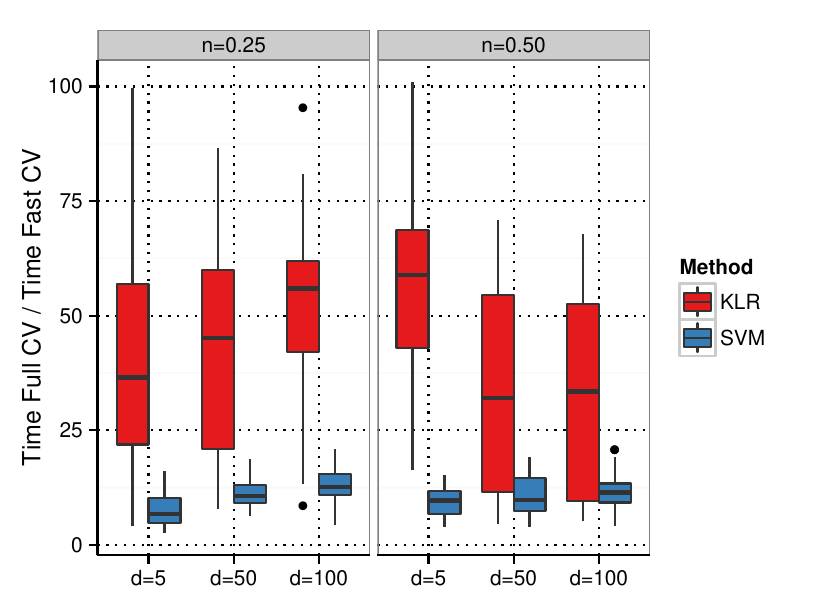}{.48}{Difference
  in mean square error (left) and relative speed gain (right) for the
  \emph{noisy sine} data set.}

The results for the \emph{noisy sine} data set can be seen in
Figure~\ref{fig:MS_Toy_noisySine_Difference}. The left boxplots show the
distribution of the difference in MSE of the best
parameter determined by CVST and normal cross-validation. In the low
noise setting ($n = 0.25$) the CVST algorithm finds the same optimal
parameter as the normal cross-validation up to the intrinsic
dimensionality of $d=50$. For $d=100$ the CVST algorithm gets stuck in
a suboptimal parameter configuration yielding an increased
classification error compared to the normal cross-validation. This
tendency is slightly increased in the high noise setting ($n = 0.5$)
yielding a broader distribution. The classification method used seems
to have no direct influence on the difference, both SVM and KLR show
nearly similar behavior. This picture changes when we look at the
speed gains: While the SVM nearly always ranges between 10 and 20, the
KLR shows a speed-up between 20 and 70 times. The variance of the
speed gain is generally higher compared to the SVM which seems to be a
direct consequence of the inner workings of KLR: The main
loop performs at each step a matrix inversion of the whole kernel
matrix until the calculated coefficients converge. Obviously this
convergence criterion leads to a relative wide-spread distribution of
the speed gain when compared to the SVM performance.

\JMLRfigureP{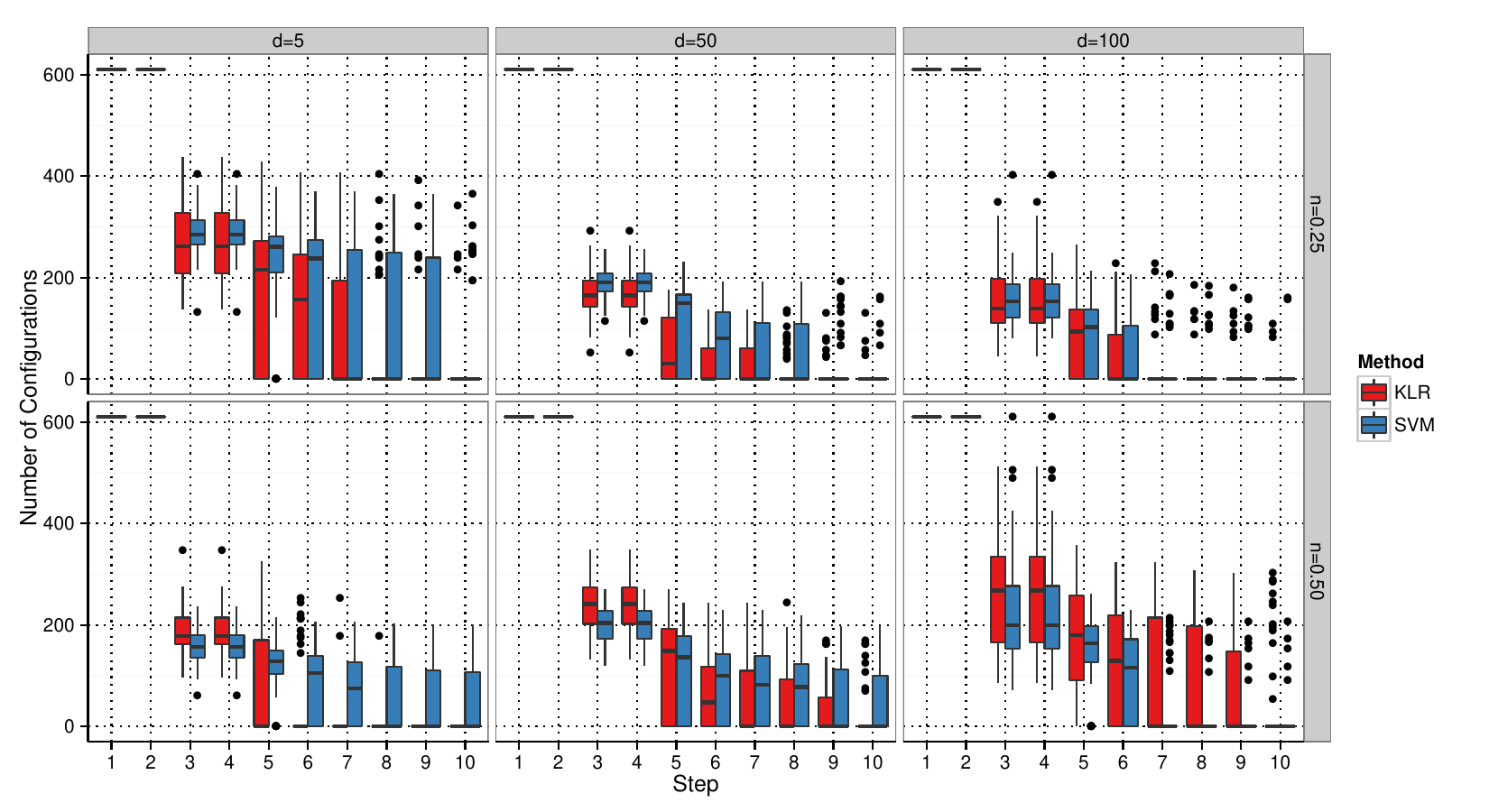}{1}{Remaining configurations after
  each step for the \emph{noisy sine} data set.}

Figure~\ref{fig:MS_Toy_noisySine_Trace} shows the distribution of the
number of remaining configurations after each step of the CVST
algorithm. In the low noise setting (upper row) we can observe a
tendency of higher drop rates up to $d=100$. For the high noise
setting (lower row) we observe a steady increase of kept
configurations combined with a higher spread of the
distribution. Overall we see a very effective drop rate of
configurations for all settings. The SVM and the KLR show nearly similar
behavior so that the higher speed gain of the KLR we have seen before is a
consequence of the algorithm itself and is not influenced by
the CVST algorithm.

\JMLRdoublefigure{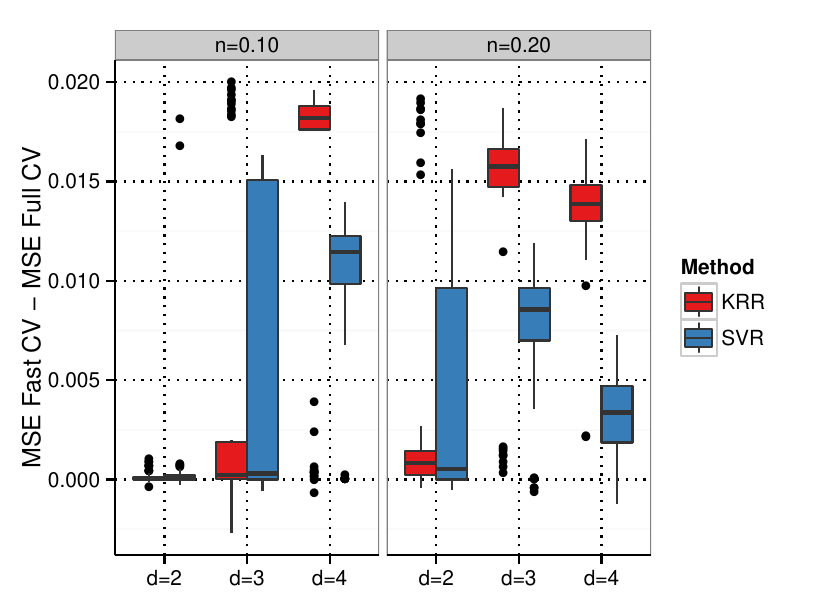}{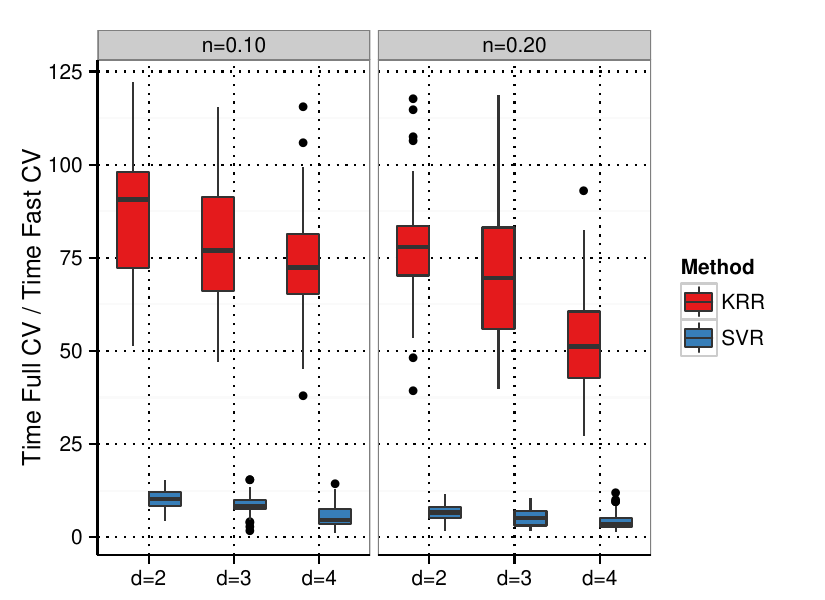}{.48}{Difference
  in mean square error (left plots) and relative speed gain (right
  plots) for the \emph{noisy sinc} data set.}

The performance on the \emph{noisy sinc} data set is shown in
Figure~\ref{fig:MS_Toy_noisySinc_Difference}. The first striking
observation is the transition of the CVST algorithm which can be
observed for the intrinsic dimensionality of $d=3$. At this point the
overall excellent performance of the CVST algorithm is on the verge of
choosing a suboptimal parameter configuration. This behavior is more
evident in the high noise setting. In the case of SVR the difference
to the solution found by the normal cross-validation is always smaller
than for KRR. The speed gain observed shows a small decline over the
different dimensionalities and noise levels and ranges between 10 and
20 for the SVR and 50 to 100 for KRR.

\JMLRfigureP{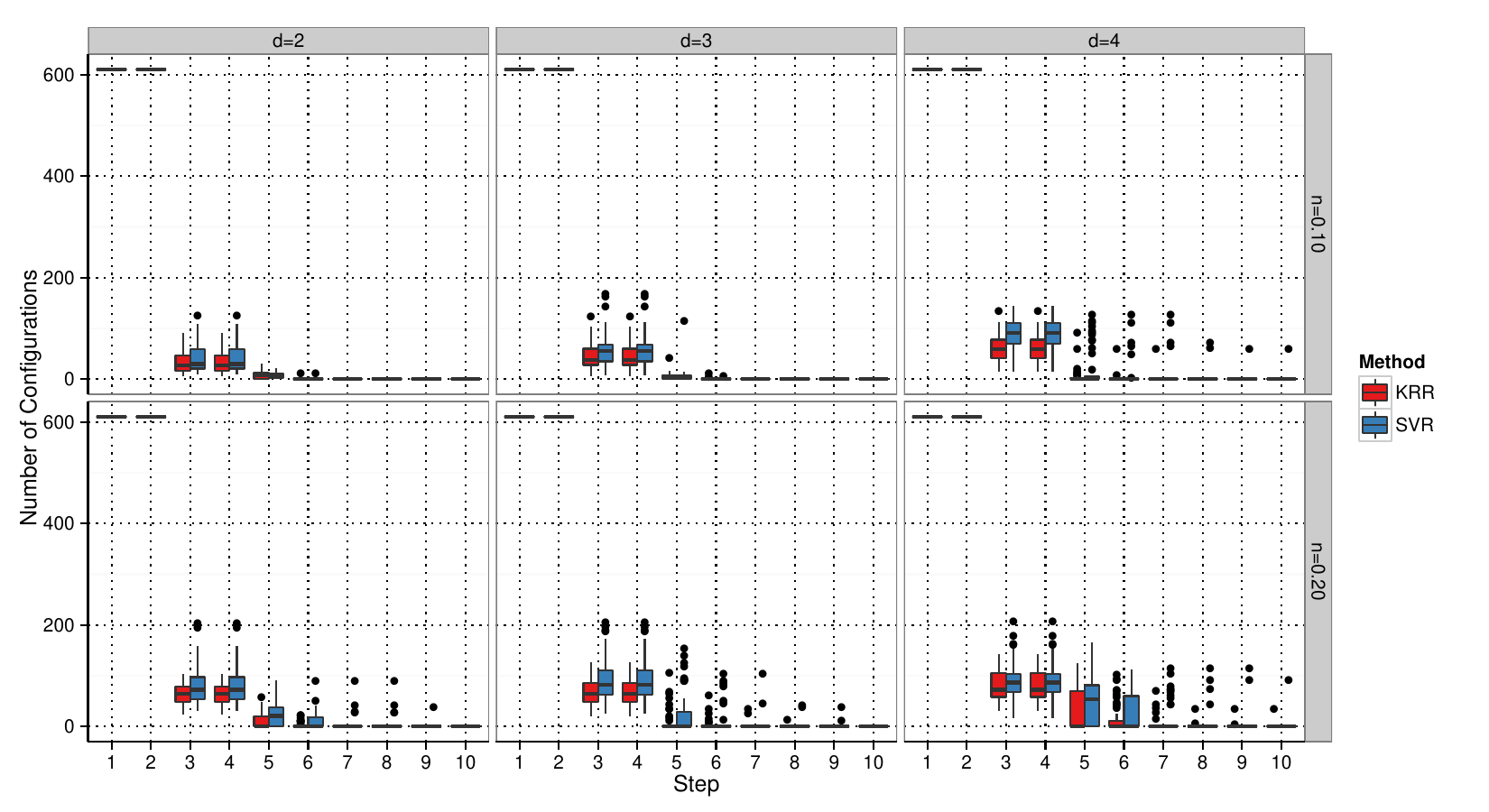}{1}{Remaining configurations after
  each step for the \emph{noisy sinc} data set.}

This is a direct consequence of the behavior which can be observed in
the number of remaining configurations shown in
Figure~\ref{fig:MS_Toy_noisySinc_Trace}. Compared to the classification
experiments the drop is much more drastic. The intrinsic
dimensionality and the noise level show a small influence (higher
dimensionality or noise level yields more remaining configurations) but
the overall variance of the distribution is much smaller than in the
classification experiments.

\JMLRdoublefigure{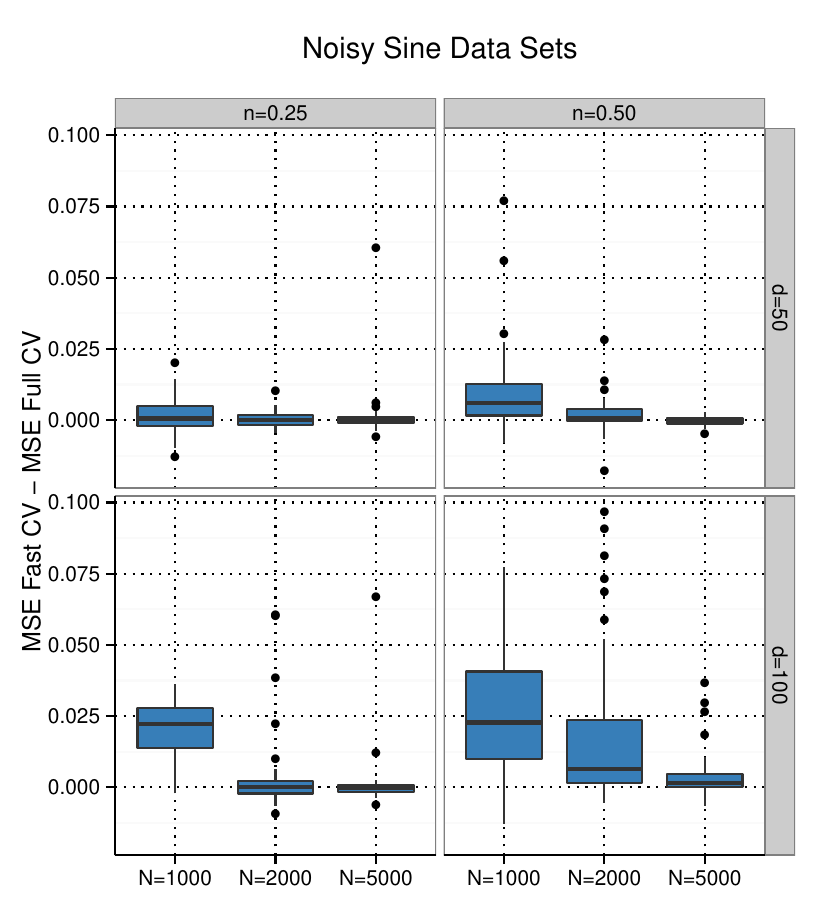}{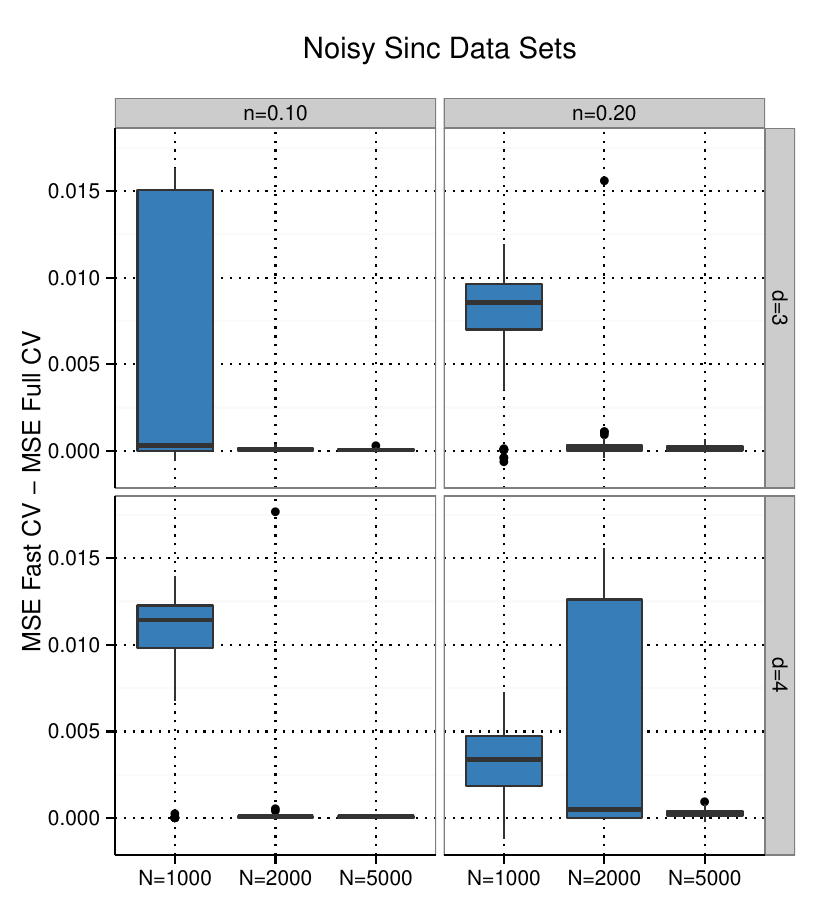}{.48}{Difference
  in mean square error for SVM/SVR with increasing data set size for
  \emph{noisy sine} (left) and the \emph{noisy sinc} (right) data sets. By adding
  more data, the CVST algorithm converges to the correct parameter
  configuration.}

In Figure~\ref{fig:MS_Toy_noisySine_DifferenceAdd} we examine the
influence of more data on the performance of the CVST algorithm. Both
for the \emph{noisy sine} and \emph{noisy sinc} data set we are able
to estimate the correct parameter configuration for all noise and
dimensionality settings if we feed the CVST with enough
data.\footnote{Note that we have to limit this experiment to the
  SVM/SVR method, since the full cross-validation of the KLR/KRR would
  have taken too much time to compute.} Clearly, the CVST is capable of
extracting the right parameter configuration if we increase the
amount of data to $2000$ or $5000$ data points, rendering our method
even more suitable for big data scenarios: If data is abundant, CVST
will be able to estimate the correct parameter in a much smaller time
frame.

\subsection{Benchmark Data Sets}

After demonstrating the overall performance of the CVST algorithm on
controlled data sets we will investigate its performance on real life
and well known benchmark data sets. For classification we picked a
representative choice of data sets from the IDA benchmark repository
(see \citealt{ROM2001}\footnote{Available at
\url{http://www.mldata.org}.}). Furthermore we added the first two classes
with the most entries of the \emph{covertype} data set
\citep[see][]{blackard1999comparative}. Then we follow the procedure
of the paper in sampling 2,000 data points of each class for the model
learning and estimate the test error on the remaining data points. For
regression we pick the data used in \cite{donoho1994} and add the
\emph{bank32nm}, \emph{pumadyn32nm} and \emph{kin32nm} of the Delve
repository.\footnote{Available at \url{http://www.cs.toronto.edu/~delve}.}

We process each data set as follows: First we normalize each variable
of the data to zero mean and variance of one, and in case of
regression we also normalize the dependent variable. Then we split the
data set in half and use one part for training and the other for the
estimation of the test error. This process is repeated 50 times to get
sufficient statistics for the performance of the methods. As in the
artificial data setting we compare the speed gain of the fast compared
to the normal cross-validation on the same parameter grid of 610
values. To allow for better comparability of the performance on the
different data sets we report the mean square error (MSE) ratio of the
CVST procedure compared to the normal cross-validation, i.e., values
over 1.0 favor the CVST procedure. For the \emph{blocks},
\emph{bumps}, and \emph{doppler} data set of \cite{donoho1994} we
adjusted the range of $\sigma$ to a smaller scale ($\log_{10}(\sigma)
\in \{-6, -5.9, \dots , 0\}$) to have reasonable results in the
parameter grid of 610 values since these data sets contain a very
fine-grained structure. Note that this adjustment is just for the sake
of comparability to the other data sets.

\JMLRdoublefigure{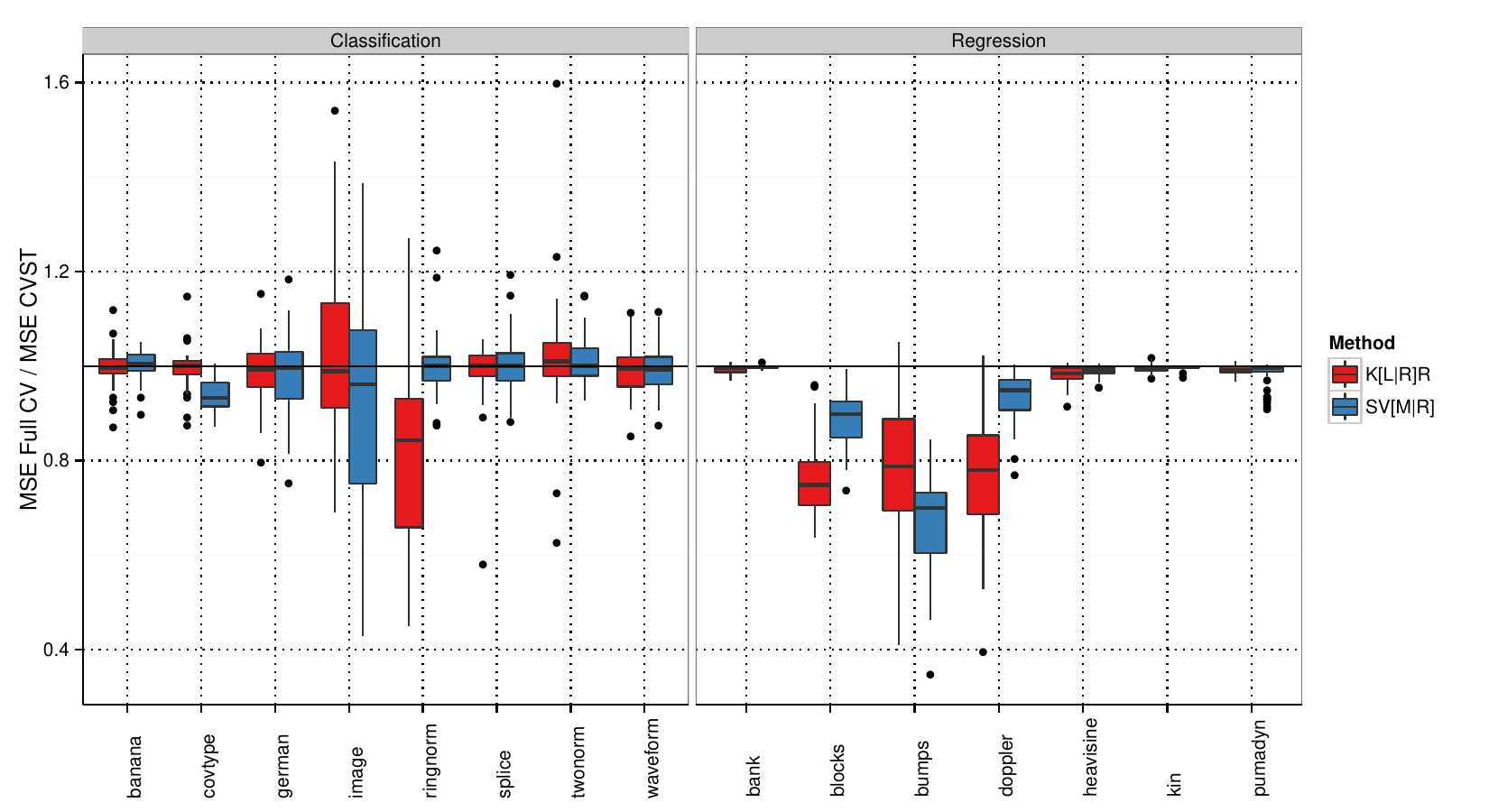}{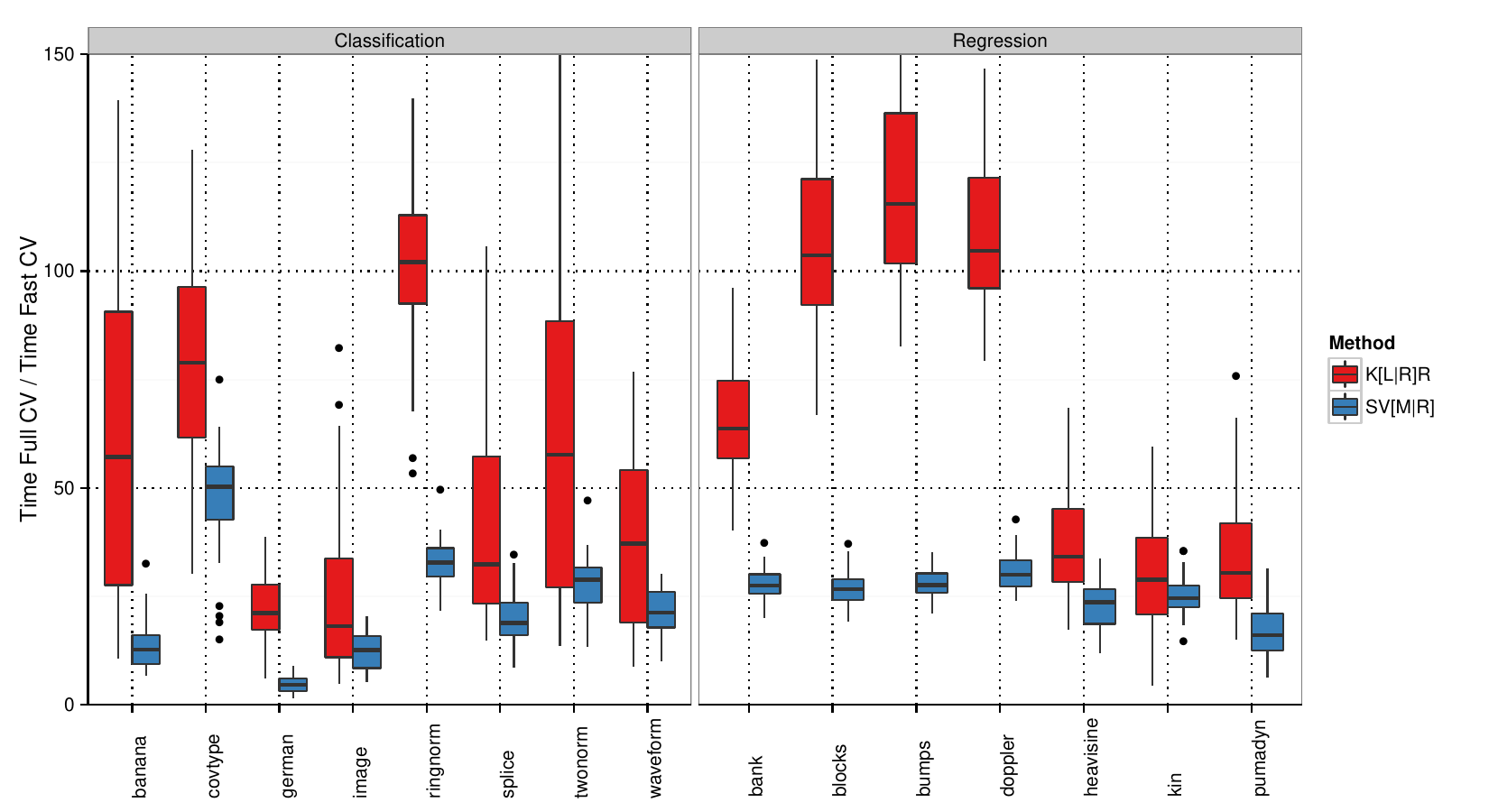}{1}{MSE ratio (upper plots) and relative speed gain (lower
  plots) for the \emph{benchmark} data sets.}

\JMLRdoublefigure{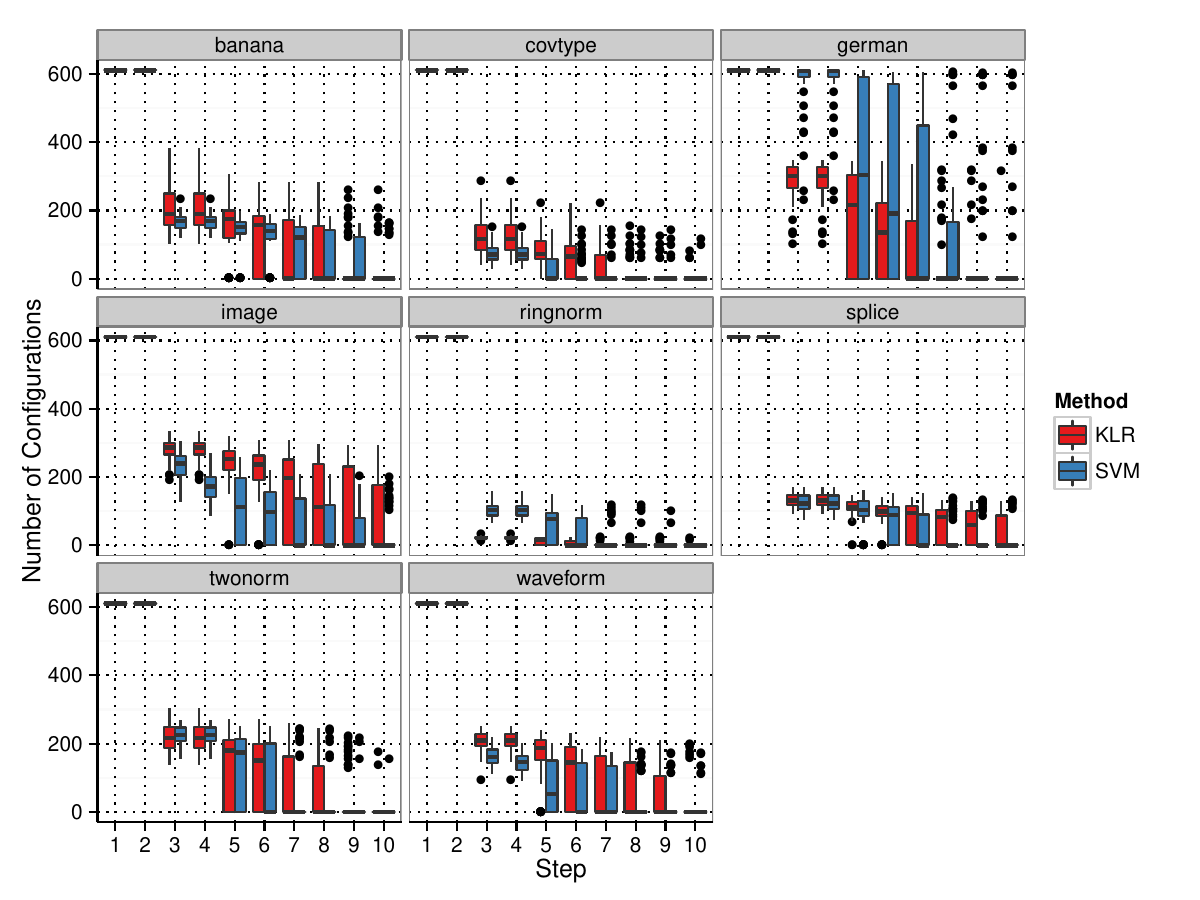}{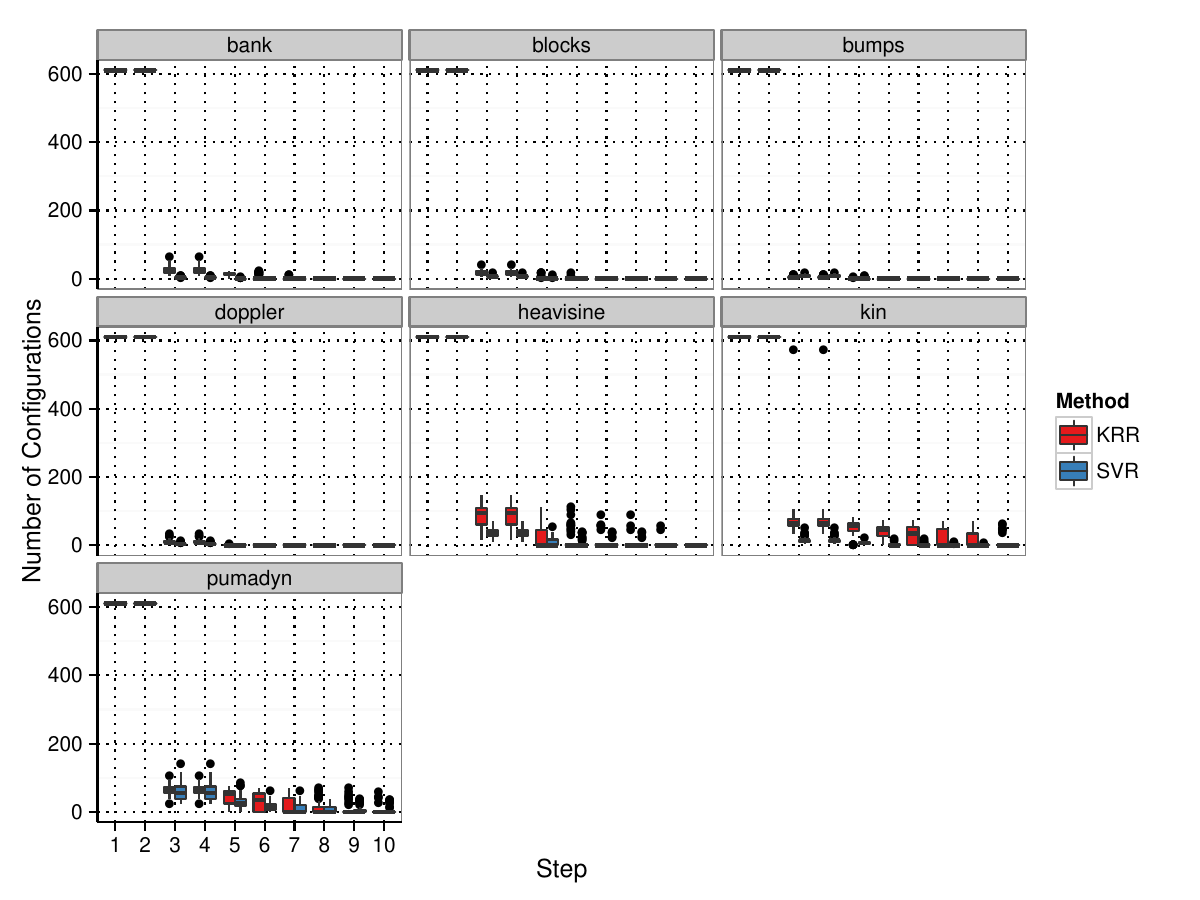}{.9}{Remaining
  configurations after each step for different \emph{benchmark} data sets.}

\begin{table}[bt]
\centering
\footnotesize
\begin{tabular}{rrrr|rrrr}
  \toprule
\bf Data & \bf Method & \bf MSE Ratio & \bf Speed & \bf Data & \bf Method & \bf MSE Ratio & \bf Speed \\ 
  \midrule
  banana & KLR & $ \mathbf{ 0.998 \pm 0.012 }$ & $ 61.1 $ &   bank & KRR & $ 0.994 \pm 0.003 $ & $ 85.7 $ \\       
  banana & SVM & $ \mathbf{ 1.001 \pm 0.008 }$ & $ 16.9 $ &   bank & SVR & $ 0.998 \pm 0.001 $ & $ 43.9 $ \\       
  covtype & KLR & $ \mathbf{ 0.994 \pm 0.011 }$ & $ 84.3 $ &  blocks & KRR & $ 0.762 \pm 0.022 $ & $ 81.3 $ \\     
  covtype & SVM & $ 0.939 \pm 0.009 $ & $ 53.4 $ &            blocks & SVR & $ 0.886 \pm 0.016 $ & $ 39.5 $ \\     
  german & KLR & $ \mathbf{ 0.987 \pm 0.017 }$ & $ 27.5 $ &   bumps & KRR & $ 0.784 \pm 0.043 $ & $ 86.1 $ \\      
  german & SVM & $ \mathbf{ 0.981 \pm 0.024 }$ & $ 5.5 $ &    bumps & SVR & $ 0.666 \pm 0.030 $ & $ 37.3 $ \\      
  image & KLR & $ \mathbf{ 1.032 \pm 0.051 }$ & $ 33.5 $ &    doppler & KRR & $ 0.766 \pm 0.035 $ & $ 92.4 $ \\    
  image & SVM & $ 0.923 \pm 0.060 $ & $ 13.8 $ &              doppler & SVR & $ 0.937 \pm 0.014 $ & $ 41.2 $ \\    
  ringnorm & KLR & $ 0.814 \pm 0.050 $ & $ 111.7 $ &          heavisine & KRR & $ 0.981 \pm 0.005 $ & $ 53.2 $ \\  
  ringnorm & SVM & $ \mathbf{ 0.999 \pm 0.017 }$ & $ 32.4 $ & heavisine & SVR & $ 0.988 \pm 0.003 $ & $ 33.7 $ \\  
  splice & KLR & $ \mathbf{ 0.991 \pm 0.019 }$ & $ 52.2 $ &   kin & KRR & $ 0.994 \pm 0.002 $ & $ 58.7 $ \\        
  splice & SVM & $ \mathbf{ 1.005 \pm 0.016 }$ & $ 18.1 $ &   kin & SVR & $ 0.996 \pm 0.001 $ & $ 39.9 $ \\        
  twonorm & KLR & $ \mathbf{ 1.017 \pm 0.034 }$ & $ 50.1 $ &  pumadyn & KRR & $ 0.992 \pm 0.003 $ & $ 68.2 $ \\    
  twonorm & SVM & $ \mathbf{ 1.015 \pm 0.014 }$ & $ 25.7 $ &  pumadyn & SVR & $ 0.984 \pm 0.007 $ & $ 29.6 $ \\    
  waveform & KLR & $ \mathbf{ 0.989 \pm 0.014 }$ & $ 54.0 $ \\ 
  waveform & SVM & $ \mathbf{ 0.992 \pm 0.013 }$ & $ 22.8 $ \\ 
  \bottomrule
\end{tabular}
\caption{Comparison of performance of the CVST algorithm to full cross-validation (classification data sets in left part, regression data sets in right part). MSE ratio is the relative gain in MSE of CVST compared to the full cross-validation. Speed denotes the relative speed increase of CVST compared to the full cross-validation. We report the
  mean values over 50 repetitions and 1.96 standard errors. If CVST performs on par or better than the full cross-validation (i.e., MSE ratio plus 1.96 standard errors is bigger than 1.0) the values are in boldface.}
\label{tab:bench}
\end{table}

Figure~\ref{fig:Benchmark_Diffratio} shows the result for the
classification data sets (left side) and the regression data sets
(right side). The upper panels depict the relative gain in MSE of CVST
compared to the full cross-validation. For the classification tasks we
see that CVST is on par with the full cross-validation except for the
SVM for \emph{covtype} and KLR for \emph{ringnorm}. For the regression
task we observe that except for the \emph{blocks}, \emph{bumps} and
\emph{doppler} data sets CVST chooses reasonable parameter set.
Although for some problems the CVST algorithm picks a suboptimal
parameter set, even then the relative performance decreases are always
relatively small and range around 80\%. The learners have hardly any
impact on the behavior; just for the \emph{ringnorm} and the
\emph{blocks}, \emph{bumps} and \emph{doppler} data set we see a
strong difference of the corresponding methods. These findings can
also be observed in Table~\ref{tab:bench}: Given the mean of the
relative MSE ratio and the speed ratio with its corresponding 1.96
standard errors we mark entries in boldface where the MSE ratio plus
the 1.96 standard error is bigger than 1.0 indicating a performance on
par or better than the normal cross-validation. While CVST can tackle
most of the classification task we see a relative decline in the
regression tasks. But except for the \emph{blocks}, \emph{bumps} and
\emph{doppler} data sets the relative performance decrease ranges
around 99\% indicating a nearly optimal performance. 

In terms of speed gain we see a much more diverse and varying
picture. Overall, the speed improvements for KLR and KRR are higher
than for SVM and SVR and reach up to 120 times compared to normal
cross-validation. Regression tasks in general seem to be solved faster
than classification tasks, which can clearly be explained when we look
at the traces in Figure~\ref{fig:MS_Benchmark_TraceC}: For
classification tasks the number of kept configurations is generally
much higher than for the regression tasks. Furthermore we can observe
several types of difficulty of the learning problems. For instance the
\emph{german} data set seems to be much more difficult than the
\emph{ringnorm} data \citep[see][]{braun2008} which is also reflected
in the difference and speed improvement seen in the previous
figure. We will see in Section~\ref{sec:friedman} that we can trade
time for increasing the accuracy of CVST for the regression tasks
by leveraging the modular construction of the CVST procedure.

Since finding the top configurations inside the loop of the CVST
algorithm is a crucial step to the overall performance of the
procedure we further investigate how our choice of tests influence the
performance of the CVST procedure. Recall from Algorithm~\ref{alg:topConfigurations}
that we used the Cochran's Q test for classification experiments (see
line~\ref{alg:topConfigurationsClass}) and Friedman test for
regression problems (see line~\ref{alg:topConfigurationsReg}) to find
the top configurations in an iterative testing scheme. Both these test
are non-parametric and paired tests, i.e., they both take into account
the pointwise performance of a configuration and, thus, compare the
performance on individual data points. In Section~\ref{sec:transform}
we argued that this pooling effect robustifies the estimation of the
top configurations since outliers in the testing data do not have such
a dramatic effect on the test results compared to using for instance
the overall test error as input for the test. To verify this claim we
have replaced the tests in the Algorithm~\ref{alg:topConfigurations}
by unpaired, non-parametric versions which solely test whether the
test error is significantly different without taking the results of
the individual data points into account. To this end we have replaced
the Cochran's Q test for classification experiments on 
line~\ref{alg:topConfigurationsClass} of
Algorithm~\ref{alg:topConfigurations} by an unpaired version described
by \cite{wilson27} and the Friedman test by the Kruskal-Wallis rank
sum test \citep{Kruskal52}. Again we repeat the procedure for each
benchmark set 50 times to get reliable statistics.

The results are reported in Table~\ref{tab:all}, upper part. We can see that
the relative level of MSE is almost always around 1.0 if we take the 1.96 standard
error ranges into account. The upper plot of
Figure~\ref{fig:Benchmark_UnpairedSim_Diffratio} shows the
distribution of the MSE ratio. Except for the \emph{bumps}
and \emph{doppler} data set all distributions are clearly centered
around 1.0 with a narrow spread which further highlights the equality
of the two methods in terms of accuracy.

Comparing the ratio of the CVST procedure to the unpaired variant we
can see that the paired test variant improves the runtime of the
procedure significantly. This clearly demonstrates that the usage of
the paired tests which directly estimate the top configurations on the
pointwise predictions saves computation time with no impact on the
accuracy.

Since we compared the CVST algorithm to a full cross-validation it is
also of interest to see how CVST compares to a simple heuristic which
uses just 10\% of the data for the cross-validation. We have executed
this experiment for all benchmark data sets and repeated the procedure
50 times to get statistically sound estimates. The middle part of
Table~\ref{tab:all} reports the MSE ratio of the CVST
compared to the 10\% cross-validation annotated with their
corresponding 1.96 standard errors and again the speed ratio. The
first striking thing to observe is that the MSE ratios are all
significantly bigger than 1.0, indicating that CVST always finds a
better performing configuration than the simple 10\% heuristic. While
the accuracy impact varies across the different data sets CVST 
always picks significantly better performing configurations with a
modest impact on the runtime compared to the simple 10\%
heuristic. This trend can clearly be seen in the middle plot of
Figure~\ref{fig:Benchmark_UnpairedSim_Diffratio} which shows the
corresponding distributions of the relative level of MSE. For all data sets the
bulk of the distribution is above 1.0 indicating better performance
of the CVST method compared to the 10\% heuristic.

The last comparison of the performance of the CVST method is shown in
the lower part of Table~\ref{tab:all}. Here we show the 
MSE ratio of CVST compared to a random search as described in
\cite{Bergstra12}. In each step of the random search procedure we choose
parameters uniformly distributed over the range of the corresponding
grid of the CVST procedure and learn a full-data model. After having
spent the same amount of time as the CVST procedure on a specific data
set we stop the random search. Then we pick the best model of the so far
evaluated parameters and compare its performance to the CVST
model. Again, the lower part of Table~\ref{tab:all} shows the relative
gain in MSE and their corresponding 1.96 standard errors gathered over
50 repetitions for each data set. Both the table and the lower part of
Figure~\ref{fig:Benchmark_UnpairedSim_Diffratio} indicate that CVST
shows better performance compared to random search especially in the
regression data sets. In some cases (\emph{covtype} with SVM,
\emph{ringnorm} with SVM and \emph{splice} with SVM) the random search
outperforms CVST but in the majority of cases the CVST procedure can
extract better parameter configurations in the same amount of time
than the random search.

\begin{table}[p]
\centering
\footnotesize
\begin{tabular}{rrrr|rrrr}
  \toprule
\bf Data & \bf Method & \bf MSE Ratio & \bf Speed & \bf Data & \bf Method & \bf MSE Ratio & \bf Speed \\ 
  \midrule
banana & KLR & $ 1.001 \pm 0.010 $ & $ 2.9 $ &      bank & KRR & $ 0.984 \pm 0.005 $ & $ 3.0 $ \\        
  banana & SVM & $ 1.005 \pm 0.009 $ & $ 1.6 $ &    bank & SVR & $ 1.000 \pm 0.001 $ & $ 1.2 $ \\     
  covtype & KLR & $ 1.003 \pm 0.005 $ & $ 1.6 $ &   blocks & KRR & $ 0.957 \pm 0.030 $ & $ 1.4 $ \\   
  covtype & SVM & $ 0.992 \pm 0.007 $ & $ 1.2 $ &   blocks & SVR & $ 0.986 \pm 0.006 $ & $ 1.4 $ \\   
  german & KLR & $ 0.993 \pm 0.013 $ & $ 0.9 $ &    bumps & KRR & $ 0.920 \pm 0.074 $ & $ 1.0 $ \\    
  german & SVM & $ 0.984 \pm 0.019 $ & $ 0.6 $ &    bumps & SVR & $ 0.999 \pm 0.003 $ & $ 1.1 $ \\    
  image & KLR & $ 1.011 \pm 0.040 $ & $ 3.3 $ &     doppler & KRR & $ 0.954 \pm 0.051 $ & $ 1.3 $ \\  
  image & SVM & $ 1.030 \pm 0.047 $ & $ 1.5 $ &     doppler & SVR & $ 0.981 \pm 0.010 $ & $ 1.1 $ \\  
  ringnorm & KLR & $ 0.963 \pm 0.037 $ & $ 1.1 $ &  heavisine & KRR & $ 0.990 \pm 0.005 $ & $ 2.6 $ \\
  ringnorm & SVM & $ 1.005 \pm 0.013 $ & $ 1.0 $ &  heavisine & SVR & $ 0.995 \pm 0.003 $ & $ 6.9 $ \\
  splice & KLR & $ 0.992 \pm 0.020 $ & $ 1.6 $ &    kin & KRR & $ 1.000 \pm 0.003 $ & $ 6.8 $ \\      
  splice & SVM & $ 1.012 \pm 0.014 $ & $ 1.1 $ &    kin & SVR & $ 1.000 \pm 0.000 $ & $ 3.8 $ \\      
  twonorm & KLR & $ 0.988 \pm 0.029 $ & $ 0.9 $ &   pumadyn & KRR & $ 1.003 \pm 0.003 $ & $ 26.0 $ \\ 
  twonorm & SVM & $ 1.001 \pm 0.010 $ & $ 0.9 $ &   pumadyn & SVR & $ 1.000 \pm 0.001 $ & $ 18.4 $ \\ 
  waveform & KLR & $ 0.997 \pm 0.013 $ & $ 1.5 $ & \multicolumn{4}{c}{\multirow{2}{*}{\normalsize \it Unpaired Version of CVST} } \\
  waveform & SVM & $ 1.011 \pm 0.012 $ & $ 1.3 $ & \\
  \midrule
banana & KLR & $ \mathbf{ 1.056 \pm 0.036 }$ & $ 0.7 $ &     bank & KRR & $ \mathbf{ 1.073 \pm 0.024 }$ & $ 0.6 $ \\   
  banana & SVM & $ \mathbf{ 1.106 \pm 0.086 }$ & $ 0.3 $ &   bank & SVR & $ \mathbf{ 1.006 \pm 0.003 }$ & $ 0.7 $ \\   
  covtype & KLR & $ \mathbf{ 1.062 \pm 0.028 }$ & $ 0.9 $ &  blocks & KRR & $ \mathbf{ 1.428 \pm 0.102 }$ & $ 0.7 $ \\ 
  covtype & SVM & $ \mathbf{ 1.031 \pm 0.020 }$ & $ 0.4 $ &  blocks & SVR & $ \mathbf{ 1.189 \pm 0.033 }$ & $ 0.7 $ \\ 
  german & KLR & $ \mathbf{ 1.085 \pm 0.031 }$ & $ 1.5 $ &   bumps & KRR & $ \mathbf{ 1.947 \pm 0.380 }$ & $ 0.7 $ \\  
  german & SVM & $ \mathbf{ 1.132 \pm 0.061 }$ & $ 0.6 $ &   bumps & SVR & $ \mathbf{ 1.049 \pm 0.015 }$ & $ 0.7 $ \\  
  image & KLR & $ \mathbf{ 1.544 \pm 0.185 }$ & $ 0.4 $ &    doppler & KRR & $ \mathbf{ 1.896 \pm 0.186 }$ & $ 0.7 $ \\
  image & SVM & $ \mathbf{ 1.279 \pm 0.143 }$ & $ 0.8 $ &    doppler & SVR & $ \mathbf{ 1.210 \pm 0.035 }$ & $ 0.7 $ \\
  ringnorm & KLR & $ \mathbf{ 2.083 \pm 0.291 }$ & $ 1.3 $ & heavisine & KRR & $ \mathbf{ 1.104 \pm 0.022 }$ & $ 0.4 $ \\
  ringnorm & SVM & $ \mathbf{ 1.044 \pm 0.038 }$ & $ 0.5 $ & heavisine & SVR & $ \mathbf{ 1.042 \pm 0.011 }$ & $ 0.5 $ \\
  splice & KLR & $ \mathbf{ 1.106 \pm 0.067 }$ & $ 0.6 $ &   kin & KRR & $ \mathbf{ 1.074 \pm 0.030 }$ & $ 0.3 $ \\    
  splice & SVM & $ \mathbf{ 1.092 \pm 0.039 }$ & $ 0.6 $ &   kin & SVR & $ \mathbf{ 1.014 \pm 0.006 }$ & $ 0.7 $ \\    
  twonorm & KLR & $ \mathbf{ 1.197 \pm 0.103 }$ & $ 0.7 $ &  pumadyn & KRR & $ \mathbf{ 1.053 \pm 0.016 }$ & $ 0.4 $ \\
  twonorm & SVM & $ \mathbf{ 1.032 \pm 0.026 }$ & $ 0.4 $ &  pumadyn & SVR & $ \mathbf{ 1.026 \pm 0.007 }$ & $ 0.4 $ \\
  waveform & KLR & $ \mathbf{ 1.099 \pm 0.036 }$ & $ 0.7 $ & \multicolumn{4}{c}{\multirow{2}{*}{\normalsize \it Cross-Validation on 10\% of Data} } \\
  waveform & SVM & $ \mathbf{ 1.088 \pm 0.063 }$ & $ 0.4 $ & \\
  \midrule
banana & KLR & $ \mathbf{ 1.198 \pm 0.131 }$ & $ 1.0 $ &        bank & KRR & $ \mathbf{ 1.339 \pm 0.068 }$ & $ 1.0 $ \\    
  banana & SVM & $ 0.970 \pm 0.012 $ & $ 1.0 $ &                bank & SVR & $ \mathbf{ 3.070 \pm 0.376 }$ & $ 1.0 $ \\    
  covtype & KLR & $ \mathbf{ 1.185 \pm 0.044 }$ & $ 1.0 $ &     blocks & KRR & $ \mathbf{ 2.059 \pm 0.529 }$ & $ 1.0 $ \\  
  covtype & SVM & $ 0.886 \pm 0.032 $ & $ 1.0 $ &               blocks & SVR & $ \mathbf{ 1.526 \pm 0.314 }$ & $ 1.0 $ \\  
  german & KLR & $ \mathbf{ 1.128 \pm 0.098 }$ & $ 1.0 $ &      bumps & KRR & $ \mathbf{ 2.567 \pm 0.629 }$ & $ 1.0 $ \\   
  german & SVM & $ 1.020 \pm 0.024 $ & $ 1.0 $ &                bumps & SVR & $ \mathbf{ 4.461 \pm 0.511 }$ & $ 1.0 $ \\   
  image & KLR & $ \mathbf{ 2.188 \pm 0.512 }$ & $ 1.0 $ &       doppler & KRR & $ \mathbf{ 1.891 \pm 0.405 }$ & $ 1.0 $ \\ 
  image & SVM & $ \mathbf{ 1.133 \pm 0.064 }$ & $ 1.0 $ &       doppler & SVR & $ \mathbf{ 2.337 \pm 0.583 }$ & $ 1.0 $ \\ 
  ringnorm & KLR & $ \mathbf{ 5.972 \pm 2.315 }$ & $ 1.0 $ &    heavisine & KRR & $ \mathbf{ 1.075 \pm 0.012 }$ & $ 1.0 $ \\
  ringnorm & SVM & $ 0.648 \pm 0.094 $ & $ 1.0 $ &              heavisine & SVR & $ 1.012 \pm 0.043 $ & $ 1.0 $ \\         
  splice & KLR & $ \mathbf{ 2.062 \pm 0.389 }$ & $ 1.0 $ &      kin & KRR & $ \mathbf{ 1.283 \pm 0.049 }$ & $ 1.0 $ \\     
  splice & SVM & $ 0.874 \pm 0.022 $ & $ 1.0 $ &                kin & SVR & $ \mathbf{ 1.190 \pm 0.033 }$ & $ 1.0 $ \\     
  twonorm & KLR & $ 2.427 \pm 1.444 $ & $ 1.0 $ &               pumadyn & KRR & $ \mathbf{ 1.113 \pm 0.026 }$ & $ 1.0 $ \\ 
  twonorm & SVM & $ 0.939 \pm 0.027 $ & $ 1.0 $ &               pumadyn & SVR & $ \mathbf{ 1.070 \pm 0.025 }$ & $ 1.0 $ \\ 
  waveform & KLR & $ \mathbf{ 2.032 \pm 0.600 }$ & $ 1.0 $ & \multicolumn{4}{c}{\multirow{2}{*}{\normalsize \it Random Search} } \\
  waveform & SVM & $ 0.959 \pm 0.019 $ & $ 1.0 $ \\ 
  \bottomrule
\end{tabular}
\vspace{-.25cm}
\caption{Comparison of performance of the CVST algorithm to different competitors (details
  see text). MSE ratio is the relative gain in MSE of CVST compared to the other variant. Speed denotes the relative speed increase of CVST compared to the other variant. We report the
  mean values over 50 repetitions and 1.96 standard errors with
  significant better values of CVST in boldface.}
\label{tab:all}
\end{table}

\JMLRtripplefigure{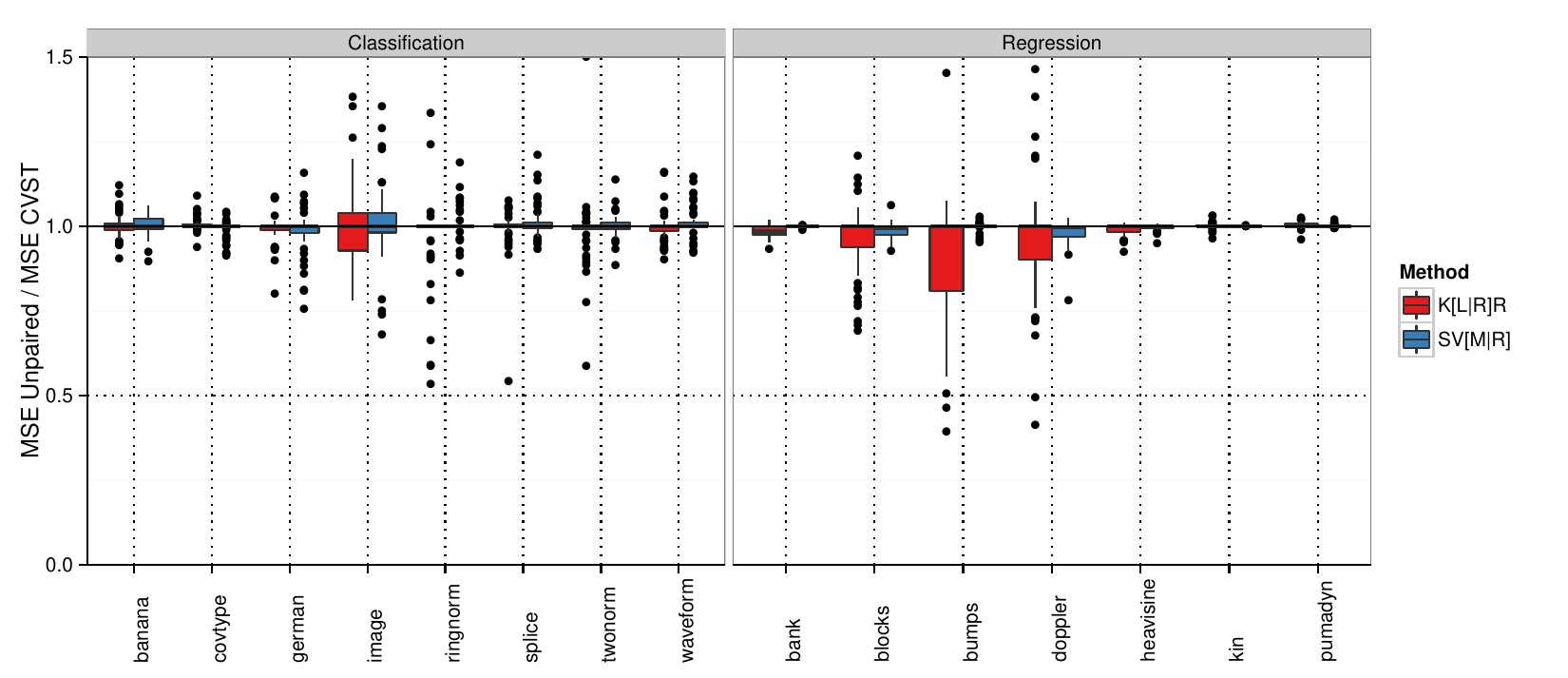}{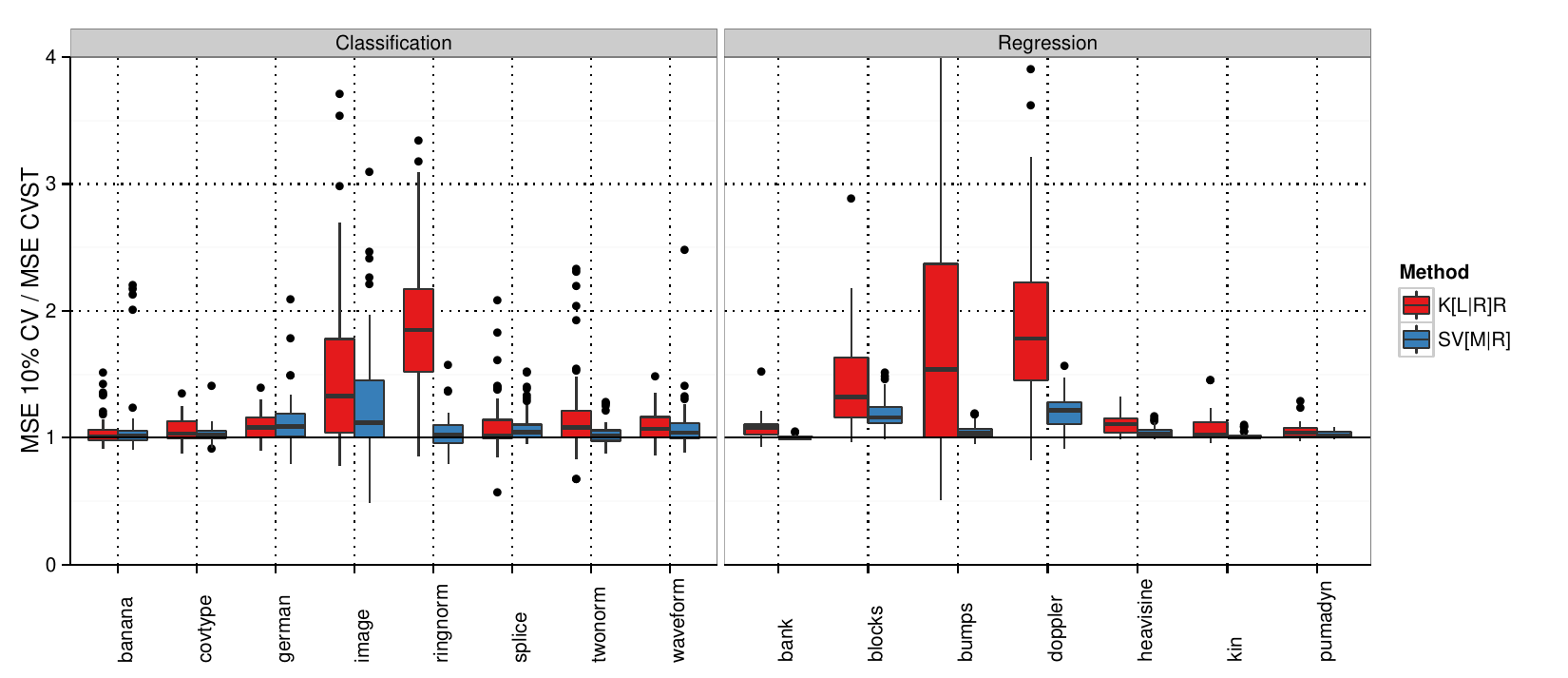}{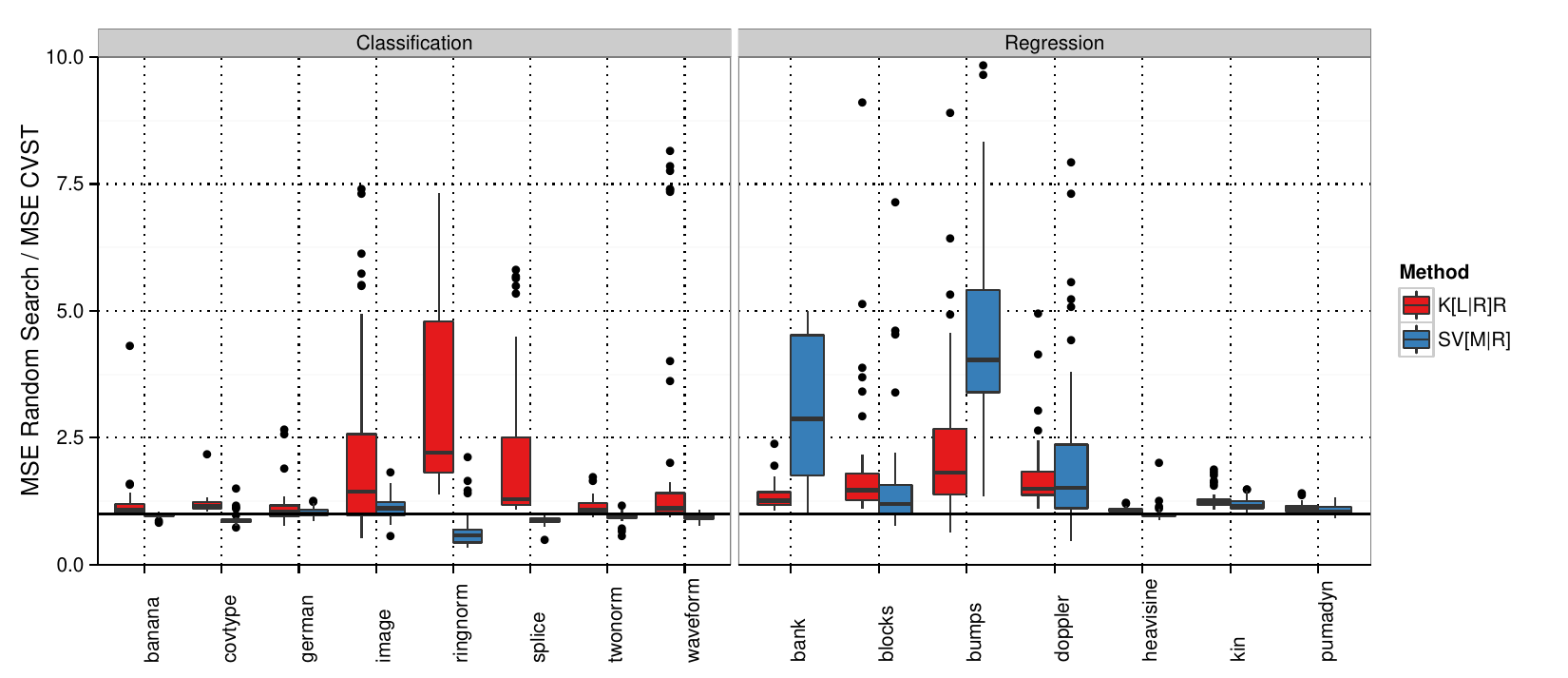}{1.0}{\vspace{-.5cm}Distributions
  of the MSE ratio of the CVST procedure compared to the
  unpaired variant (upper panel), the cross-validation on 10\% of the
  data (middle panel) and random search (lower panel). The horizontal line denotes the $1.0$, i.e., equal performance ratio. Values over 1.0 favor the CVST procedure.}

In summary, the evaluation of the benchmark data sets shows that the
CVST algorithm gives a huge speed improvement compared to the normal
cross-validation. While we see some non-optimal choices of
configurations, the total impact on the error is never exceptionally
high. We have to keep in mind that we have chosen the parameters of
our CVST algorithm to give an impression of the maximal attainable
speed-up: More conservative settings would trade computational time
for lowering the impact on the test error. The CVST outperforms both
unpaired variants of the procedure, the simple heuristics of
cross-validation on just 10\% of the data, and a random search in
parameter space. This clearly demonstrates that the individual parts
of the CVST procedure are well chosen and the combination of tests are
superior to other methods. Trading some speed compared to simpler
heuristics for more robust and stable estimates of optimal performing
configurations and the huge speed improvement compared to a full
cross-validation renders the CVST procedure as a promising candidate
for model selection in big data settings.

\section{Modularization and Extensions}
\label{sec:related}

\JMLRfigure{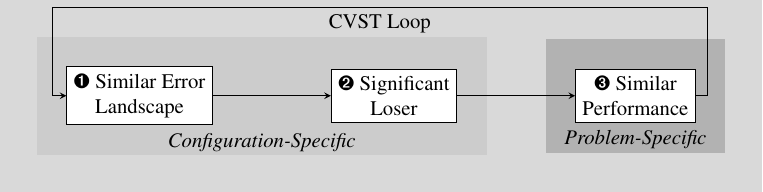}{.8}{Conceptual view of the CVST algorithm.
  Each execution of the loop body consists of a sequence of test, each
  delivering the input for the following test. This modular structure
  allows for customization of the CVST algorithm to special situations
  (multi-class experiments, structured learning etc.).\vspace{-.3cm}}

In this Section we will deal with several aspects of the CVST
algorithm: We illuminate the inner structure of the overall procedure
and discuss potential extensions and properties of specific steps.
The CVST algorithm consists of a sequence of tightly coupled
modules: The output of the top or flop test is the input for the
subsequent test for significant losers. The performance history of all
remaining configurations is then the input for the early stopping rule
which looks for similar performance of the remaining configurations on
the learning problem to capture the right point in time to stop the
CVST loop. This stepwise procedure is depicted in
Figure~\ref{fig:modularizedCVST}: While the tests for top or flop
configurations (step~\ding{202}) and the following sequential analysis
(step~\ding{203}) focuses solely on the individual configurations, the
early stopping rule (step~\ding{204}) acts on a global scope by
determining the right point to stop the CVST algorithm. Thus, we face
two kinds of test, namely the configuration-specific and the
problem-specific tests.

To complete our discussion of the CVST algorithm, we focus on the
configuration-specific procedures. First, we analyze the inner
structure of the similarity test based on the error landscape in
Section~\ref{sec:friedman} and how this module can be 
adjusted for specific side constraints. Furthermore, in
Section~\ref{sec:seqtest} we look at the suitability of the
sequential analysis for determining significant loser configurations.
It is shown that a so-called \emph{closed} sequential test lacks
essential properties of the open variant of Wald used in the CVST
algorithm, which further underlines the appropriateness of the open test of
Wald for the learning on increasing subsets of data.

\subsection{Checking the Similarity of the Error Landscape}
\label{sec:friedman}

In the evaluation of the CVST method in Section~\ref{sec:experiments}
we see that the Friedman test for the regression case shows a much
more aggressive behavior than the Cochran's Q test used in the
top or flop conversion in the classification case. This feature can be
clearly seen in Figure~\ref{fig:MS_Benchmark_TraceC} where the dropping
rates of the classification and regression benchmark data sets can be
easily compared. Since the Friedman test acts on the squared residuals
it uses more information compared to the classification task where we
just have the information whether a specific data point was correctly
classified or not. Thus, the Friedman test can exploit the higher
detail of the information and can decide much faster than the Cochran's Q test
which of the configurations are significantly different from the top
performing ones.

In this section we show how the modular design of the CVST algorithm
can be utilized to fit a less aggressive, yet more robust similarity
test for regression data into the overall framework. It comes as no
surprise that this increased tolerance affects the runtime of the CVST
procedure. In the following we will first develop the alternative
similarity test and then compare its performance both on the toy and
the benchmark data sets to the original Friedman variant.

Recall from Section~\ref{sec:transform} that the top or flop assignment
was calculated in a sequential manner: First we order all remaining
configurations according to their mean performance; then we check at
which point the addition of another configuration shows a
significantly different behavior compared to all other, better
performing configurations. To employ a less strict version of the
Friedman test we drop the actual residual information and instead use
the outlier behavior of a configuration for comparison. To this end
we assume that the residuals are normally distributed with
mean zero and a configuration-dependent variance $\sigma_c^2$ which we
estimate from the actual residuals. Now we can check for each
calculated residual whether it exceeds the $\frac{\alpha}{2}$
confidence interval around zero by using the normality
assumption, thus converting the raw residuals in a binary information
whether it is deemed as an outlier or not. Similar to the
classification case this binary matrix forms the input to the
Cochran's Q test which then asserts whether a specific configuration
belongs to the top-performing ones or not.

The results of this procedure on the \emph{noisy sinc} data set is
shown in Figure~\ref{fig:paramQM_Toy_noisySinc_Difference}: Compared
to the outcome of the Friedman test in
Figure~\ref{fig:MS_Toy_noisySinc_Difference} we can clearly see that
the conservative nature of the outlier-based test helps in finding the
correct parameter configuration. Obviously its higher retention rate
leads to lower runtime performance: The speed ratio drops roughly by a
factor of $\frac{2}{3}$. A similar behavior can be observed on the
\emph{benchmark} data sets in
Figure~\ref{fig:paramQM_Benchmark_Diffratio}: The conservative
behavior of the outlier-based measure increases the  
MSE ratio compared to the residual-based test, but at the same time lowers the
speed ratio. Interestingly, for the \emph{benchmark} data sets the
speed impact on the SVR is much lower compared to the speed ratio
decrease of the KRR method. We can observe this shift also in the
number of kept configurations shown in
Figure~\ref{fig:paramQM_Toy_noisySinc_Trace} both for the \emph{noisy
  sinc} and the \emph{benchmark} data sets.

The conclusion of this discussion is two-fold: First, this section
shows how the modular construction of the CVST methods allows for the
exchange of the individual parts of the algorithm without disrupting
the workflow of the procedure. If the residual-based test turns out to
be unsuitable for a given regression problem, it is extremely easy to
devise an adapted version for instance by looking at the outlier
behavior of the configurations. Second, we see the inherent
flexibility of the CVST algorithm. If there is a need for different
error measures (for instance multi-class experiments, structured
learning etc.), the modularized structure of the CVST algorithms allows for
maximal flexibility and adaptability to special cases. 

\JMLRdoublefigure{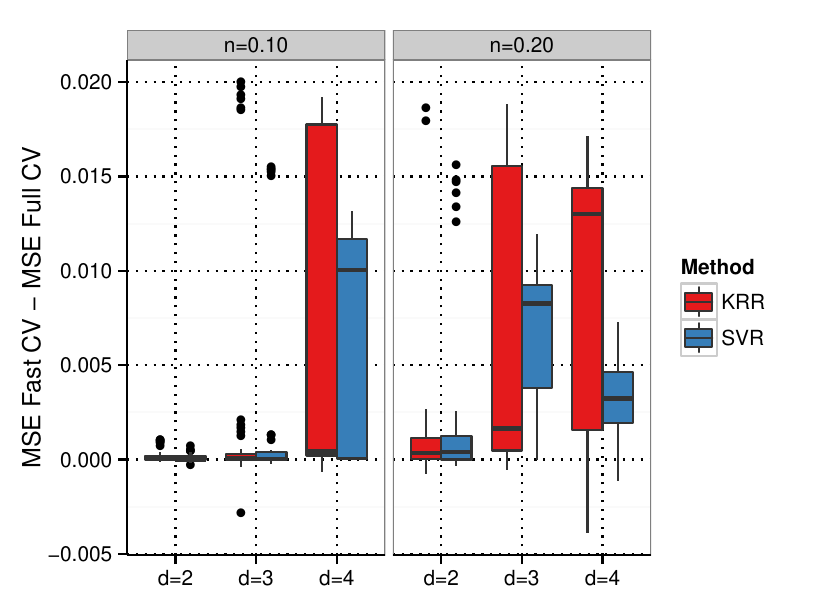}{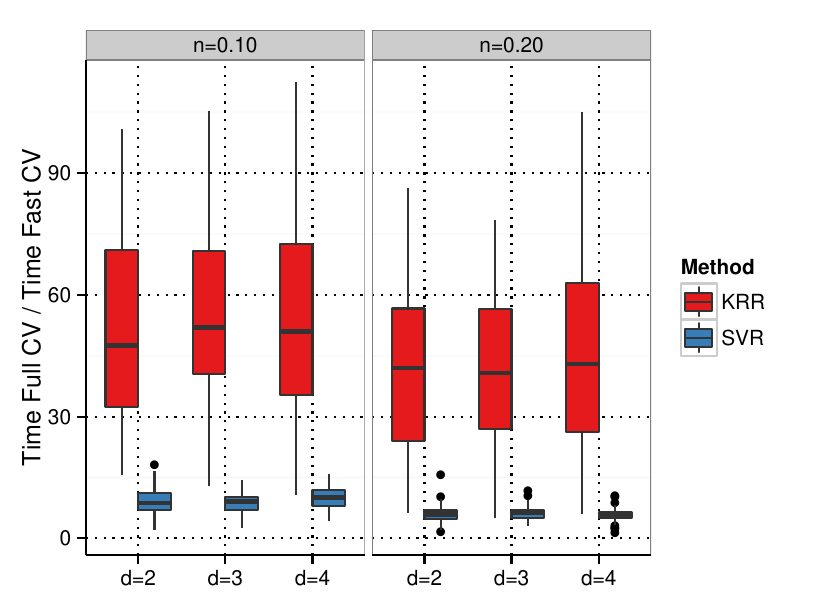}{.48}{Difference
  in mean square error (left plots) and relative speed gain (right
  plots) for the \emph{noisy sinc} data set using the outlier-based
  similarity test. In comparison to the stricter Friedman test used in
  Figure~\ref{fig:MS_Toy_noisySinc_Difference} we can observer a more
  conservative behavior resulting in increased robustness at the
  expense of performance.}

\JMLRdoublefigure{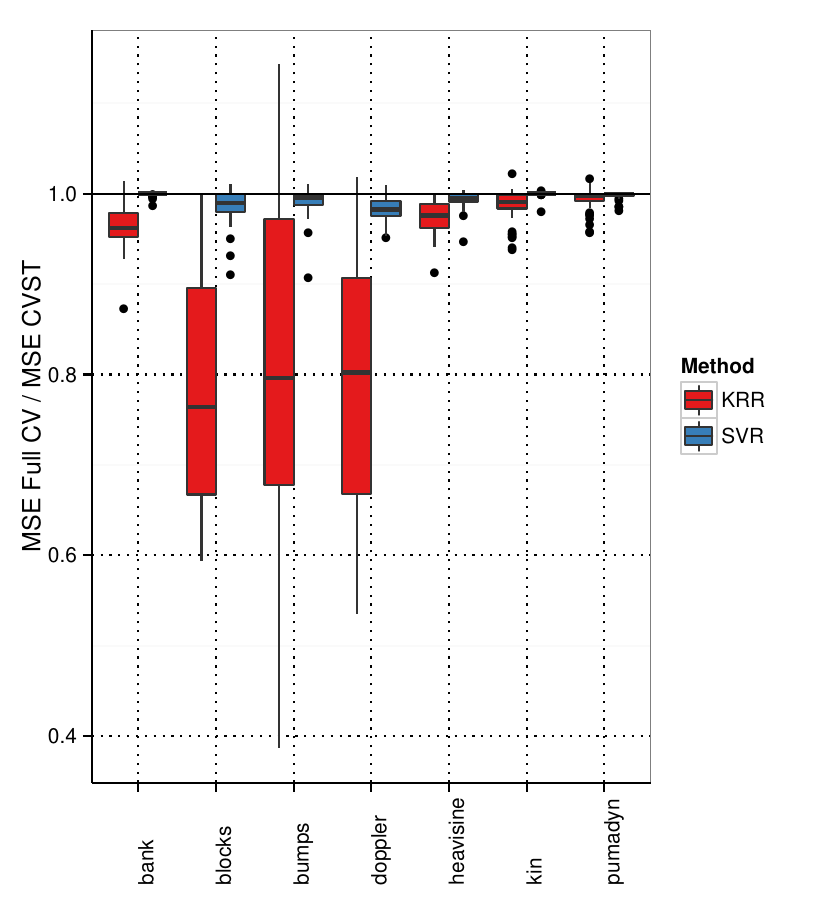}{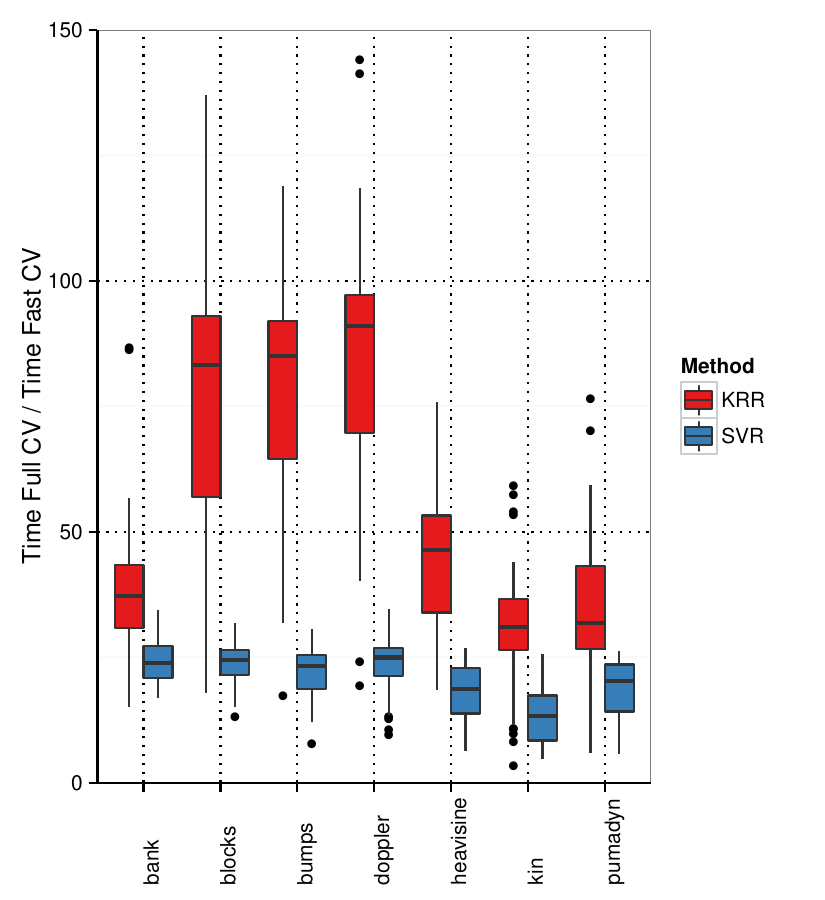}{.48}{MSE ratio (left plot) and relative speed gain (right
  plot) for the \emph{benchmark} data sets using the outlier-based
  similarity test. Compared to Figure~\ref{fig:Benchmark_Diffratio}
  we can see better accuracy but decreased speed performance.}

\JMLRdoublefigure{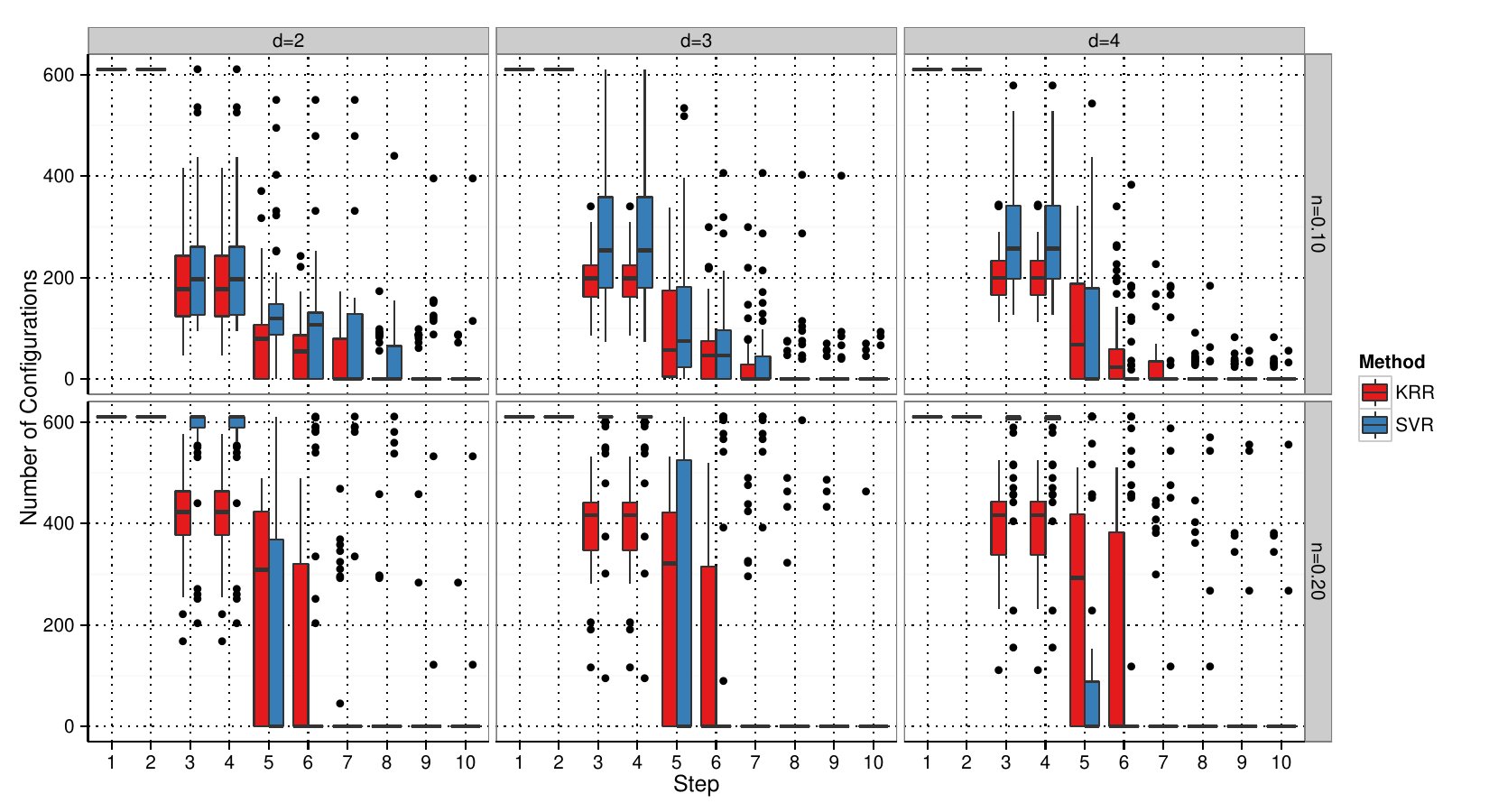}{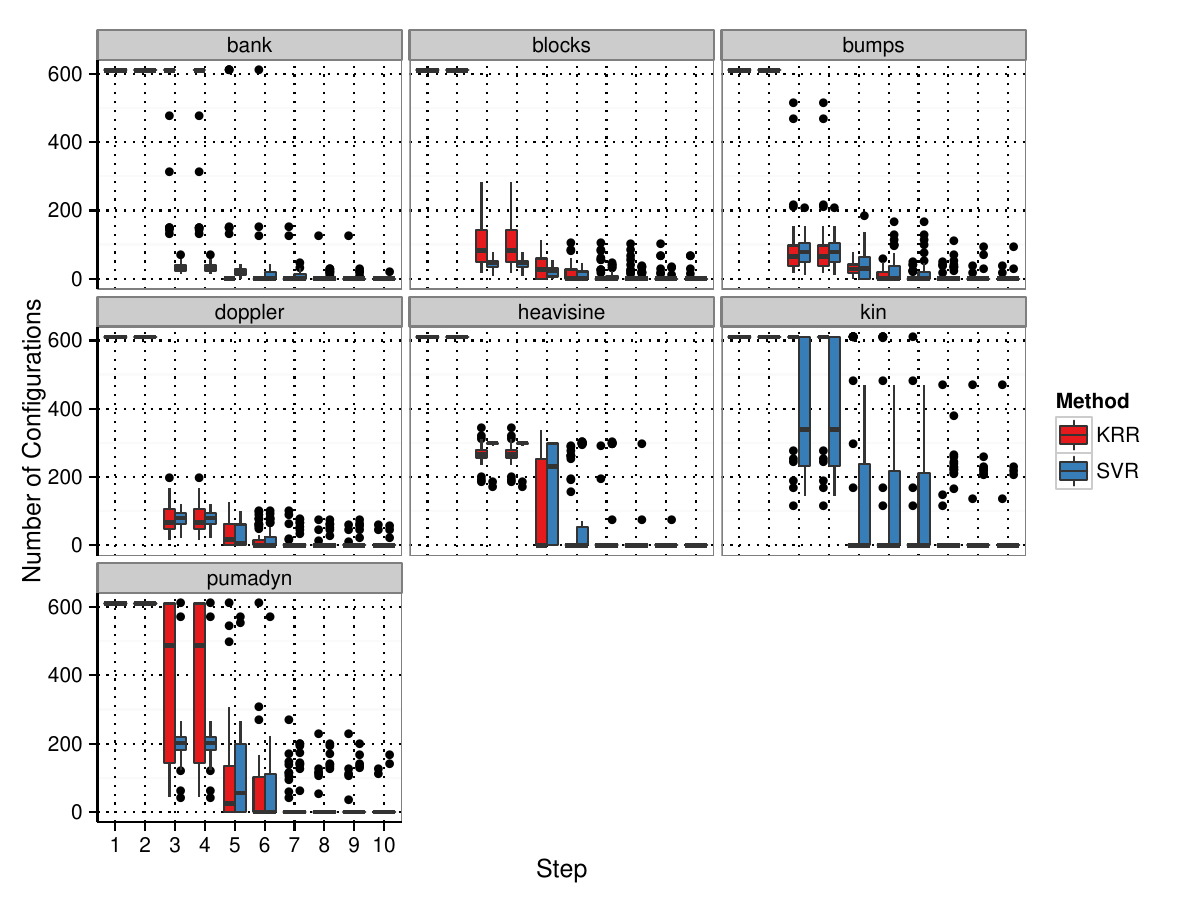}{.9}{Remaining
  configurations after each step for the \emph{noisy sinc} and
  different \emph{benchmark} data sets using the outlier-based
  similarity test. Compared to Figure~\ref{fig:MS_Toy_noisySinc_Trace}
  and Figure~\ref{fig:MS_Benchmark_TraceC} we can clearly observe the
  higher retention rate of this more conservative test.}

\subsection{Determining Significant Losers: Open versus Closed Sequential Testing}
\label{sec:seqtest}
As already introduced in Section~\ref{sec:seqana} the sequential
testing was pioneered by \cite{wald47}; the test monitors a likelihood
ratio of a sequence of i.i.d. Bernoulli variables $B_1, B_2, \dots$:
\[
\ell = \prod_{i=1}^{n}f(b_i, \pi_1) /
\prod_{i=1}^{n}f(b_i, \pi_0)\quad\text{given } H_h: B_i \sim \pi_h, h \in \{0, 1\}.
\] 
Hypothesis $H_1$ is accepted if $\ell \geq A$ and contrary $H_0$ is
accepted if $\ell \leq B$. If neither of these conditions apply, the
procedure cannot accept either of the two hypotheses and needs more
data. $A$ and $B$ are chosen such that the error probability of the
two decisions does not exceed $\alpha_l$ and $\beta_l$
respectively. In \cite{wald48} it is proven that the open sequential
probability ratio test of Wald is optimal in the sense that compared
to all tests with the same power it requires on average fewest
observations for a decision. The testing scheme of Wald is called
\emph{open} since the procedure could potentially go on forever, as
long as $\ell$ does not leave the $(A,B)$-tunnel.

The open design of Wald's procedure led to a development of a
different kind of sequential tests, where the number of observations
is fixed beforehand
\citep[see][]{armitage60,spicer62,alling66,mcpherson71}. For instance
in clinical studies it might be impossible or ethically prohibitive to
use a test which potentially could go on forever. Unfortunately, none
of these so-called closed tests exhibit an optimality criterion,
therefore we choose one which at least in simulation studies showed
the best behavior in terms of average sample number statistics: The
method of \cite{spicer62} is based on a gambler's ruin scenario in
which both players have a fixed fortune and decide to play for $n$
games. If $f(n, \pi, F_a, F_b)$ is the probability that a player with
fortune $F_a$ and stake $b$ will ruin his opponent with fortune $F_b$
in exactly $n$ games, then the following recurrence holds:
\begin{equation*}
f(n, \pi, F_a, F_b) =
\begin{cases}
  0 \quad\text{if } F_a < 0 \vee (n = 0 \wedge F_b > 0),\\
  1 \quad\text{if } n = 0 \wedge F_a > 0 \wedge F_b \leq 0,\\
  \pi f(n-1, \pi, F_a+1, F_b-b) \\
  \qquad + (1-\pi) f(n-1, \pi, F_a-b, F_b+b) \quad \text{otherwise.}\\  
\end{cases}
\end{equation*}
In each step, the player can either win a game with probability $\pi$ and
win 1 from his opponent or lose the stake $b$ to the other
player. Now, given $n = x + y$ games of which player $A$ has won $y$
and player $B$ has won $x$, the game will stop if either of the
following conditions hold:
\[
y - bx = - F_a \Leftrightarrow y = \frac{b}{1+b}n -
\frac{F_a}{1+b} \quad \text{or} \quad y - bx = F_b \Leftrightarrow y = \frac{b}{1+b}n -
\frac{F_b}{1+b}.
\]
This formulation casts the gambler's ruin problem into a Wald-like
scheme, where we just observe the cumulative wins of player $A$ and
check whether we reached the lower or upper line. If we now choose
$F_a$ and $F_b$ such that $f(n, 0.5, F_a, F_b) \leq \alpha_l$, we
construct a test which allows us to check whether a given
configuration performs worse than $\pi = 0.5$ (i.e., crosses the lower
line) and can therefore be flagged as an overall loser with controlled
error probability of $\alpha_l$ (see \citealt{alling66}). For more details
on the closed design of Spicer please consult \cite{spicer62}.

Since simulation studies show that the closed variants of the
sequential testing exhibit low average sample number statistics, we
first have a look at the runtime performance of the CVST algorithm
equipped with either the open or the closed sequential test. The most
influential parameter in terms of runtime is the $\steps$
parameter. In principle, a larger number of steps leads to more robust
estimates, but also to an increase of computation time. We study the
effect of different choices of this parameter in a simulation. For the
sake of simplicity we assume that the binary top or flop scheme
consists of independent Bernoulli variables with $\pi_\text{winner}
\in [0.9, 1.0]$ and $\pi_\text{loser} \in [0.0, 0.1]$. We test both
the open and the closed sequential test and compare the relative
speed-up of the CVST algorithm compared to a full 10-fold
cross-validation in case the learner is cubic.

\JMLRfigure{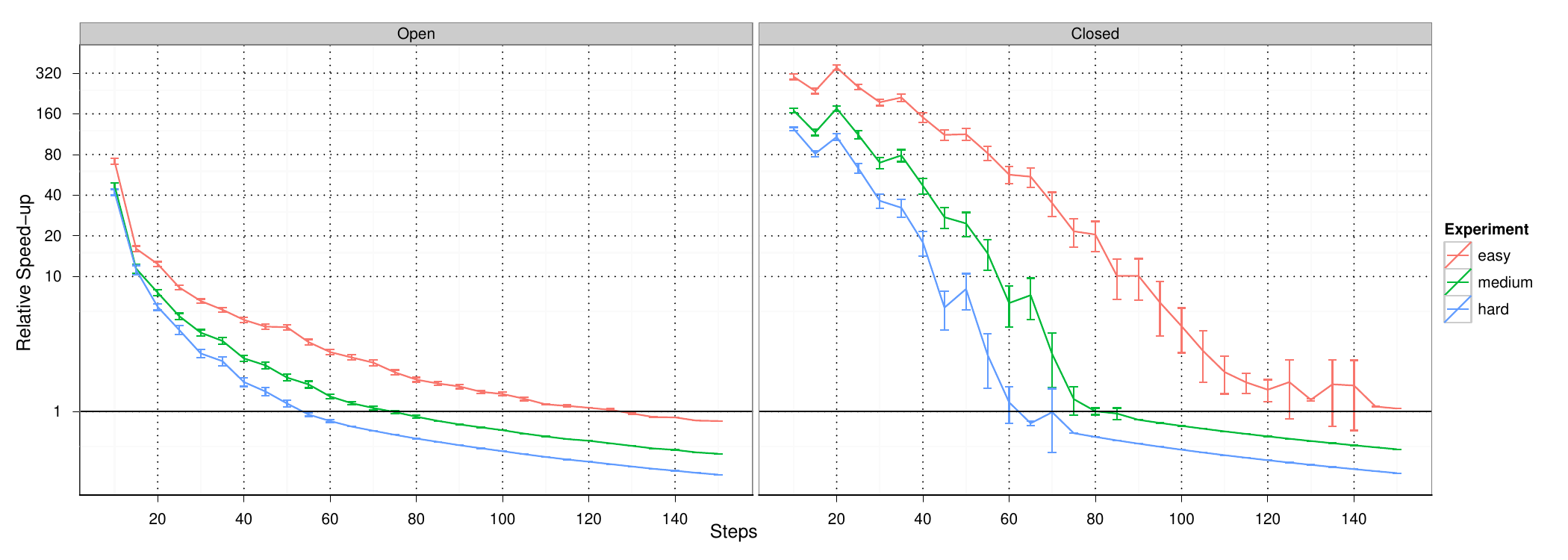}{1}{Relative speed gain of fast
  cross-validation compared to full cross-validation. We assume that
  training time is cubic in the number of samples. Shown are runtimes
  for 10-fold cross-validation on different problem classes by
  different loser/winner ratios (easy: 3:1; medium: 1:1, hard: 1:3)
  over 200 resamples.}

Figure~\ref{fig:speedSimulationsResults} shows the resulting simulated
runtimes for different settings. The overall speed-up is much higher
for the closed sequential test indicating a more aggressive behavior
compared to the more conservative open alternative. Both tests show
their highest increase in the range of 10 to 20 steps with a rapid
decline towards the higher step numbers. So in terms of speed the closed
sequential test definitely beats the more conservative open
test. 

\JMLRfigure{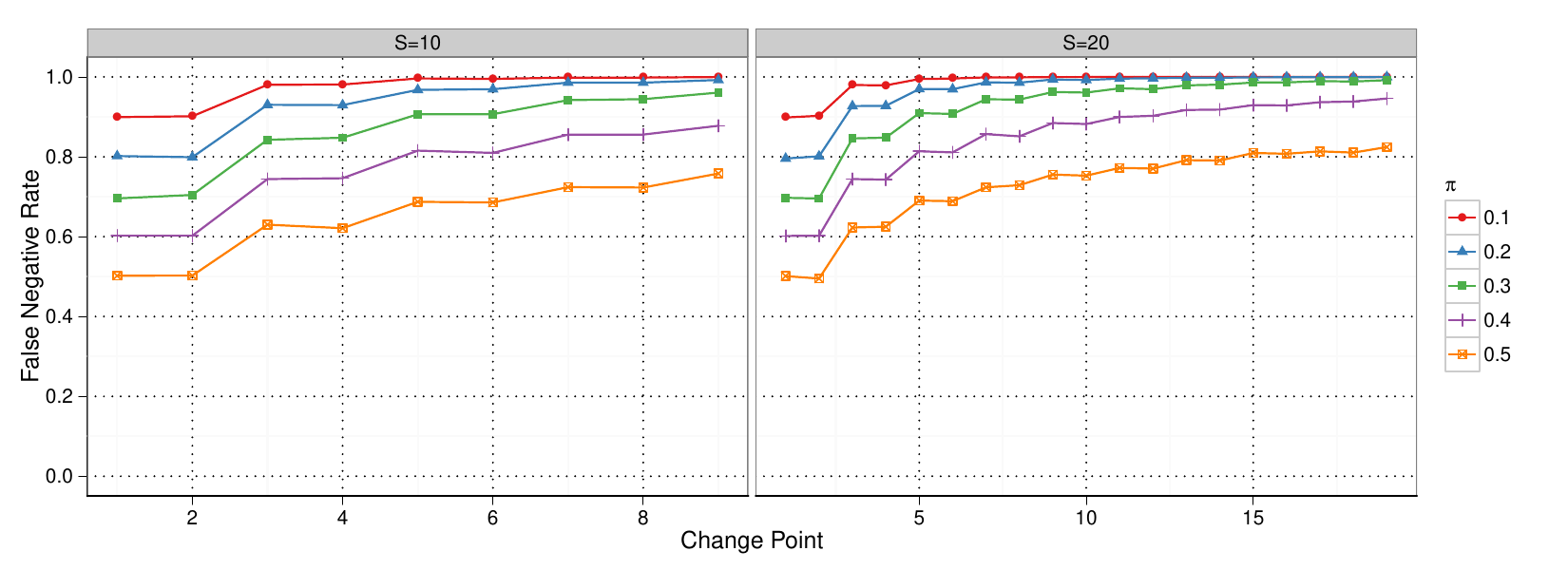}{1}{False negatives
  generated with the closed sequential test for non-stationary
  configurations, i.e., at the given change point the Bernoulli
  variable changes its $\pi_\text{before}$ from the indicated value to
  1.0.}

To evaluate the false negatives of the closed sequential test we
simulate switching configurations by independent Bernoulli variables
which change their success probability $\pi$ from a chosen
$\pi_\text{before} \in \{0.1, 0.2, \dots, 0.5\}$ to a constant $1.0$
at a given change point. By using this setup we mimic the behavior of
a switching configuration which starts out as a loser and after enough
data is available turns into a constant winner. The results can be
seen in Figure~\ref{fig:falseNegativeSimulationsResultsClosed} which
reveals that the speed gain comes at a price: Apart from having no
control over the safety zone, the number of falsely dropped
configurations is much higher than for the open sequential test (see
Figure~\ref{fig:falseNegativeSimulationsResultsOpen} in
Section~\ref{sec:fnr}). While having a definitive advantage over
the open test in terms of speed, the false negative rate of the closed
test renders it useless for the CVST algorithm.

\section{Conclusion}
\label{sec:conclusion}

We presented a method to speed up the cross-validation procedure by
starting at subsets of the full training set size, identifying
clearly underperforming parameter configurations early on and focusing on
the most promising candidates for the larger subset sizes. We have
discussed that taking subsets of the data set has theoretical
advantages when compared to other heuristics like local search on the
parameter set because the effects on the test errors are systematic
and can be understood statistically. On the one hand, we argued that
the optimal configurations converge to the true ones as sample sizes
tend to infinity, but we also discussed in a concrete setting how the
different behaviors of estimation error and approximation error lead
to much faster convergence practically. These insights led to the
introduction of a safety zone through sequential testing, which
ensures that underperforming configurations are not removed
prematurely when the minima are not converged yet. In experiments we
showed that our procedure leads to a speed-up of up to 120 times
compared to the full cross-validation without a significant increase
in prediction error.

It will be interesting to combine this method with other
procedures like the Hoeff\-ding races or algorithms for multi-armed
bandit problems. Furthermore, getting accurate convergence bounds even
for finite sample size settings is another topic for future
research.

\acks{We would like to acknowledge support for this project from the
  BMBF project ALICE, ``Autonomous Learning in Complex Environments''
  (01IB10003B). We would like to thank Klaus-Robert Müller, Marius
  Kloft, Raphael Pelossof and Ralf Herbrich for fruitful discussions
  and help. We would also like to thank the anonymous reviewers who
  have helped to clarify a number of points and further improve the
  paper.}

\newpage

\appendix

\section{Sample-Size Independent Parametrization}
\label{app:cond}

In Section~\ref{sec:cvsubsets} we needed as a precondition that the
test performances converge for a fixed parameter configuration $c$ as
$n$ tends to infinity. In this section, we discuss this condition in
the context of the empirical risk minimization. We refer to the book
by \citet{devroy96} and the references contained therein for the
theoretical results.

In the empirical risk minimization framework, a learning algorithm is
interpreted as choosing the solution $g_n$ with the best error on the
training set $\hat R_n(g) = \frac 1n \sum_{i=1}^n \ell(g(X_i), Y_i)$
from some hypothesis class $\hypocl$. If the VC-dimension of
$\hypocl$, which roughly measures the complexity of $\hypocl$, is
finite then it holds that
\begin{equation*}
  \hat R_n(g) \to R(g)
\end{equation*}
uniformly over $g\in\hypocl$, and consequently also
$R(g_n)\to \inf_{g\in\hypocl} R(g)$.

Now in order to make the link to our condition, we need that each
parameter $c$ corresponds to a fixed hypothesis class $\hypocl_c$ (and
not depend on the sample size in some way), and that the VC-dimension
is finite. For feed-forward neural networks, one can show, for
example, that neural networks with one hidden layer with $k$ inner
nodes and sigmoid activation function have finite VC-dimension
\cite[Theorem 30.6]{devroy96}.

For kernel machines, we consider the reproducing kernel Hilbert space
(RKHS, \citealt{aronszajn50}) view: Let $\rkhs_k$ the RKHS induced by a
Mercer kernel $k$ with norm $\|\cdot\|_{\rkhs_k}$. \citet{evgeniou99}
show that the $V_\gamma$-dimension of the hypothesis class $\hypocl(A)
= \{f\in \rkhs_k \mid \|f\|^2_{\rkhs_k} \le A\}$ is finite, from which
uniform converges of the kind described above follows, and thus also
that our condition holds.

Many kernel methods, including kernel ridge regression and support
vector machines can be written as regularized optimization problems in
the RKHS of the form:
\begin{equation*}
  \min_{f\in \rkhs_k} \left(\frac1n\sum_{i=1}^n \ell(f(X_i), Y_i) + C
  \|f\|^2_{\rkhs_k}\right)
  = \min_{f\in \rkhs_k}\left(\hat R_n(f) + C\|f\|^2_{\rkhs_k}\right).
\end{equation*}
Now if we assume that $\ell(f(x), y)$ is bounded by $B$ and continuous
in $f$, it follows that the minimum is attained for some $f$ with
$\|f\|^2_{\rkhs_k} \le B/C$: For $\|f\|_{\rkhs_k} = 0$, $\hat R_n(f) +
C\|f\|^2_{\rkhs_k} \le B$, and for $\|f\|_{\rkhs_k} > B/C$, $\hat R_n(f) +
C\|f\|^2_{\rkhs_k} > B$. Because $\hat R(f) + C\|f\|^2_{\rkhs_k}$ is
continuous in $f$, it follows that the minimum is somewhere in-between.

Now $\hat R_n(f)$ converges to $R(f)$ uniformly over $f\in\hypocl(B/C)$,
such that there exists an $A \le B/C$ such that
\begin{equation*}
  \min_{f\in \rkhs_k}\left(\hat R(f) + C\|f\|^2_{\rkhs_k}\right)
  = \min_{f\in \hypocl(A)} R(f),
\end{equation*}
and we see that a regularization constant $C$ corresponds to a fixed
hypothesis class $\hypocl(A)$ and our condition holds again.

As a direct consequence of this discussion we have to take care of the
correct scaling of the regularization constants during the CVST
run. Thus, for kernel ridge regression we have to scale the $\lambda$
parameter linearly with the data set size and for the SVR divide the
$C$ parameter accordingly (see the reference implementation in the
official CRAN package named \texttt{CVST} or the development version
at \url{https://github.com/tammok/CVST} for details).

\section{Example Run of CVST Algorithm}
\label{sec:exampleRun}
In this section we give an example of the whole CVST algorithm on one
\emph{noisy sinc} data set of $n=1,000$ data points with intrinsic
dimensionality of $d = 2$. The CVST algorithm is executed with $\steps
= 10$ and $\stoppingWindow=4$. We use a $\nu$-SVM
\citep{Scholkopf:2000} and test a parameter grid of $\log_{10}(\sigma)
\in \{-3, -2.9, \dots , 3\}$ and $\nu \in \{0.05, 0.1, \dots,
0.5\}$. The procedure runs for 4 steps after which the early stopping
rule takes effect. This yields the following traces matrix (only
remaining configurations are shown):

\begin{center}
\begin{tabular}{rrrrr}
  \toprule
 & $\mathbf{\modelsize=90}$ & $\mathbf{\modelsize=180}$ & $\mathbf{\modelsize=270}$ & $\mathbf{\modelsize=360}$ \\ 
  \midrule
$\log_{10}(\sigma)=-2.3, \nu=0.35$ & 0 & 0 & 1 & 0 \\ 
  $\log_{10}(\sigma)=-2.3, \nu=0.40$ & 0 & 1 & 1 & 0 \\ 
  $\log_{10}(\sigma)=-2.3, \nu=0.45$ & 0 & 1 & 0 & 1 \\ 
  $\log_{10}(\sigma)=-2.2, \nu=0.30$ & 0 & 1 & 0 & 0 \\ 
  $\log_{10}(\sigma)=-2.2, \nu=0.35$ & 0 & 1 & 1 & 0 \\ 
  $\log_{10}(\sigma)=-2.2, \nu=0.40$ & 0 & 1 & 1 & 1 \\ 
  $\log_{10}(\sigma)=-2.2, \nu=0.45$ & 0 & 1 & 1 & 1 \\ 
  $\log_{10}(\sigma)=-2.2, \nu=0.50$ & 0 & 0 & 1 & 1 \\ 
  $\log_{10}(\sigma)=-2.1, \nu=0.35$ & 0 & 1 & 1 & 1 \\ 
\hline
  $\log_{10}(\sigma)=-2.1, \nu=0.40$ & 0 & 1 & 1 & 1 \\ 
\hline
  $\log_{10}(\sigma)=-2.1, \nu=0.45$ & 0 & 1 & 1 & 1 \\ 
  $\log_{10}(\sigma)=-2.1, \nu=0.50$ & 1 & 0 & 1 & 1 \\ 
  $\log_{10}(\sigma)=-2.0, \nu=0.50$ & 0 & 0 & 1 & 1 \\ 
   \bottomrule
\end{tabular}
\end{center}

The corresponding mean square errors of the remaining configurations
after each step are shown in the next matrix. Based on these values,
the winning configuration, namely $\log_{10}(\sigma)=-2.1, \nu=0.40$ is chosen:

\begin{center}
\begin{tabular}{rrrrr}
  \toprule
 & $\mathbf{\modelsize=90}$ & $\mathbf{\modelsize=180}$ & $\mathbf{\modelsize=270}$ & $\mathbf{\modelsize=360}$ \\ 
  \midrule
$\log_{10}(\sigma)=-2.3, \nu=0.35$ & 0.0370 & 0.0199 & 0.0145 & 0.0150 \\ 
  $\log_{10}(\sigma)=-2.3, \nu=0.40$ & 0.0362 & 0.0197 & 0.0146 & 0.0146 \\ 
  $\log_{10}(\sigma)=-2.3, \nu=0.45$ & 0.0356 & 0.0197 & 0.0146 & 0.0144 \\ 
  $\log_{10}(\sigma)=-2.2, \nu=0.30$ & 0.0365 & 0.0195 & 0.0146 & 0.0148 \\ 
  $\log_{10}(\sigma)=-2.2, \nu=0.35$ & 0.0351 & 0.0193 & 0.0142 & 0.0145 \\ 
  $\log_{10}(\sigma)=-2.2, \nu=0.40$ & 0.0345 & 0.0194 & 0.0143 & 0.0141 \\ 
  $\log_{10}(\sigma)=-2.2, \nu=0.45$ & 0.0340 & 0.0193 & 0.0143 & 0.0140 \\ 
  $\log_{10}(\sigma)=-2.2, \nu=0.50$ & 0.0332 & 0.0200 & 0.0145 & 0.0138 \\ 
  $\log_{10}(\sigma)=-2.1, \nu=0.35$ & 0.0353 & 0.0194 & 0.0144 & 0.0142 \\ 
\hline
  $\log_{10}(\sigma)=-2.1, \nu=0.40$ & 0.0343 & 0.0195 & 0.0142 & 0.0138 \\ 
\hline
  $\log_{10}(\sigma)=-2.1, \nu=0.45$ & 0.0340 & 0.0197 & 0.0140 & 0.0138 \\ 
  $\log_{10}(\sigma)=-2.1, \nu=0.50$ & 0.0329 & 0.0199 & 0.0142 & 0.0137 \\ 
  $\log_{10}(\sigma)=-2.0, \nu=0.50$ & 0.0351 & 0.0204 & 0.0145 & 0.0137 \\ 
   \bottomrule
\end{tabular}
\end{center}

\section{Nonparametric Tests}
\label{sec:tests}
The tests used in the CVST algorithm are common tools in the field of
statistical data analysis. Here we give a short summary based on 
\citet{filliben} and cast the notation into the CVST framework
context. Both methods deal with the performance matrix of $K$
configurations with performance values on $r$ data points:

\begin{center}
\begin{tabular}{ccccc}
  \toprule
  & \multicolumn{4}{c}{\bf Data Points} \\
  \cmidrule{2-5}
  \bf Configuration & 1 & 2 & $\hdots$ & $r$ \\
  \midrule
  1 & $x_{11}$ & $x_{12}$ & $\hdots$ & $x_{1r}$ \\
  2 & $x_{21}$ & $x_{22}$ & $\hdots$ & $x_{2r}$ \\
  3 & $x_{31}$ & $x_{32}$ & $\hdots$ & $x_{3r}$ \\
  $\vdots$ &   $\vdots$ &  $\vdots$ &  &  $\vdots$ \\
  $K$ & $x_{K1}$ & $x_{K2}$ & $\hdots$ & $x_{Kr}$ \\
  \bottomrule
\end{tabular}
\end{center}

Both tests treat similar questions (``Do the $K$ configurations have
identical effects?'') but are designed for different kinds of data:
Cochran's Q test is tuned for binary $x_{ij}$ while the Friedman test
acts on continuous values. In the context of the CVST algorithm the
tests are used for two different tasks:
\begin{enumerate}
\item Determine whether a set of configurations are the top performing
  ones (step \ding{202} in the overview Figure~\ref{fig:overview} and the function
  \textsc{topConfigurations} defined in
  Algorithm~\ref{alg:topConfigurations}). \label{exp1}
\item Check whether the remaining configurations behaved similar in
  the past (step \ding{204} in the overview Figure~\ref{fig:overview}
  and the function \textsc{similarPerformance} in
  Algorithm~\ref{alg:similarPerformance}). \label{exp2}
\end{enumerate}

In both cases, the configurations are compared either by the
performance on the samples (Point~\ref{exp1} above) or on the last
\stoppingWindow traces (Point~\ref{exp2} above) of the remaining
configurations.  Depending on the learning problem either the Friedman
test for regression task or the Cochran's Q test for classification
tasks is used in Point~\ref{exp1}.

In both cases the hypotheses for the tests are as follows:
\begin{itemize}
\item $H_0$: All configurations are equally effective (no effect)
\item $H_1$: There is a difference in the effectiveness among the
  configurations, i.e., there is at least one configuration showing a
  significantly different effect on the data points.
\end{itemize}

\subsection{Cochran's Q Test}

The test statistic $T$ is calculated as follows:
\[
T = K(K-1)\frac{\sum_{i=1}^K (R_i - \frac{M}{K})^2}{\sum_{i=1}^rC_i(K-C_i)}
\]
with $R_i$ denoting the row total for the $i^{th}$ configuration,
$C_i$ the column total for the $i^{th}$ data point, and $M$ the grand
total. We reject $H_0$, if $T > \chi^2(1-\alpha, K-1)$ with
$\chi^2(1-\alpha, K-1)$ denoting the $(1 - \alpha)$-quantile of the
$\chi^2$ distribution with $K - 1$ degrees of freedom and $\alpha$ is
the significance level. As \cite{cochran50} points out, the $\chi^2$
approximation breaks down for small tables. \cite{tate1970} state that
as long as the table contains at least 24 entries, the $\chi^2$
approximation will suffice, otherwise the exact distribution should be
used which can either be calculated explicitly
\citep[see][]{patil1975} or determined via permutation.

\subsection{Friedman Test}

Let $R(x_{ij})$ be the rank assigned to $x_{ij}$ within data point $i$
(i.e., rank of a configuration on data point $i$). Average ranks are
used in the case of ties. The ranks for a configuration at position $k$
are summed up over the data points to obtain
\[
R_k = \sum_{i=1}^r R(x_{ki}).
\]
The test statistic $T$ is then calculated as follows:
\[
T = \frac{12}{rK(K+1)}\sum_{i=1}^K(R_i-r(K+1)/2)^2.
\]
If there are ties, then
\[
T = \frac{(K-1) \sum_{i=1}^K (R_i-r(K+1)/2)^2}{[\sum_{i=1}^K
  \sum_{j=1}^r R(x_{ij})^2] - [rK(K+1)^2] / 4}.
\]
We reject $H_0$ if $T > \chi^2(\alpha, K-1)$ with $\chi^2(\alpha,
K-1)$ denoting the $\alpha$-quantile of the $\chi^2$ distribution
with $K - 1$ degrees of freedom and $\alpha$ being the significance
level.

\section{Proof of Safety Zone Bound}
\label{sec:SecProof}
In this section we prove the safety zone bound of
Section~\ref{sec:fnr} of the paper. We will follow the notation and
treatment of the sequential analysis as found in the original
publication of \cite{wald47}, Sections~5.3 to 5.5. First of all,
Wald proves in Equation~5:27 that the following approximation holds:
\[
\operatorname{ASN}(\pi_0, \pi_1| \pi = 1.0) \approx \frac{\Bounda}{\Boundp}.
\]
where $\operatorname{ASN}(\cdot, \cdot)$ (Average Sample Number) is
the expected number of steps until the given test will yield a
decision, if the underlying success probability of the tested sequence
is $\pi = 1.0$. The minimal $\operatorname{ASN}(\pi_0, \pi_1| \pi =
1.0)$ is therefore attained if $\Boundp$ is maximal, which is clearly
the case for $\pi_1=1.0$ and $\pi_0 = 0.5$, which holds by
construction. So we get the lower bound of $\steps$ for a given
significance level $\alpha_l, \beta_l$:
\[
\steps \geq \Big\lceil \Bounda / \log 2 \Big\rceil.
\]
The lower line $L_0$ of the graphical sequential analysis test as
exemplified in Figure~\ref{fig:overview} of the paper is defined as follows (see Equation~5:13 -
5:15):
\[
L_0 = \frac{\Boundb}{\Boundp - \Boundq} + n \frac{\BoundQ}{\Boundp - \Boundq}.
\]
Setting $L_0 = 0$, we can get the intersection of the lower test line
with the x-axis and therefore the earliest step $\ssafe$, in
which the procedure will drop a constant loser configuration. This
yields
\[
\ssafe = -\frac{\Boundb}{\Boundp - \Boundq} /
\frac{\BoundQ}{\Boundp - \Boundq} = -\frac{\Boundb}{\BoundQ} =
\frac{\Boundb}{\Boundq} = \secZone.
\]
The last equality can be derived by inserting the closed form of
$\pi_1$ given $\pi_0 = 0.5$:
\begin{align*}
  \steps = \operatorname{ASN}(\pi_0, \pi_1| \pi = 1.0) =
  \frac{\Bounda}{\Boundp} = \frac{\Bounda}{\log 2\pi_1}
  \Leftrightarrow 2\pi_1 = \BoundPi \Leftrightarrow   \pi_1  = \frac{1}{2} \BoundPi.
\end{align*}

Setting $\ssafe$ in relation to the maximal number of steps
\steps~yields the safety zone bound of Section~\ref{sec:fnr}.




\section{Proof of Computational Budget}
\label{sec:BudgetProof}
For the size $N$ of the whole data set and a learner of time complexity
$f(n) = n^m$, where $m \in \mathbb{N}$, resulting
in a learning time of $\tfull = N^m$, one observes that learning on a
proportion of size $\frac{i}{\steps}N$ takes about
$\frac{i^m}{\steps^m}\tfull$ time. By construction one has to learn on
all $K$ parameter configurations in each step before hitting
$\srate\times\steps$ and on $K\times r$ parameter configurations with
drop rate $(1-r)$ afterwards.
Thus the entirely needed computation time is given by
\[
  K\times(1-r)\sum\limits_{i = 1}^{\srate\times\steps} \frac{i^m}{\steps^m}\tfull + K\times r\sum\limits_{i = 1}^{\steps} \frac{i^m}{\steps^m}\tfull
\]
which should be smaller than the given time budget $T$.

Making use of the fact proved in Appendix~\ref{sec:BudgetProof2} that $\frac{1}{n^{m-1}}\sum\limits_{i = 1}^{n} i^m \overset{\cdot}{\leq}
\frac{n^{2}}{m+1} + \frac{n}{2} + \frac{m}{12}$ holds under the mild condition of $n > \frac{m}{2\pi}$
, where $\overset{\cdot}{\leq}$ describes an asymptotic relation, one can reformulate the inequality as follows:
\begin{gather*}
  \frac{\tfull K (1-r)\srate^{m-1}}{\steps}\frac{1}{(\srate\steps)^{m-1}}\sum\limits_{i = 1}^{\srate\steps} i^m
    + \frac{\tfull K r}{\steps}\frac{1}{\steps^{m-1}}\sum\limits_{i = 1}^{\steps} i^m\\
    \overset{\cdot}{\leq} \frac{\tfull K}{\steps}\Big[(1-r)\srate^{m-1}\left(\frac{(\srate\steps)^{2}}{m+1} + \frac{\srate\steps}{2} + \frac{m}{12}\right) + r\left(\frac{\steps^{2}}{m+1} + \frac{\steps}{2} + \frac{m}{12}\right)\Big]
    \overset{\cdot}{\leq} T.
\end{gather*}
It is obvious that this inequality is quadratic in the variable
$\steps$ which can be solved by bringing the above inequality in
standard form:
\begin{gather*}
  0 \overset{\cdot}{\geq}
  \Big[(1-r)\srate^{m-1}\left(\frac{(\srate\steps)^{2}}{m+1} + \frac{\srate\steps}{2} + \frac{m}{12}\right) + r\left(\frac{\steps^{2}}{m+1} + \frac{\steps}{2} + \frac{m}{12}\right)\Big] -
  \frac{T \steps}{\tfull K}\\
 \Leftrightarrow \; 0 \overset{\cdot}{\geq}
  \frac{(1-r)\srate^{m+1} + r}{m+1}\steps^2 + \Big[\frac{(1-r)\srate^m+r}{2} -
  \frac{T}{\tfull K}\Big]\steps + \left((1-r)\srate^{m-1}+r\right)\frac{m}{12}\\ 
\Leftrightarrow
  \; 0 \overset{\cdot}{\geq} \steps^2 + 2\left[\frac{m+1}{4}\frac{\tfull K(1-r)\srate^m + \tfull K r - 2T}{((1-r)\srate^{m+1}+r)\tfull K}\right]\steps +
  \frac{m(m+1)}{12}\frac{(1-r)\srate^{m-1}+r}{(1-r)\srate^{m+1}+r}.
\end{gather*}
Substituting $a = \frac{m+1}{4}\frac{\tfull K(1-r)\srate^m + \tfull K r - 2T}{((1-r)\srate^{m+1}+r)\tfull K}$ and $b =
\frac{m(m+1)}{12}\frac{(1-r)\srate^{m-1}+r}{(1-r)\srate^{m+1}+r}$ above is equivalent to:
\begin{equation*}
  \steps = -a + y,\, y \in \left\{-\sqrt{a^2 - b},+\sqrt{a^2 - b}\right\}.
\end{equation*}
For the sake of a meaningful step amount, i.e., $\steps > 0$ and
furthermore $\steps$ as large as possible we choose it as
\begin{equation*}
  \steps = \left\lfloor -a +\sqrt{a^2 - b}\right\rfloor.
\end{equation*}
Note that $\steps$ is a function of the parameter $\srate$.
The mild condition for the upper bound of power sums mentioned above has to be fulfilled. Since obviously
$b \geq 0$ holds, $a$ must be negative in order to gain a positive step
amount. Furthermore the root has to be solvable.  So the following
constraints on $\srate$ have to be made:


\begin{align*}
  (1)&\qquad 2T \geq \tfull K(1-r)\srate^m+\tfull K r\\
  (2)&\qquad a^2 \geq b\\
  (3)&\qquad \srate\steps > \frac{m}{2\pi}.
\end{align*}
Note that condition (3) is trivial for a small degree of complexity $m$, which is the common case.
\subsection{Proof of the Upper Bound}
\label{sec:BudgetProof2}
Assume that $n  = \frac{m}{c}$ where $c < 2\pi$. Denote by $B_i$ the \emph{Bernoulli} numbers.
\begin{align*}
 \frac{1}{n^{m-1}}\sum\limits_{i = 1}^{n}i^m = \frac{1}{n^{m-1}}\big[\frac{n^m}{2} + \frac{1}{m+1}\sum\limits_{k=0}^{\lfloor\frac{m}{2}\rfloor}{m+1 \choose 2k}B_{2k}n^{m+1-2k}\big]\\
 = \frac{n^{2}}{m+1} + \frac{n}{2} + \frac{m}{12} + \sum\limits_{k=2}^{\lfloor\frac{m}{2}\rfloor}(-1)^{i+1}\frac{1}{m+1}{m+1 \choose 2k}|B_{2k}|n^{2-2k}
\end{align*}
The sum term is alternating in sign and asymptotically monotone decreasing in $k$:
\begin{align*}
 \frac{1}{m+1}{m+1 \choose 2k}|B_{2k}|n^{2-2k} \sim \frac{m!}{(2k)!(m+1-2k)!}2\frac{(2k)!}{(2\pi)^{2k}}\left(\frac{m}{c}\right)^{2-2k}\\
 = \frac{2m}{c^2}\underbrace{\prod\limits_{j=0}^{2k-2}\left(1 - \frac{j}{m}\right)}_{\downarrow\;0\;as\;k\;\rightarrow\;\infty}\underbrace{\left(\frac{c}{2\pi}\right)^{2k}}_{\downarrow\;0\;,\;\frac{c}{2\pi}<\;1}
\end{align*}
where we use the asymptotic behavior of $|B_{2k}| \sim 2\frac{(2k)!}{(2\pi)^{2k}}$.
Now that the sequence under the sum starts negative, grouping up each two subsequent elements gives a negative value, such that the sum also is negative.

Therefore
\begin{align*}
\frac{1}{n^{m-1}}\sum\limits_{i = 1}^{n}i^m \overset{\cdot}{\leq} \frac{n^{2}}{m+1} + \frac{n}{2} + \frac{m}{12}.
\end{align*}

\bibliography{fastCVjmlr}

\end{document}